\documentclass[11pt]{article}

\usepackage{amssymb}
\usepackage{amsthm}
\usepackage{amsmath}
\usepackage{latexsym}
\usepackage{graphics}
\usepackage{xspace}
\usepackage{psfrag}
\usepackage{epsfig}
\usepackage{pst-all}
\usepackage{algorithmic}
\usepackage{enumerate}
\usepackage{amsmath}
\usepackage{verbatim}

\newtheorem{theorem}{Theorem}[section]
\newtheorem{lemma}[theorem]{Lemma}

\newtheorem{proposition}[theorem]{Proposition}

\newtheorem{definition}[theorem]{Definition}

\newcommand{\ER}{Erd{\"o}s-R\'{e}nyi }

\newcommand{\dist}{\mathrm{dist}}

\newcommand{\ignore}[1]{\relax}

\begin{document}

\title{\Large Matrix Completion from $O(n)$ Samples in Linear Time
}

\author{
{\sf David Gamarnik}\thanks{MIT; e-mail: {\tt gamarnik@mit.edu}}
\and
{\sf Quan Li}\thanks{MIT; e-mail: {\tt quanli@mit.edu}}
\and
{\sf Hongyi Zhang}\thanks{MIT; e-mail: {\tt hongyiz@mit.edu}}
}

\date{\vspace{-5ex}}
\maketitle

\let\thefootnote\relax\footnotetext{Accepted for presentation at Conference on Learning Theory (COLT) 2017}

\begin{abstract}
We consider the problem of reconstructing a rank-$k$ $n \times n$ matrix $M$ from a sampling of its entries. 
Under a certain incoherence assumption on $M$ and for the case when both the  rank and the condition number of $M$ are bounded, 
it was shown in \cite{CandesRecht2009, CandesTao2010, keshavan2010, Recht2011, Jain2012, Hardt2014} that $M$ can be recovered exactly or approximately (depending on some trade-off between accuracy and computational complexity)
using $O(n \, \text{poly}(\log n))$ samples in super-linear time $O(n^{a} \, \text{poly}(\log n))$ for some constant $a \geq 1$. 

In this paper, we propose a new matrix completion algorithm using a novel sampling scheme based on a union of independent sparse 
random regular bipartite  graphs. 
We show that under the same conditions w.h.p. our algorithm recovers an $\epsilon$-approximation of $M$ in terms of the Frobenius norm using $O(n \log^2(1/\epsilon))$ samples and in linear time $O(n \log^2(1/\epsilon))$. This provides the best known bounds both on the sample complexity and computational complexity for reconstructing (approximately) an unknown low-rank matrix. 

The novelty of  our algorithm  is two new steps of thresholding singular values and rescaling singular vectors in the application of the ``vanilla'' alternating minimization algorithm. The structure of sparse random regular graphs is used heavily for controlling the impact of these regularization steps.  
\end{abstract}

\section{Introduction}
We consider the problem of reconstructing a hidden rank-$k$ matrix from a sampling of its entries. Specifically, consider an $n \times n$ matrix $M$. The goal is to design a sampling index set $\Omega \subseteq [n] \times [n]$ such that $M$ can be reconstructed efficiently from the entries in $M$ associated with $\Omega$, that is, from the entries $M_{ij}, (i,j) \in \Omega$, with the cardinality $|\Omega|$ as small as possible. The problem has a wide range of applications in recommendation systems, system identification, global positioning, computer vision , etc. \cite{CandesPlan2009}.

For the convenience of discussing various matrix completion results and comparing them to our results, we will assume in the discussion below that the rank $k$, condition number $\kappa$ and the incoherence parameter $\mu_0$ of $M$ (appropriately defined) are bounded in $n$.
The problem of reconstructing $M$ under uniform sampling received considerable attention in recent years.
One research direction of matrix completion under this sampling scheme focuses on the exact recovery of $M$. 
Recht \cite{Recht2011} and Gross \cite{gross2011} showed that $M$ can be reconstructed exactly from $O(n \log^2 n)$ samples using trace-norm based optimization. Keshavan et al. \cite{keshavan2010} showed that $M$ can be reconstructed exactly from $O(n \log n)$ samples using singular value decomposition (SVD) followed by gradient descent on Grassmanian manifold. 
Another research direction of matrix completion under uniform sampling pays more attention to the efficiency of the algorithm, and only requires approximate matrix completion. 
Jain et al. \cite{Jain2012} showed that an $\epsilon$-approximation (appropriately defined) of $M$ in the Frobenius norm can be reconstructed from $O(n \log n \log(1/\epsilon))$ samples using alternating minimization in $O(n \log n \log(1/\epsilon))$ time. 
Then Hardt \cite{Hardt2014} refined the analysis of alternating minimization and improved the sample complexity to $O(n \log(n/\epsilon))$.
With extensive research on this subject, it is tempting to believe that the sample complexity obtained by \cite{Jain2012} or \cite{Hardt2014} are optimal (up to a constant factor) for $\epsilon$-approximation of matrix completion as well. 
Perhaps surprisingly, we establish that this is not the case and propose a new algorithm, which constructs an $\epsilon$-approximation of $M$ in Frobenius norm using $O(n \log^2(1/\epsilon))$ samples in linear time $O(n \log^2(1/\epsilon))$. The comparison of various matrix completion methods is given in Table \ref{comparison}.
In order to compare various methods for exact and approximate matrix completion, the criterion $\| M - \tilde{M} \|_F \leq \epsilon \| M \|_F$ is used where $\epsilon$ is the tolerance, $\tilde{M}$ is the reconstructed matrix and $\| \cdot \|_F$ is the Frobenius norm.      
\begin{center}
	{\renewcommand{\arraystretch}{2}%
	\begin{table}
	\label{comparison}
	\caption{Comparison of various matrix completion methods. The methods with superscript symbols `$\dagger$' are for exact matrix completion while the remaining methods without `$\dagger$' are for approximate matrix completion. The methods with superscript symbols `$\ast$' are under stronger incoherence assumption than the standard incoherence assumption (Assumption $1$, appropriately defined) while others are under the standard incoherence assumption.  
	$\epsilon$ is the tolerance such that the reconstructed matrix $\tilde{M}$ satisfies $\| M - \tilde{M} \|_F \leq \epsilon \| M \|_F$ w.h.p. $\tilde{O}$ notation hides factors polynomial in $k$, $\kappa$ and $\mu_0$.}
	\vspace{5pt}
	\begin{tabular}{c|cc}
		\hline
		Methods & Sample Complexity &  Running Time \\
		\hline \hline
		\cite{keshavan2010}$^{\dagger, \ast}$ & $O\left( \kappa^2\mu_0 k    n \left(\log n + \kappa^4\mu_0 k\right)\right)$ & $O\left(  \   \kappa^2\mu_0 k^2    n \log n \left(\log n + \kappa^4\mu_0 k\right)\right)$ \\
		\hline
		\cite{Recht2011}$^{\dagger, \ast}$, \cite{gross2011}$^{\dagger, \ast}$ & $O\left(\mu_0^2 k  n \log^2 n \right)$ & $O(n^2 \log n /\sqrt{\epsilon})$ or $O(n^5 \log(1/\epsilon))$ \\
		\hline
		\cite{chen2015incoherence}$^{\dagger}$ & $O\left(\mu_0 k   n\log^2 n\right)$ & $O(n^2 \log n /\sqrt{\epsilon})$ or $O(n^5 \log(1/\epsilon))$ \\
		\hline
		\cite{sun2016guaranteed}$^{\dagger}$ & $O\left(\kappa^2\mu_0 k  n \left(\log n + \mu_0 k^6\kappa^4\right)\right)$ & $\tilde{O}(\text{poly}(n) \log(\frac{1}{\epsilon}) )$ \\
		\hline
		\cite{zheng2016convergence}$^{\dagger}$ & $O\left(\kappa^2\mu_0 k^2    n(\log n + \mu_0)\right)$ & $O\left(k n^2 \log(\frac{1}{\epsilon}) \right)$ \\
		\hline
		\cite{balcan2017optimal}$^{\dagger}$ & $O\left(\kappa^2\mu_0 k   n \log n \log_{2 \kappa} n \right)$ & $O\left(\frac{n^3}{\epsilon}\right)$ \\
		\hline
		\cite{Jain2012} & $O\left(\kappa^4\mu_0^2k^{4.5}\log \left(\frac{k}{\epsilon}\right)  n\log n \right)$ & $O\left(    \kappa^4\mu_0^2k^{6.5}  n\log n \log \left(\frac{k}{\epsilon}\right)  \right)$ \\
		\hline
		\cite{Hardt2014} & $O\left( \frac{\|M^*\|_F^2}{(\sigma_k^*)^2} \mu_0 k    n\left(\log\left(\frac{n}{\epsilon}\right) + k\right) \right)$ & $O\left(\frac{\|M^*\|_F^2}{(\sigma_k^*)^2} \mu_0 k^3   n\left(\log\left(\frac{n}{\epsilon}\right) + k\right)  \right)$ \\
		\hline
		\cite{zhao2015nonconvex} & $O\left(\kappa^4 \mu_0  k^3 n\log n \log(\frac{1}{\epsilon}) \right)$ & $O\left(\kappa^4 \mu_0  k^4 \  n\log n \log(\frac{1}{\epsilon}) \right)$ \\
		\hline
		Ours & $O\left(\left(\kappa^2\mu_0^2k^4 + \mu_0 k\log\left(\frac{1}{\epsilon}\right)\right) n \log(\frac{1}{\epsilon}) \right)$ & $O( \left(\kappa^2\mu_0^2 k^6 + \mu_0 k^3 \log\left(\frac{1}{\epsilon}\right)\right) n \log(\frac{1}{\epsilon}) )$\\
		\hline
	\end{tabular}
	\end{table}}
\end{center}

Our proposed algorithm adds two new steps: a thresholding of singular values and a rescaling of singular vectors upon the ``vanilla" alternating minimization algorithm. 
The idea behind these steps is regularization of the least square estimation in the form of the singular value thresholding. 
The singular value thresholding step is necessary due to the decreased sample complexity. 
More specifically, due to the decreased sample complexity by a logarithmic factor $\log n$, certain matrices inverted in each step of the alternating minimization algorithm may become ill-conditioned.  
Our algorithm avoids this ill-conditioning problem by adding to the ``vanilla'' alternating minimization
an extra step of singular value thresholding applied to these matrices (i.e. the Gramian matrices inverted in (\ref{algo:stept1}) and (\ref{algo:stept3})) before their inversion. 
This extra singular value thresholding step enforces that the singular values of the Gramian matrices inverted in (\ref{algo:stept1}), (\ref{algo:stept2}), (\ref{algo:stept3}) and (\ref{algo:stept4}) deviate from their expected values by at most $1-\beta$ after proper normalization, and as a result, guarantees the nonsingularity of these (adjusted) Gramian matrices.  
We call this algorithm Thresholded Alternating Minimization ($\mathcal{TAM}$), referring to the extra singular value thresholding steps added to alternating minimization. 
A rescaling of the entries of singular vectors is also implemented in the $\mathcal{TAM}$ algorithm in order to maintain
the proximity to incoherence. 
A more specific discussion of the intuition behind these two new steps appears after the introduction of the $\mathcal{TAM}$ algorithm (in Pages 7 and 8). 

We restrict our attention to the case of bounded rank, bounded condition number and bounded incoherence parameter of $M$, for the convenience of the analysis.
Most of the work in this paper is to prove the following result: with high probability (w.h.p.) $\mathcal{TAM}$ produces a $1 \pm \epsilon$ multiplicative approximation of $M$ in Frobenius norm using $O(n \log^2 (1/\epsilon))$ samples under the standard incoherence Assumption $1$, given in Section \ref{section: problem_formulation}.
For simplicity, we call this just $\epsilon$-approximation.
Let $M = U^* \Sigma^* (V^*)^T$ and $U$ be the input to one of the iterations of $\mathcal{TAM}$. Also, let $\gamma$ be the distance between the subspaces spanned by $U^*$ and $U$, appropriately defined later.  
We further establish that the number of times that the singular value thresholding is applied per one iteration of $\mathcal{TAM}$ 
is bounded above by a function of $\gamma$, which is monotonically decreasing as $\gamma$ decreases.  
The novel bounding technique we used for establishing this result is based on random graph theory.
More specifically, the detailed structure of sparse random regular graphs is used heavily on controlling the impact of regularization, i.e. the number of times that the singular value thresholding steps are applied per one iteration of $\mathcal{TAM}$ algorithm. 
This result is summarized in Theorem \ref{theorem:UpperBoundSbt}.  
We use it as a key result in establishing the geometric convergence of $\mathcal{TAM}$.
The analysis of our algorithm is substantially different from the one in \cite{Jain2012}, due to this critical singular value thresholding step.
Although the proof of our main result seems involved, most of the proof steps use elementary linear algebraic derivations and are easy to follow.   

For the convenience of analysis, $\mathcal{TAM}$ employs a sampling generated from a union of independent random bipartite regular graphs.   
Although our results of $\mathcal{TAM}$ are established on this special sampling, $\mathcal{TAM}$ can be generalized to uniform sampling in the obvious manner and similar results of $\mathcal{TAM}$ under uniform sampling can be established accordingly. 
In fact, by considering Poisson cloning model \cite{Kim2006} for \ER graphs, (which we intend to research in future), we conjecture that the same sample complexity of $\mathcal{TAM}$ might hold for constructing an $\epsilon$-approximation of $M$ in Frobenius norm under uniform sampling. There is no contradiction between the information theoretic lower bound $O(n \log n)$ for exact matrix completion and this conjecture, due to its approximate nature.   
Other sampling schemes for matrix completion are also studied in \cite{meka2009matrix,kiraly2015algebraic,pimentel2016characterization}.  

Bhojanapalli and Jain \cite{Bhojanapalli2014} showed that if the index set of the sampled entries corresponds to a bipartite graph with large spectral gap, then the trace-norm based optimization exactly reconstructs $M$ that satisfies certain stricter incoherence assumptions (Assumption $1$ and condition (\ref{restricted_IC}), see below). In particular, they showed that the trace-norm based optimization exactly reconstructs $M$ for $\delta \leq 1/6$ in (\ref{restricted_IC}) using $O(k^2 n)$ samples. 
Furthermore, they raised a question of studying alternating minimization under the same incoherence assumptions, in the hope of achieving similar sample complexity.   
Our second result answers this question for the case of constant $k$: w.h.p. $\mathcal{TAM}$ under incoherence Assumptions $1$ and $2$ produces an $\epsilon$-approximation of $M$ in Frobenius norm using $O(n \log(1/\epsilon))$ samples. 
Furthermore, this result requires a less stringent incoherence condition (Assumption $2$) on $M$ than the incoherence condition (\ref{restricted_IC}), and furthermore holds for all $\delta \in (0, 1)$ satisfying condition (\ref{incoherence_assumption2}) in Assumption $2$ while the result in \cite{Bhojanapalli2014} holds for all $\delta \in (0, 1/6]$ satisfying condition (\ref{restricted_IC}). 

$\mathcal{TAM}$ maintains the computational complexity of alternating minimization, which is $O( |\Omega|)$ for bounded $k$. 
$\mathcal{TAM}$ only requires $O(n \log^2(1/\epsilon))$, or $O(n \log(1/\epsilon))$ samples, depending on whether Assumption $1$ or both Assumptions $1$ and $2$ are satisfied, respectively.
Hence, $\mathcal{TAM}$ is a linear algorithm of computational complexity $O(n \log^2(1/\epsilon))$ or $O(n \log(1/\epsilon))$. Like alternating minimization, $\mathcal{TAM}$ has computational efficiency advantage over trace-norm based optimization, which requires time $O(n^2 \log n \slash \sqrt{\epsilon})$ using the singular value thresholding algorithm \cite{CaiCandes2010} or $O(n^5 \log(1 \slash \epsilon))$ using interior point methods. More specific computational complexity comparison between trace-norm based optimization and alternating minimization is given in \cite{Jain2012}.   

The remainder of the paper is structured as follows. In the next section, we define the problem of matrix completion and state necessary  assumptions. 
In Section \ref{section: main_result}, we introduce the random $d$-regular graph model of $\Omega$ and formally state our two main results: the one regarding the performance of $\mathcal{TAM}$ under the incoherence Assumption $1$ and the one regarding the performance of $\mathcal{TAM}$ under the incoherence Assumptions $1$ and $2$. 
Section \ref{section:analysis} is devoted to the proof of two main results. 
We conclude in Section \ref{section:conclusion} with some open questions.  

We close this section with some notational conventions.    
We use standard notations $o(\cdot)$, $O(\cdot)$ and $\Omega(\cdot)$ with respect to $n \rightarrow \infty$. Let $\sigma_i(A)$ be the $i$-th largest singular value of matrix $A$ and $\sigma_{min}(A)$ be the least singular value of matrix $A$. Let $\| A \|_2$ be the spectral norm (largest singular value) of matrix $A$ and $\|A\|_F$ be the Frobenius norm of matrix A. Let $A^T$ be the transpose of a vector or matrix $A$. 
For $a \in \mathbb{N}$, let $[a]$ be a set of indices $\{1,2, \ldots, a\}$.
Let $k \in \mathbb{N}$ be the rank of matrix $M$. 
For a matrix $U \in \mathbb{R}^{n \times k}$, let $u_i^T$, $i \in [n]$, be the $i$-th row of $U$ where $u{}_i \in \mathbb{R}^{k \times 1}$ is a column vector. Also, let $\text{Span}(U)$ be the subspace spanned by the $k$ columns of $U$. 
For a matrix $A \in \mathbb{R}^{n \times n}$, let $\text{SVD}(A, k) \in \mathbb{R}^{n \times k}$ be the matrix consisting of the top-$k$ left singular vectors of the matrix $A$. 
Let $\langle x, y \rangle$ be the inner product of two vectors $x$ and $y$, and $\lceil z \rceil$ be the smallest integer no less than $z$. We say that a sequence of events $E_n$ occurs w.h.p. if $\mathbb{P}(E_n) \rightarrow 1$ as $n \rightarrow \infty$.
Given $l \leq n$, we call a matrix $A \in \mathbb{R}^{n \times l}$ with orthonormal columns (column-)orthonormal matrix.  
A QR decomposition of a matrix $A \in \mathbb{R}^{n \times k}$ is 
$A =  Q R$
where $Q \in \mathbb{R}^{n \times k}$ is an orthonormal matrix and $R \in \mathbb{R}^{k \times k}$ is an upper triangular matrix. 
We include the following list of matrix inequalities to be used later. Given a matrix $A$ of rank $l$
\begin{align}
\label{inequality: matrix_inequality1}
\| A \|_F &\leq \sqrt{l} \| A \|_2.
\end{align}
Given two matrices $A$ and $B$
\begin{align}
\label{inequality: matrix_inequality2}
\| A B \|_F &\leq \| A \|_2 \| B \|_F. 
\end{align}
Give matrices $A, B \in \mathbb{R}^{n \times n}$, the following Ky Fan singular value inequality \cite{Moslehian2012} holds
\begin{align}
\label{KyFanSVInequality}
\sigma_{r+t+1}(A + B) \leq \sigma_{r+1}(A) + \sigma_{t+1}(B)
\end{align}
for $t \geq 0, r \geq 0$ and $r+t+1 \leq n$.

\section{Problem Formulation and Assumptions} \label{section: problem_formulation}
Let $M \in \mathbb{R}^{n \times m}$ be a rank-$k$ matrix and $M = U^* \Sigma^* (V^{*})^T$ be its SVD where the singular values are $\sigma_1^* \geq \sigma_2^* \ldots \geq \sigma_k^*$ in decreasing order. The entries in $M$ associated with the index set $\Omega \subseteq [n] \times [m]$ are observed, that is, the entries $M_{ij}$, $\forall (i,j) \in \Omega$, are known. Define the sampling operator $P_{\Omega}: \mathbb{R}^{n \times m} \rightarrow \mathbb{R}^{n \times m}$ by
\begin{align*}
P_{\Omega}(M) = \left\{
\begin{array}{ll}
M_{ij} & \quad \text{if } (i,j) \in \Omega, \\
0 & \quad \text{if } (i,j) \notin \Omega.
\end{array}
\right.
\end{align*}  
Let $\mathcal{V}_R$ and $\mathcal{V}_C$ be the sets of rows and columns of matrix $M$, respectively, indexed by the sets $\{1,2,\ldots, n\}$ and $\{1,2,\ldots, m\}$. Also, let $\mathcal{G} = (\mathcal{V}, \mathcal{E})$ be a bipartite undirected graph on the vertex set $\mathcal{V} = \mathcal{V}_{R} \cup \mathcal{V}_{C}$ with edge set $\mathcal{E} \ni (i,j) $ if and only if $(i,j) \in \Omega$.  Our goal is to obtain an $\epsilon$-approximation of the matrix $M$ from the observed $P_{\Omega}(M)$.    

For the rest of the paper, we will assume for simplicity that $m = n$. Our results can be easily extended to the more general case $m = \Theta(n)$, using the generalization as in the appendix $D$ of \cite{Hardt2014}.  
We say a graph is a random bipartite $d$-regular graph $\mathbb{G}_d(n,n)$ if it is chosen uniformly at random from all bipartite $d$-regular graphs with $n$ vertices $\{1,2,\ldots, n\}$ on the left and another $n$ vertices $\{1,2,\ldots, n\}$ on the right.
Let $G_n \in \mathbb{R}^{n \times n}$ be the bi-adjacency matrix of $\mathbb{G}_d(n, n)$ with the $(i,j)$ entry $(G_{n})_{ij}=1$ if and only if there is an edge between vertex $i$ on the left and vertex $j$ on the right in $\mathbb{G}_d(n,n)$ and $(G_{n})_{ij}=0$ otherwise. For our proposed algorithm, we choose $\mathcal{G}$ to be a union of several independent random bipartite $d$-regular graphs $\mathbb{G}_d(n, n)$. Two essential properties of the random bipartite $d$-regular graph are
\begin{itemize}
	\item $P1 \quad$ Top left (right) singular vector of $G_n$ is $[1/\sqrt{n},1/\sqrt{n},\ldots,1/\sqrt{n}]^T$. 
	\item $P2 \quad$ The largest singular value $\sigma_1(G_n) = d$. As discussed below, w.h.p. the second largest singular value $\sigma_2(G_n)$ is upper bounded by $(7 \sqrt{d}) / 3$ for any $d \geq 3$. 
\end{itemize}
The eigenvalues of the adjacency matrix of the graph $\mathbb{G}_d(n, n)$ are $\cup_{i=1}^n \{ -\sigma_i(G_n), \sigma_i(G_n)\}$. 
Corollary 1.6 in \cite{Puder2015} states that alongside the two trivial eigenvalues $\pm d$, all other eigenvalues of the adjacency matrix of the graph $G_d(n,n)$ are within $[-2\sqrt{d-1}-0.84, 2\sqrt{d-1} + 0.84]$ w.h.p. as $n \rightarrow \infty$. 
For $d \geq 3$, we have $2\sqrt{d-1} + 0.84 \leq (7 \sqrt{d}) / 3$ and hence property $P2$ follows. Random (bipartite) regular graphs are widely studied in recent years. Bayati et al. \cite{Bayati2010} proposed an algorithm for generating a random bipartite $d$-regular graph $\mathbb{G}_d(n, n)$ in expected running time $O(nd^2)$.

Let $u_i^{*,T}$, $i \in [n]$, be the $i$-th row of $U^*$ and $v_j^{*,T}$, $j \in [n]$, be the $j$th row of $V^*$.
Now we present the incoherence assumptions on $M$. 
\begin{itemize}
	\item \textbf{Assumption $1$.} There exists a constant $\mu_0 \geq 1$ such that  
	\begin{align}
	\label{incoherence_assumption1}
	\| u_i^* \|_2^2 \leq \frac{\mu_0 k}{n}, \forall i \in [n] \text{ and } \| v_j^* \|_2^2 \leq \frac{\mu_0 k}{n}, \forall j \in [n]. 
	\end{align} 
	\item \textbf{Assumption $2$.} Given the degree $d$ of $\mathbb{G}_d(n, n)$, let $S_n$ be a subset of $[n]$ chosen uniformly at random from all the subsets of $[n]$ with cardinality $d$. There exists a constant $\delta \in (0, 1)$ such that
	{\small
	\begin{align}
	\label{incoherence_assumption2}
	\quad &  \mathbb{P}(\| \sum_{i \in S_n } \frac{n}{d} u_i^* u_i^{*,T} - I \|_2 \leq \delta) = 1-o(1) \;\; \text{and} \;\;   \mathbb{P}(\| \sum_{j \in S_n } \frac{n}{d} v_j^* v_j^{*,T} - I \|_2 \leq \delta) = 1-o(1).             
	\end{align}
	}
\end{itemize} 
where Assumption $1$ is the standard incoherence condition assumed by most of existing low-rank matrix completion results \cite{CandesRecht2009,keshavan2010,Jain2012, Hardt2014} etc. We call Assumption $2$ the probabilistic generalized restricted isometry condition, which is strictly weaker, for example, than the incoherence assumption $A2$ in \cite{Bhojanapalli2014}. 
The latter requires 
\begin{align}
\label{restricted_IC}
\bigg\| \sum_{i \in {S}_n^1 } \frac{n}{d} u_i^* u_i^{*,T} - I \bigg\|_2 \leq \delta \quad \text{and} \quad \bigg\| \sum_{j \in {S}_n^2 } \frac{n}{d} v_j^* v_j^{*,T} - I \bigg\|_2 \leq \delta,
\end{align}
for $\delta \leq 1/6$ and all ${S}_n^1, {S}_n^2 \subset [n]$ of cardinality $|{S}_n^1|=|{S}_n^2|=d$ while the probabilistic generalized restricted isometry condition (\ref{incoherence_assumption2}) requires the inequalities above hold for majority of the subsets ${S}_n^1 \subset [n]$ of cardinality $|{S}_n^1|=d$ and for majority of the subsets ${S}_n^2 \subset [n]$ of cardinality $|{S}_n^2|=d$.          

\section{Main Results} \label{section: main_result}
We are about to present a new matrix completion algorithm and give recovery guarantees of the proposed algorithm for two scenarios: matrix completion under Assumption $1$, and matrix completion under both Assumption $1$ and Assumption $2$. 
Furthermore, we will assume that Assumption $1$ always holds, and that the rank $k$, the condition number $\sigma_1^*/\sigma_k^*$, and the incoherence parameter $\mu_0$ of the matrix $M$ are bounded from above by a constant, as $n \rightarrow \infty$. 

Now we formally describe the matrix completion algorithm we propose in this paper and state our main results. For the statement of our algorithm, we first introduce two operators acting on the matrices. Define $\mathcal{T}_1: \mathbb{R}^{k \times 1} \rightarrow \mathbb{R}^{1 \times k}$ by 
\begin{align}
\label{definition:T1}
\mathcal{T}_1(u) \triangleq \left\{
\begin{array}{ll}
\sqrt{\frac{\mu_0 k}{n}} \frac{u^{T}}{\|u\|_2}  & \quad \|u\|_2  \geq 2 \sqrt{\frac{  \mu_0 k}{n}},  \\
u^T & \quad \|u\|_2 < 2 \sqrt{\frac{  \mu_0 k}{n}}. \\    
\end{array}
\right.
\end{align} 
Specifically, the operator $\mathcal{T}_1$ normalizes the vector $u$ of length at least $2 \sqrt{\mu_0 k/n}$ to the vector of the same direction and of length $\sqrt{\mu_0 k/n}$. For the convenience of notation we extend $\mathcal{T}_1$ to the one acting on matrix $U = (u_i^T, i\in [n]) \in \mathbb{R}^{n \times k}$ by  
\begin{align*}
\mathcal{T}_1(U) \triangleq \left( \begin{array}{c}
\mathcal{T}_1(u_1)   \\
\vdots  \\
\mathcal{T}_1(u_n)   \end{array} \right).
\end{align*}
Then it follows from the definition of $\mathcal{T}_1(\cdot)$ in (\ref{definition:T1}) that any row vector of $\mathcal{T}_1(U)$ has length at most $2 \sqrt{\mu_0 k/n}$. 

For $A \in \mathbb{R}^{d \times k}$, let the SVD of $A$ be 
$$
A = U_A \Sigma_A (V_A)^T.
$$ 
We write $\Sigma_A$ in the form $\sqrt{d/n} \; \text{diag} (\sigma_1, \cdots, \sigma_k)$ where the diagonal entries $\sigma_1, \sigma_2 \ldots, \sigma_k$ ($\sigma_1 \geq \sigma_2 \ldots \geq \sigma_k$) are the singular values of $A$ divided by $\sqrt{d/n}$. For a given $a \in (0, 1)$ and $\forall i \in [k]$, let 
\begin{align*}
\sigma_{i,a} = \left\{
\begin{array}{ll}
\sigma_i & \quad \text{ if } \sigma_i \in [\sqrt{a}, \sqrt{2-a} ], \\
\sqrt{a} & \quad \text{ if }  \sigma_i < \sqrt{a}, \\
\sqrt{2 - a} & \quad \text{ if }  \sigma_i > \sqrt{2 - a}. 
\end{array}
\right.
\end{align*}  
Define $\mathcal{T}_2(A, a)$ by 
\begin{align}
\label{T2onUS}
\mathcal{T}_2(A, a) \triangleq U_A \hat{\Sigma}_A (V_A)^T
\end{align}
where $\hat{\Sigma}_A = \sqrt{d/n} \; \text{diag} (\sigma_{1, a}, \cdots, \sigma_{k, a})$ and hence 
the entire $\sigma_{1, a}, \cdots, \sigma_{k, a}$ satisfy 
$$
\sqrt{2-a} \geq \sigma_{1,a} \geq \sigma_{2,a} \ldots \geq \sigma_{k,a} \geq \sqrt{a}.
$$ 
Specifically, the operator $\mathcal{T}_2$ lifts the normalized singular values in $\Sigma_A$ less than $\sqrt{a}$ to $\sqrt{a}$ and truncates the normalized singular values in $\Sigma_A$ more than $\sqrt{2 - a}$ to $ \sqrt{2 - a}$. 

Let $\Omega_{t} \subseteq [n] \times [n]$, $t=0,1,\ldots,2N$, be the index sets associated with $2N+1$ independent random bipartite $d$-regular graphs $\mathbb{G}_d(n, n)$.
Define $\mathcal{RRG}(d, n, N)$ as the random $d$-regular graph model of $\Omega$, that is,  
\begin{align}
\label{RRG}
\mathcal{RRG}(d, n, N) \triangleq \{\Omega_0, \Omega_1, \cdots, \Omega_{2N}\}.
\end{align}
Let $D$ be a subset of $[n]$ with $d$ entries, namely, $D = \{i_1, i_2, \ldots, i_d\}$. For a matrix $U = (u_i^T, i\in [n]) \in \mathbb{R}^{n \times k}$, let its submatrix with the row indices in $D$ and the column indices the same as $U$ be  
\begin{align*}
U_{D} = \left( \begin{array}{c}
u_{i_1}^T  \\
\vdots \\
u_{i_d}^T  \end{array} \right).
\end{align*} 
Let $S_j^{t,L} = \{i \in [n]: (i,j) \in \Omega_t \}$, $\forall j \in [n]$. Then $|S_j^{t,L}|=d$. Namely, $S_j^{t,L}$ consists of all the left neighbors of vertex $j$ on the right in the random bipartite $d$-regular graph associated with the index set $\Omega_t$. 
Correspondingly given any $a \in (0, 1)$ and any $j \in [n]$, we denote $\hat{U}_{S_j^{t,L}} = \mathcal{T}_2(U_{S_j^{t,L}}, a)$ and the row in $\hat{U}_{S_j^{t,L}}$ associated with the index $i \in S_j^{t,L}$ by $\hat{u}_{i}^{t,T}$. Similarly, let $S_i^{t, R} = \{j \in [n]: (i,j) \in \Omega_t\}$, $\forall i \in [n]$, that is, $S_i^{t,R}$ consists of all the right neighbors of vertex $i$ on the left in the random bipartite $d$-regular graph associated with the index set $\Omega_t$. Also, we have $|S_i^{t, R}|=d$. For a matrix $V \in \mathbb{R}^{n \times k}$ and a given $a \in (0, 1)$, denote similarly $\hat{V}_{S_i^{t,R}} = \mathcal{T}_2(V_{S_i^{t,R}}, a)$ and the row in $\hat{V}_{S_i^{t,R}}$ associated with the index $j \in S_i^{t,R}$ by $\hat{v}_{j}^{t,T}$. 

Now we introduce the algorithm $\mathcal{TAM}$ for matrix completion in the sparse regime. For the algorithm below we fix arbitrary $\delta \in (0, 1)$ and we let $\beta$ be any constant in $(0, 1-\delta)$.
\\
\begin{algorithmic}
	\STATE \textbf{Thresholded Alternating Minimization algorithm ($\mathcal{TAM}$) } \\
	\emph{
		\STATE \textbf{Input}: Observed index sets $\mathcal{RRG}(d, n, N)$ and values $P_{\cup_{t=0}^{2N} \Omega_t }(M)$. \\ 
		\STATE \textbf{Initialize}: $\bar{U}^0 = \text{SVD}(\frac{n}{d} P_{\Omega_0}(M), k)$, i.e. top-$k$ left singular vectors of $\frac{n}{d} P_{\Omega_0}(M)$. \\
		\STATE  Truncation step: first apply $\mathcal{T}_1$ on $\bar{U}^0$ then orthonormalize the columns of $\mathcal{T}_1(\bar{U}^{0})$. Denote the resultant orthonormal matrix by $U^{0} = (u_i^{0,T}, 1\leq i \leq n)$.   
		\STATE \textbf{Loop}: For $t = 0$ to $N-1$  \\
		$\quad \quad \; \;$ For each $j \in [n]$:  \\
		$\quad$  $\quad$ $\quad \quad$ If $\frac{n}{d}\sigma_l (\sum_{i \in [n]: (i,j) \in \Omega_{t+1}} u_{i}^t u_{i}^{t,T}) \in [\beta, 2 - \beta] $ for all $l \in [k]$, then set 
		\begin{align}
		\label{algo:stept1}
		\tilde{v}_{j}^{t+1} = \left( \sum_{i \in [n]: (i,j) \in \Omega_{t+1}} u_{i}^{t} u_{i}^{t,T} \right)^{-1} \sum_{i \in [n]: (i,j) \in \Omega_{t+1}} u_{i}^t M_{ij}.
		\end{align}
		$\quad$  $\quad$ $\quad \quad$  Otherwise let $\hat{U}_{S_j^{t+1,L}}^t = \mathcal{T}_2(U_{S_j^{t+1,L}}^t, \beta)$ and 
		\begin{align}
		\label{algo:stept2}
		\tilde{v}_{j}^{t+1} = \left( \sum_{i \in [n]: (i,j) \in \Omega_{t+1}} \hat{u}_{i}^t \hat{u}_{i}^{t,T}  \right)^{-1} \sum_{i \in [n]: (i,j) \in \Omega_{t+1}} \hat{u}_{i}^t M_{ij}.
		\end{align}
		$\quad \quad \; \; $ Let $\tilde{V}^{t+1} = (\tilde{v}_j^{t+1, T}, 1 \leq j \leq n)$ and $ \tilde{V}^{t+1} = \bar{V}^{t+1} R^{t+1} $ be the \textrm{QR} decomposition of  \\
		$\quad \quad \; \; $ $\tilde{V}^{t+1}$. Orthonormalize the columns of $\mathcal{T}_1(\bar{V}^{t+1})$. Denote the resultant orthonormal matrix  \\
		$\quad \quad \; \; $ by $V^{t+1} = (v_j^{t+1, T}, 1 \leq j \leq n)$.    \\
		$\quad \quad \; \;$ For each $i \in [n]$:  \\
		$\quad$  $\quad$ $\quad \quad$ If $\frac{n}{d}\sigma_l (\sum_{j \in [n]: (i,j) \in \Omega_{N+t+1}} v_{j}^{t+1} v_{j}^{t+1,T}) \in [\beta, 2 - \beta]$ for all $l \in [k]$, then set 
		\begin{align}
		\label{algo:stept3}
		\tilde{u}_{i}^{t+1} = \left( \sum_{j \in [n]: (i,j) \in \Omega_{N+t+1}} v_{j}^{t+1} v_{j}^{t+1,T} \right)^{-1} \sum_{j \in [n]: (i,j) \in \Omega_{N+t+1}} v_{j}^{t+1} M_{ij}.
		\end{align}
		$\quad$  $\quad$ $\quad \quad$  Otherwise let $\hat{V}_{S_i^{N+t+1, R}}^{t+1} = \mathcal{T}_2(V_{S_i^{N+t+1, R}}^{t+1}, \beta)$ and
		\begin{align}
		\label{algo:stept4}
		\tilde{u}_{i}^{t+1} = \left( \sum_{j \in [n]: (i,j) \in \Omega_{N+t+1}} \hat{v}_{j}^{t+1} \hat{v}_{j}^{t+1, T} \right)^{-1} \sum_{j \in [n]: (i,j) \in \Omega_{N+t+1}} \hat{v}_j^{t+1} M_{ij}.
		\end{align}
		$\quad \quad \; \; $ Let $\tilde{U}^{t+1} = (\tilde{u}_j^{t+1, T}, 1 \leq j \leq n)$ and $\tilde{U}^{t+1} = \bar{U}^{t+1} R^{N+t+1}  $ be the \textrm{QR} decomposition of  \\ 
		$\quad \quad \; \; $ $\tilde{U}^{t+1}$. Orthonormalize the columns of $\mathcal{T}_1(\bar{U}^{t+1})$. Denote the resultant orthonormal matrix  \\
		$\quad \quad \; \; by $ $U^{t+1} = (u_i^{t+1, T}, 1 \leq i \leq n)$.   \\
		\STATE \textbf{Output}: Set $U^{N-1}=(u_{i}^{N-1,T}, 1 \leq i \leq n), \tilde{V}^N = (\tilde{v}_{j}^{N,T}, 1\leq j \leq n)$. Output $M_N = U^{N-1} \tilde{V}^{N,T}$.
	}
\end{algorithmic}
\bigskip

Now we provide the intuition behind the algorithm. 
Given $j \in [n]$ and a constant $d$, it is not guaranteed that at the $t$-th iteration of the alternating minimization algorithm
$$
U_{S_j^{t+1,L}}^{t,T} U_{S_j^{t+1,L}}^{t} = \sum_{i \in [n]: (i,j) \in \Omega_{t+1}} u_{i}^{t} u_{i}^{t,T}
$$ 
concentrates around its expectation 
\begin{align}
\mathbb{E}[U_{S_j^{t+1,L}}^{t,T} U_{S_j^{t+1,L}}^{t}] &= \frac{1}{{n \choose d}} \sum_{D \in \{S \subset [n]: |S| = d \}} U_D^{t, T} U_D^t  \nonumber\\
&= \frac{1}{{n \choose d}} \frac{ {n \choose d} d}{n} \sum_{i \in [n]} u_{i}^{t} u_{i}^{t,T} = \frac{d}{n} I. \nonumber
\end{align}
Some $U_{S_j^{t+1,L}}^{t,T} U_{S_j^{t+1,L}}^{t}$ might be ill-conditioned, namely, its least singular value is $0$ or closed to zero. 
If the matrix $U_{S_j^{t+1,L}}^{t,T} U_{S_j^{t+1,L}}^{t}$ is ill-conditioned, the results from the iteration (\ref{algo:stept1}) in the ``vanilla" alternating minimization algorithm might blow up. 
To prevent this adversarial scenario, we use the operations $\mathcal{T}_2$ to lift the small singular values and truncate the large singular values of $U_{S_j^{t+1,L}}^{t}$, $\forall j \in [n]$, before each row vector of $\tilde{V}^{t+1} = (\tilde{v}_{j}^{t+1,T}, 1\leq j \leq n)$ is computed. 
The convergence of the algorithm relies on the fact that w.h.p. the number of times the algorithm applies the operation $\mathcal{T}_2$ in each iteration is a small fraction of $n$. 
We will elaborate this point later in Theorem \ref{theorem:UpperBoundSbt}.
Also, the operators $\mathcal{T}_1$ are applied at the end of each iteration to guarantee the incoherence of the input $V^{t+1}$ (or $U^{t+1}$) for the next iteration while maintaining that $V^{t+1}$ (or $U^{t+1}$) is still close enough to $V^*$ (or $U^*$).

Our main result concerns the performance of the algorithm $\mathcal{TAM}$ under Assumption $1$ and under both Assumptions $1$ and $2$, respectively. We recall that $\mathcal{TAM}$ is parameterized by $\delta$ and $\beta$.
\begin{theorem}
	\label{mainTheorem}
	Suppose $M \in \mathbb{R}^{n \times n}$ is a rank-$k$ matrix satisfying Assumption $1$. Suppose the observed index set $\Omega$ is sampled according to the model $\mathcal{RRG}(d, n, N)$ in (\ref{RRG}).  
	Given any $\delta \in (0,1)$, $\beta \in (0, 1-\delta)$ and $\epsilon \in (0,2/3)$, there exists a $C(\delta, \beta)>0$ such that for 
	\begin{align}
	\label{inequality:dLowerBound1}
	d \geq C(\delta, \beta) k^4 \mu_0^2 \left(\frac{\sigma_1^*}{\sigma_k^*} \right)^2   + \frac{5 \mu_0 k (1+ \delta/3)}{\delta^2} \log \left( \frac{1}{\epsilon} \right)
	\end{align}
	and $N \geq 1 + \lceil \log(\frac{2}{\epsilon}) / \log 4 \rceil$, the $\mathcal{TAM}$ algorithm produces a matrix $M_N$ satisfying $\|M - M_N \|_F \leq \epsilon \|M\|_F$ w.h.p. 
	
	Furthermore, suppose $M$ satisfies both Assumptions $1$ and $2$. Then for $\delta \in (0 , 1)$ as defined in Assumption $2$ and $\beta \in (0, 1-\delta)$, the same result holds when  
	\begin{align}
	\label{inequality:dLowerBound2}
	d \geq C(\delta, \beta) k^4 \mu_0^2 \left(\frac{\sigma_1^*}{\sigma_k^*} \right)^2,
	\end{align}
	for the same constant $C(\delta, \beta)$ in (\ref{inequality:dLowerBound1}).             
\end{theorem}

Theorem \ref{mainTheorem} states that under Assumption $1$ the $\mathcal{TAM}$ algorithm produces a rank-$k$ $\epsilon$-approximation of matrix $M$ using $O(dn \log (1/\epsilon))$ samples for $d$ satisfying (\ref{inequality:dLowerBound1}). 
Furthermore, under both Assumption $1$ and Assumption $2$ the $\mathcal{TAM}$ algorithm produces a rank-$k$ $\epsilon$-approximation of matrix $M$ using $O(dn \log (1/\epsilon))$ samples for $d$ satisfying (\ref{inequality:dLowerBound2}). 

In terms of computational complexity, the cost in the initialization of $\mathcal{TAM}$ is mainly contributed by computing the top-$k$ left singular vectors of a sparse matrix $\frac{n}{d} P_{\Omega_0}(M) \in \mathbb{R}^{n \times n}$, which requires time $O(k |\Omega_0|)$ by exploiting the sparsity of $\frac{n}{d} P_{\Omega_0}(M)$ \cite{Mazumder2010}. 
In each iteration $t = 0, 1,\cdots, N - 1$, the cost is mainly contributed by computing $\sum_{i \in [n]: (i,j) \in \Omega_{t+1}} u_{i}^{t} u_{i}^{t,T} \in \mathbb{R}^{k \times k}$, $\forall \; j \in [n]$, $\sum_{j \in [n]: (i,j) \in \Omega_{N+t+1}} \hat{v}_{j}^{t+1} \hat{v}_{j}^{t+1, T} \in \mathbb{R}^{k \times k} $, $\forall \; i \in [n]$, at most $n$ SVD of $U_{S_j^{t+1,L}}^t \in \mathbb{R}^{d \times k}$, $\forall \; j \in [n]$, and at most $n$ SVD of $V_{S_i^{N+t+1, R}}^{t + 1} \in \mathbb{R}^{d \times k}$, $\forall \; i \in [n]$. Each component of the first two terms is the sum of $d$ $k$-by-$k$ matrices. 
Each matrix is the outer product of two $k$-by-$1$ vectors. 
Hence in each iteration it costs $O(dk^2 n)$ to compute the first two terms and $O(dk^2 n)$ to compute at most $2n$ SVD of $d$-by-$k$ matrices. By $|\Omega_0| = O(dn)$ and $N$ chosen as the lower bound given by Theorem \ref{mainTheorem}, the overall cost for $\mathcal{TAM}$ algorithm is  
\begin{align*}
O(k |\Omega_0|) + O(dk^2 n N) = O(dk^2 \log(1/\epsilon) n).
\end{align*}
Choosing the lower bound of $d$ given by (\ref{inequality:dLowerBound1}) or (\ref{inequality:dLowerBound2}) in Theorem \ref{mainTheorem}, $\mathcal{TAM}$ algorithm runs in linear time in $n$. 

\section{Analysis of the $\mathcal{TAM}$ algorithm} \label{section:analysis}
\subsection{Initialization}
The convergence of the $\mathcal{TAM}$ algorithm requires a warm start point $U^0$ close to the true $U^*$. To measure the closeness between two subspaces spanned by two matrices, we introduce the following definition of distance between subspaces. 
\begin{definition}
	\label{def:SubspaceDistance}
	\cite{Golub1996} Given any two matrices $X, Y \in \mathbb{R}^{n \times k}$, let $\hat{X}, \hat{Y} \in \mathbb{R}^{n \times k}$ be their corresponding orthonormal basis, and $\hat{X}_{\perp}, \hat{Y}_{\perp} \in \mathbb{R}^{n \times (n-k)}$ be any orthonormal basis of the orthogonal complement of $\hat{X}$ and $\hat{Y}$. Then the distance between the subspaces spanned by the columns of $X$ and $Y$ is defined by 
	\begin{align*}
	\dist(X, Y) \triangleq \| \hat{X}_{\perp}^T \hat{Y} \|_2.
	\end{align*} 
\end{definition}
The range of $\dist(\cdot,\cdot)$ is $[0,1]$.
Also, the distance $\dist(X, Y)$ defined above depends only on the spaces spanned by the columns of $X$ and $Y$, that is, $\text{Span}(X)$ and $\text{Span}(Y)$. Furthermore,
\begin{align}
\label{dist_symmetric}
& \dist(X, Y) = \dist(Y, X) \Rightarrow \| \hat{X}_{\perp}^T \hat{Y} \|_2 = \| \hat{Y}_{\perp}^T \hat{X} \|_2, \\
\label{dist_subspace_property3}
& \sigma_{min}(\hat{X}^T \hat{Y})^2 + \| \hat{X}_{\perp}^T \hat{Y} \|_2^2 = 1, \\
\label{dist_subspace_property4}
& \| \hat{X}_{\perp}^T \hat{Y} \|_2 = \| \hat{X} \hat{X}^T - \hat{Y} \hat{Y}^T \|_2.  
\end{align}
We refer to Theorem 2.6.1 in \cite{Golub1996} and its proof for the three properties above.  

We now obtain a bound on the distance $\dist(\bar{U}^0, U^*)$.  
\begin{lemma}
	\label{lemma:distanceBarU0UStar}
	Let $M$ be a rank-$k$ matrix that satisfies Assumption $1$. 
	Also, let $\Omega_0$ be as defined in $\mathcal{RRG}(d, n ,N)$ in (\ref{RRG}) and 
	$\bar{U}^0 = SVD(\frac{n}{d} P_{\Omega_0}(M), k)$ as defined in the first step of the $\mathcal{TAM}$ algorithm. 
	For any $C>0$ and $d \geq C k^4 \mu_0^2 (\sigma_1^*/\sigma_k^*)^2$, w.h.p. we have
	\begin{align}
	\label{distU0Ustar}
	\dist(\bar{U}^0, U^*) \leq \frac{14}{3 \sqrt{C}} \frac{1}{k}. 
	\end{align}
\end{lemma}
The proof of this lemma is similar to the proof of Lemma C.1. in \cite{Jain2012}. We give its proof in the Appendix \ref{ap:proof_of_lemma_distanceBarU0UStar} for completeness.

While $\bar{U}^0$ is close enough to $U^*$, $\bar{U}^0$ might not be incoherent. 
Hence, $\mathcal{TAM}$ algorithm implements the operation $\mathcal{T}_1$ on $\bar{U}^0$ in the truncation step to obtain an incoherent warm start $U^0$ for the iterations afterward. 
\begin{lemma}
	\label{lemma:TruncateMaintainIncoherence}
	Suppose $U^*$ satisfies Assumption $1$.  
	Let $\bar{U} \in \mathbb{R}^{n \times k}$ be an orthonormal matrix such that $\dist(\bar{U}, U^*) \leq \frac{1}{\phi k^{1/2}}$ for some $\phi \geq \frac{\sqrt{10}}{\sqrt{5} - 2}$.
	Let $\hat{U} = \mathcal{T}_1(\bar{U})$, and $U \in \mathbb{R}^{n \times k}$ be an orthonormal basis of $\hat{U}$. 
	Also, let $u_i^T \in \mathbb{R}^{1 \times k}$, $i \in [n]$, be the $i$-th row of $U$. Then  
	\begin{align}
	\label{incoherenceui}
	& \| u_i \|_2 \leq \sqrt{\frac{5 \mu_0 k}{n}} \quad \forall \; i \in [n],   \\
	\label{distU}
	& \dist(U, U^*) \leq \frac{\sqrt{10}}{\phi}. 
	\end{align}                
\end{lemma}
This lemma states that by applying the operator $\mathcal{T}_1$ to $\bar{U}$ and then orthonormalizing $\hat{U}$, the resultant matrix $U$ loses a factor $\sqrt{10} k^{1/2}$ in $\dist(\cdot, U^*)$ but gains the incoherence. Applying this lemma to $\bar{U}^0$, from Lemma \ref{lemma:distanceBarU0UStar} w.h.p. the corresponding $\phi$ is $\frac{ 3 \sqrt{C} k^{0.5} }{14}$. Choosing a large enough constant $C>0$ such that $\phi \geq \frac{\sqrt{10}}{\sqrt{5} - 2}$, this lemma implies that w.h.p. the following inequalities hold.
\begin{align}
\label{inequality:FirstIteration}
\| u_i^0 \|_2 \leq \sqrt{\frac{5 \mu_0 k}{n}} \quad \forall i \in [n] \quad \text{ and } \quad \dist(U^0, U^*) \leq \frac{14 \sqrt{10} }{3 \sqrt{C} k^{0.5}}. 
\end{align} 
We delay the proof of this lemma to Appendix \ref{ap:proof_of_lemma_TruncateMaintainIncoherence}.

\subsection{Convergence of the algorithm $\mathcal{TAM}$. Proof of Theorem \ref{mainTheorem}}

First we formulate the update of $\bar{V}^{t+1}$ at the $t$-th iteration in the algorithm $\mathcal{TAM}$ in a more compact form. For $j \in [n]$ and $\beta$ as given in the algorithm, let 
\begin{align}
\label{equation:HatBjHatCj}
\hat{B}^j &=  \begin{cases}
\frac{n}{d} \sum_{i: (i,j) \in \Omega_{t+1}} u_i^t u_i^{t,T} & \quad  \text{if } \frac{n}{d}\sigma_l (\sum_{i \in [n]: (i,j) \in \Omega_{t+1}} u_{i}^t u_{i}^{t,T}) \in  [\beta, 2 - \beta] \; \forall l \in [k]    \nonumber\\
\frac{n}{d} \sum_{i: (i,j) \in \Omega_{t+1}} \hat{u}_i^t \hat{u}_i^{t,T}  & \quad  \text{o.w. }  
\end{cases} \nonumber\\
\hat{C}^j  &=  \begin{cases}
\frac{n}{d} \sum_{i: (i,j) \in \Omega_{t+1}} u_i^{t} u_i^{*,T} & \quad  \text{if } \frac{n}{d}\sigma_l (\sum_{i \in [n]: (i,j) \in \Omega_{t+1}} u_{i}^t u_{i}^{t,T}) \in  [\beta, 2 - \beta] \; \forall l \in [k]     \\
\frac{n}{d} \sum_{i: (i,j) \in \Omega_{t+1}} \hat{u}_i^{t} u_i^{*,T}  & \quad  \text{o.w. }
\end{cases}
\end{align}
and 
\begin{align}
\label{expression_D}
D =  U^{t,T} U^{*}. 
\end{align} 
Using $M_{ij} = u_i^{*,T} \Sigma^* v_j^{*}$, we combine (\ref{algo:stept1}) and (\ref{algo:stept2}) for $j \in [n]$ at the $t$-th iteration and rewrite $\tilde{v}_j^{t+1}$ by
\begin{align} 
\label{equation:vjtplus1}
\tilde{v}_j^{t+1} &= (\hat{B}^j)^{-1} \hat{C}^j \Sigma^* v_j^*.        
\end{align} 
Then we rearrange the equation above as follows
\begin{align*}
\tilde{v}_j^{t+1,T} = v_j^{*,T} \Sigma^* D^T - v_j^{*,T} \Sigma^* \left(D^T \hat{B}^j - (\hat{C}^j)^T \right) (\hat{B}^j)^{-1} .
\end{align*}
Recall that $\tilde{V}^{t+1} \in \mathbb{R}^{n \times k}$ is a matrix with the $j$-th row equal to $\tilde{v}_j^{t+1, T}$ and the $QR$ decomposition $\tilde{V}^{t+1} = \bar{V}^{t+1} R^{t+1}$. We then rewrite the equation above in a more compact form 
\begin{align}
\label{equation:VtPlusOne}
\tilde{V}^{t+1} &= V^* \Sigma^* U^{*,T} U^t - F^t   \nonumber\\
\bar{V}^{t+1} &= \tilde{V}^{t+1} (R^{t+1})^{-1}
\end{align}
where 
\begin{align}
\label{equation:F}
F^t =  \left( \begin{array}{c}
v_{1}^{*,T} \Sigma^*  (D^T \hat{B}^1 - (\hat{C}^{1})^T) (\hat{B}^1)^{-1} \\
\vdots  \\
v_{n}^{*,T} \Sigma^*  (D^T \hat{B}^n - (\hat{C}^{n})^T) (\hat{B}^n)^{-1}   \end{array} \right).
\end{align}

Next we establish the geometric decay of the distance between the subspaces spanned by $V^{t+1}$ and $V^*$ and the distance between the subspaces spanned by $U^{t+1}$ and $U^*$. Then we use this property to conclude the proof of Theorem \ref{mainTheorem}. Our first step is to show an upper bound on the Frobenius norm of the error term $F^t$ in (\ref{equation:F}) for the $t$-th iteration in the algorithm $\mathcal{TAM}$.

\begin{theorem}
	\label{theorem:FtBound}
	Suppose $U^t$ satisfies 
	\begin{align}
	\label{uitIncoherence}
	& \| u_i^t \|_2 \leq  \sqrt{\frac{5 \mu_0 k}{n}},  \; \forall i \in [n]. 
	\end{align} 
	Let $F^t$ be the matrix as defined in (\ref{equation:F}), and $M$, $\Omega_{t+1}$, $\delta$, $\beta$ and $\epsilon$ be as defined in Theorem $\ref{mainTheorem}$. Then under Assumption $1$ and for $d$ satisfying (\ref{inequality:dLowerBound1}) w.h.p. we {}have
	\begin{align}
	\label{inequality: FtFrobeniusNormUpperBound}
	\| F^t/\sigma_k^*\|_F \leq \frac{1}{5 \sqrt{10k}} \max\{\dist(U^t, U^*), \epsilon/2 \}.
	\end{align}
	Also, under Assumptions $1$ and $2$ and for $d$ satisfying (\ref{inequality:dLowerBound2}), the inequality (\ref{inequality: FtFrobeniusNormUpperBound}) holds w.h.p.   
\end{theorem}

We delay the proof of this theorem to the next subsection. Our next step in proving Theorem \ref{mainTheorem} is to show the geometric decay property of the distance between the subspaces spanned by iterates $U^{t+1}$ ($V^{t+1}$) and $U^*$ ($V^*$). In order to prove Theorem \ref{mainTheorem}, we also need the following lemma which results from Definition $3.1$ and Proposition $3.2$ in \cite{Hardt2014}.
\begin{lemma}
	\label{lemma:tan_largest_principal_angle}
	Given two orthonormal matrices $X, Y \in \mathbb{R}^{n \times k}$, let $X_{\perp}, Y_{\perp} \in \mathbb{R}^{n \times (n-k)}$ be another two orthonormal matrices which span the orthogonal complements of $X$ and $Y$, respectively. Suppose $X^T Y$ is invertible. Then
	\begin{align*}
	\frac{\| X_{\perp}^T Y \|_2}{\sigma_k(X^T Y)} = \| X_{\perp}^T Y (X^T Y)^{-1}\|_2.
	\end{align*}
\end{lemma}
In this lemma we replaced the original $\| (I - X X^T) Y \|_2$ in \cite{Hardt2014} by $\| X_{\perp}^T Y \|_2$ due to the relation
\begin{align*}
\| (I - X X^T) Y \|_2 &= \| X_{\perp} X_{\perp}^T Y \|_2 \\
&= \sup_{v \in \mathbb{R}^n: \; \| v \|_2 = 1}\| v^T X_{\perp} X_{\perp}^T Y \|_2 \\
&= \sup_{v \in \text{span}(X_{\perp}): \; \| v \|_2 = 1}\| v^T X_{\perp} X_{\perp}^T Y \|_2 \\
&= \sup_{u \in \mathbb{R}^{n - k}: \; \|u\|_2 = 1}  \| u^T X_{\perp}^T Y \|_2 \\
&= \| X_{\perp}^T Y \|_2.
\end{align*}

\begin{theorem}
	\label{theorem:geometricDecay}
	Let $\epsilon$ be as defined in Theorem \ref{mainTheorem}. 
	Under Assumption $1$ and for $d$ satisfying (\ref{inequality:dLowerBound1}), w.h.p. the $(t+1)$th iterates $V^{t+1}$ and $U^{t+1}$ of algorithm $\mathcal{TAM}$ satisfy 
	\begin{align}
	\label{equation:localconvergence1}
	& \| v_j^{t+1} \|_2 \leq \sqrt{\frac{5 \mu_0 k}{n}} \quad \forall j \in [n], \nonumber \\
	& \dist(V^{t+1}, V^*) \leq \frac{1}{2} \max \{\dist(U^{t}, U^*), \epsilon/2 \}, \quad \forall \; t = 0, 1, \ldots, N-1, 
	\end{align}
	and 
	\begin{align}
	\label{equation:localconvergence2}
	& \| u_i^{t+1} \|_2 \leq \sqrt{\frac{5 \mu_0 k}{n}} \quad \forall i \in [n],  \nonumber \\
	& \dist(U^{t+1}, U^*) \leq \frac{1}{2} \max \{\dist(V^{t+1}, V^*), \epsilon/2 \}, \quad \forall \; t = 0, 1, \ldots, N-1. 
	\end{align} 
	Also, under Assumptions $1$ and $2$ and for $d$ satisfying (\ref{inequality:dLowerBound2}), w.h.p. the $(t+1)$th iterates $V^{t+1}$ and $U^{t+1}$ of algorithm $\mathcal{TAM}$ satisfy (\ref{equation:localconvergence1}) and (\ref{equation:localconvergence2}). 
\end{theorem}
\begin{proof}
	we first prove (\ref{equation:localconvergence1}) for both cases, and then use a similar argument to show (\ref{equation:localconvergence2}).
	Under Assumption $1$, we apply Lemma \ref{lemma:TruncateMaintainIncoherence} to $\bar{U}^0$ and obtain w.h.p. (\ref{inequality:FirstIteration}) in which we choose a large enough $C(\delta, \beta)$ such that $\dist(U^0, U^*)<1/3$. Then the following inequalities hold w.h.p. for $t = 0$. 
	\begin{align}
	\label{inequality:Iterationt}
	\| u_i^t \|_2 \leq \sqrt{\frac{5 \mu_0 k}{n}} \quad \forall i \in [n] \quad \text{and} \quad \dist(U^t, U^*)<\frac{1}{3}
	\end{align}
	Now we assume the inequality (\ref{inequality:Iterationt}) holds for some $t \geq 0$. 
	It follows from Theorem \ref{theorem:FtBound} that for both the case under Assumption $1$ and $d$ satisfying (\ref{inequality:dLowerBound1}), and the case under Assumptions $1$ and $2$ and $d$ satisfying (\ref{inequality:dLowerBound2}), the following inequality holds w.h.p. 
	\begin{align}
	\label{inequality:FtBoundFromTheorem3.10}
	\| F^t/\sigma_k^* \|_2 \leq \| F^t/\sigma_k^* \|_F  \leq \frac{1}{5\sqrt{10 k}} \max \{\dist(U^t, U^*), \epsilon/2 \}. 
	\end{align}
	
	Next we derive an upper bound on $\dist(\bar{V}^{t+1}, V^*)$. First we claim that $V^{*,T} \bar{V}^{t+1}$ is invertible. Using the expression of $\tilde{V}^{t+1}$ given by (\ref{equation:VtPlusOne}), we have
	\begin{align*}
	\sigma_k(V^{*,T} \tilde{V}^{t+1}) &= \sigma_k(V^{*,T}  (V^* \Sigma^* U^{*,T} U^t - F^t))   \\
	&= \sigma_k( \Sigma^* U^{*,T} U^t - V^{*,T} F^t)  
	\end{align*}
	Using Ky Fan singular value inequality in (\ref{KyFanSVInequality}) for $A = V^{*,T} F^t$, $B = \Sigma^* U^{*,T} U^t - V^{*,T} F^t$, $r=0$ and $t=k-1$, we have 
	\begin{align*}
	\sigma_k( \Sigma^* U^{*,T} U^t - V^{*,T} F^t) &\geq \sigma_k( \Sigma^* U^{*,T} U^t ) - \sigma_1(V^{*,T} F^t)  \\
	&\geq \sigma_k( \Sigma^* U^{*,T} U^t ) - \| F^t \|_2  \\
	&\geq \sigma_k^* \sigma_k( U^{*,T} U^t ) - \sigma_k^* \| F^t /\sigma_k^* \|_2 
	\end{align*}
	By the assumption $\dist(U^t, U^*)<1/3$, that is, $\| U_{\perp}^{*,T} U^t \|_2 < 1/3$ and the identity (\ref{dist_subspace_property3}), we have 
	\begin{align}
	\label{inequality:lowerbound_min_sv}
	\sigma_k(U^{*,T} U^t) = \sqrt{1 - \| U_{\perp}^{*,T} U^t \|_2^2} \geq 2\sqrt{2}/3
	\end{align}
	which, along with the upper bound on $\| F^t/\sigma_k^*\|_F$ in (\ref{inequality:FtBoundFromTheorem3.10}), gives
	\begin{align*}
	\sigma_k^* \sigma_k( U^{*,T} U^t ) - \sigma_k^* \| F^t /\sigma_k^* \|_2  \geq \frac{2\sqrt{2}}{3} \sigma_k^* -  \sigma_k^* \frac{1}{5\sqrt{10 k}} \max \{\dist(U^t, U^*), \epsilon/2 \} > 0.
	\end{align*}
	Then we have $\sigma_k(V^{*,T} \tilde{V}^{t+1})>0$ and hence $V^{*,T} \tilde{V}^{t+1}$ is invertible. 
	Also by QR decomposition $\tilde{V}^{t+1}=\bar{V}^{t+1} R^{t+1}$, we have $ V^{*,T} \tilde{V}^{t+1} = V^{*,T} \bar{V}^{t+1} R^{t+1} $. 
	Then $V^{*,T} \bar{V}^{t+1} \in \mathbb{R}^{k \times k}$ has rank $k$ and hence the claim follows.
	Then by Lemma \ref{lemma:tan_largest_principal_angle} where the claim we just proved verifies the assumption, we have
	\begin{align}
	\label{equation:VPerpBarV}
	\frac{\| V_{\perp}^{*,T} \bar{V}^{t+1} \|_2 }{\sigma_k(V^{*,T} \bar{V}^{t+1})} = \| V_{\perp}^{*,T} \bar{V}^{t+1} \left( V^{*,T} \bar{V}^{t+1} \right)^{-1} \|_2.
	\end{align}   
	First applying the second equation in (\ref{equation:VtPlusOne}) and then the first equation in (\ref{equation:VtPlusOne}), we obtain
	\begin{align}
	\| V_{\perp}^{*,T} \bar{V}^{t+1} \left( V^{*,T} \bar{V}^{t+1} \right)^{-1} \|_2 　&=  \| V_{\perp}^{*,T} \tilde{V}^{t+1} \left( V^{*,T} \tilde{V}^{t+1} \right)^{-1} \|_2  \nonumber \\
	&= \| V_{\perp}^{*,T} \tilde{V}^{t+1} \left( \Sigma^* U^{*,T} U^t - V^{*,T} F^t \right)^{-1} \|_2.  
	\end{align} 
	It follows from (\ref{inequality:lowerbound_min_sv}) that $U^{*,T} U^t$ is invertible. Hence
	\begin{align}
	\label{inequality:twoNormVstarBarV}
	\| V_{\perp}^{*,T} \bar{V}^{t+1} \left( V^{*,T} \bar{V}^{t+1} \right)^{-1} \|_2 &= \| V_{\perp}^{*,T} \tilde{V}^{t+1} \left( U^{*,T} U^t \right)^{-1} \left( \Sigma^* - V^{*,T} F^t (U^{*, T} U^t)^{-1} \right)^{-1} \|_2  \nonumber \\
	&\leq \| V_{\perp}^{*,T} \tilde{V}^{t+1} \left( U^{*,T} U^t \right)^{-1} \|_2 \|\left( \Sigma^* - V^{*,T} F^t (U^{*, T} U^t)^{-1} \right)^{-1} \|_2  \nonumber \\
	&\leq \frac{\| V_{\perp}^{*,T} \tilde{V}^{t+1} \left( U^{*,T} U^t \right)^{-1} \|_2}{\sigma_k(\Sigma^* - V^{*,T} F^t (U^{*, T} U^t)^{-1})}.
	\end{align} 
	Using the expression of $\tilde{V}^{t+1}$ in (\ref{equation:VtPlusOne}), the numerator of the right hand side above becomes 
	\begin{align*}
	\| V_{\perp}^{*,T} \tilde{V}^{t+1} \left( U^{*,T} U^t \right)^{-1} \|_2 &\leq \| V_{\perp}^{*,T} F^t \left( U^{*,T} U^t \right)^{-1} \|_2 \\
	& \leq \| V_{\perp}^{*,T} F^t \|_2 \|\left( U^{*,T} U^t \right)^{-1} \|_2  \\
	& \leq \frac{\| F^t \|_2}{\sigma_k(U^{*,T} U^t)}.
	\end{align*}
	Using Ky Fan singular value inequality in (\ref{KyFanSVInequality}) for $A = V^{*,T} F^t (U^{*, T} U^t)^{-1}$, $B = \Sigma^* - V^{*,T} F^t (U^{*, T} U^t)^{-1}$, $r=0$ and $t=k-1$, the denominator of the right hand side in (\ref{inequality:twoNormVstarBarV}) becomes 
	\begin{align*}
	\sigma_k(\Sigma^* - V^{*,T} F^t (U^{*, T} U^t)^{-1}) &\geq \sigma_k^* - \| V^{*,T} F^t (U^{*, T} U^t)^{-1} \|_2  \\
	&\geq \sigma_k^* - \| V^{*,T} F^t \|_2 \|(U^{*, T} U^t)^{-1} \|_2  \\
	&\geq \sigma_k^* - \frac{\| F^t \|_2}{\sigma_k(U^{*, T} U^t)}.
	\end{align*}
	Then (\ref{inequality:twoNormVstarBarV}) becomes
	\begin{align*}
	\| V_{\perp}^{*,T} \bar{V}^{t+1} \left( V^{*,T} \bar{V}^{t+1} \right)^{-1} \|_2 \leq \frac{\frac{\| F^t \|_2}{\sigma_k(U^{*,T} U^t)}}{\sigma_k^* - \frac{\| F^t \|_2}{\sigma_k(U^{*, T} U^t)}}.
	\end{align*}
	By $\sigma_k(U^{*,T} U^t) \geq 2\sqrt{2}/3 > 1/2$. Then 
	\begin{align*}
	\| V_{\perp}^{*,T} \bar{V}^{t+1} \left( V^{*,T} \bar{V}^{t+1} \right)^{-1} \|_2 \leq \frac{ 2 \| F^t \|_2 }{\sigma_k^* -  2 \| F^t \|_2 } = \frac{ 2  \| F^t / \sigma_k^* \|_2 }{1 -  2 \| F^t / \sigma_k^* \|_2 }.
	\end{align*}
	Using the upper bound on $\| F^t/\sigma_k^*\|_2$ in (\ref{inequality:FtBoundFromTheorem3.10}), we obtain
	\begin{align*}
	\| V_{\perp}^{*,T} \bar{V}^{t+1} \left( V^{*,T} \bar{V}^{t+1} \right)^{-1} \|_2 \leq & \frac{2/(5\sqrt{10 k}  )}{1-2 \max \{\dist(U^t, U^*), \epsilon/2 \} /(5\sqrt{10 k})} \max \{\dist(U^t, U^*), \epsilon/2 \}. 
	\end{align*}
	By $\dist(U^{t}, U^*) \in [0,1]$ and $\epsilon \in (0, 2/3)$, we have 
	\begin{align*}
	\| V_{\perp}^{*,T} \bar{V}^{t+1} \left( V^{*,T} \bar{V}^{t+1} \right)^{-1} \|_2 \leq & \frac{2/(5\sqrt{10 k}  )}{1-2 /(5\sqrt{10 k})}  \max \{\dist(U^t, U^*), \epsilon/2 \} \\
	\leq & \frac{1}{2 \sqrt{10k}} \max \{\dist(U^t, U^*), \epsilon/2 \}.
	\end{align*}
	Then it follows from (\ref{equation:VPerpBarV}) that
	$$
	\frac{\| V_{\perp}^{*,T} \bar{V}^{t+1} \|_2 }{\sigma_k(V^{*,T} \bar{V}^{t+1})} \leq  \frac{1}{2 \sqrt{10k}} \max \{\dist(U^t, U^*), \epsilon/2 \}.
	$$
	We have shown that $V^{*,T} \bar{V}^{t+1}$ is invertible and hence $\sigma_k(V^{*,T} \bar{V}^{t+1}) \in (0, 1]$ from which it follows that 
	$$
	\dist(\bar{V}^{t+1}, V^*) = \| V_{\perp}^{*,T} \bar{V}^{t+1} \|_2  \leq  \frac{1}{2 \sqrt{10k}} \max \{\dist(U^t, U^*), \epsilon/2 \}.
	$$
	Now we apply Lemma \ref{lemma:TruncateMaintainIncoherence} where $\bar{U}$ and $U^*$ are replaced by $\bar{V}^{t+1}$ and $V^*$, respectively, and by the inequality above $\phi = 2\sqrt{10} / \max \{\dist(U^t, U^*), \epsilon/2 \} $. Then by $\dist(U^t, U^*) < 1/3$ and $\epsilon/2 < 1/3$, we obtain 
	$\phi \geq 6 \sqrt{10}  \geq \sqrt{10}/(\sqrt{5} - 2)$. Thus (\ref{incoherenceui}) and (\ref{distU}) yield (\ref{equation:localconvergence1}), namely,
	$$
	\| v_j^{t+1} \|_2 \leq \sqrt{\frac{5 \mu_0 k}{n}} \quad \forall j \in [n] \quad \text{ and } \quad \dist(V^{t+1}, V^*) \leq \frac{1}{2} \max \{\dist(U^t, U^*), \epsilon/2 \}. 
	$$
	The second inequality above also implies $\dist(V^{t+1}, V^*) < 1/3$. Using 
	$$
	\| v_j^{t+1} \|_2 \leq \sqrt{\frac{5 \mu_0 k}{n}} \quad \forall j \in [n] \quad \text{ and } \quad \dist(V^{t+1}, V^*) < \frac{1}{3},
	$$ 
	(\ref{equation:localconvergence2}) is established similarly and then (\ref{inequality:Iterationt}) holds by replacing $t$ by $t+1$. By repeating the arguments above, (\ref{equation:localconvergence1}) and (\ref{equation:localconvergence2}) hold for all $t=0,1,\cdots,N-1$.
\end{proof}

We are now ready to prove Theorem \ref{mainTheorem}, assuming the validity of Theorem \ref{theorem:FtBound}.
\begin{proof}[Proof of Theorem \ref{mainTheorem}]
	By Theorem \ref{theorem:geometricDecay}, after $N \geq 1 + \log(2/\epsilon)/\log 4 $ iterations, we obtain
	\begin{align}
	\label{distU_N-1UStar}
	\dist(U^{N-1}, U^*) &\leq \frac{1}{2} \max \{\dist(V^{N-1}, V^*), \epsilon/2\} \nonumber \\
	&\leq \frac{1}{2} \max \bigg\{\frac{1}{2} \max \{\dist (U^{N-2}, U^*), \epsilon/2\}, \epsilon/2 \bigg\}  \nonumber\\
	&= \max \bigg\{ \frac{1}{4} \dist(U^{N-2}, U^*), \frac{\epsilon}{4} \bigg\} \nonumber \\
	&\; \vdots \nonumber\\
	&\leq \max \left\{ \left(\frac{1}{4} \right)^{N-1} \dist(U^{0}, U^*), \frac{\epsilon}{4} \right\} \nonumber \\
	&\leq \frac{\epsilon}{2},
	\end{align}
	and 
	\begin{align}
	\label{uiN-1Incoherence}
	\| u_i^{N-1} \|_2 \leq \sqrt{\frac{5 \mu_0 k}{n}} \quad \forall i \in [n].
	\end{align}
	Using the expression of $\tilde{V}^{N}$ in (\ref{equation:VtPlusOne}) for $t=N-1$, we obtain
	\begin{align*}
	\| M - U^{N-1} \tilde{V}^{N,T} \|_F  &= \| U^{*} \Sigma^* V^{*,T} - U^{N-1} ( U^{N-1,T} U^{*} \Sigma^* V^{*,T}   -F^{N-1,T}) \|_F \\
	&\leq \| (I - U^{N-1} U^{N-1,T}) U^{*} \Sigma^* V^{*,T}  \|_F  +  \| U^{N-1} F^{N-1,T} \|_F. 
	\end{align*}
	Using the inequality (\ref{inequality: matrix_inequality2}), we obtain
	\begin{align}
	\label{inequality:FrobeniusNormMU-N-1}
	\| M - U^{N-1} \tilde{V}^{N,T} \|_F  &\leq \| (I - U^{N-1} U^{N-1,T}) U^* \|_2  \|\Sigma^* V^{*,T}  \|_F  +  \| F^{N-1,T} \|_F   \nonumber \\                                     &= \| U_{\perp}^{N-1} U_{\perp}^{N-1,T} U^* \|_2  \|\Sigma^* V^{*,T}  \|_F  +  \| F^{N-1,T} /\sigma_k^{*}\|_F \, \sigma_k^{*}   
	\end{align}
	Then by the upper bound on $\dist(U^{N-1}, U^*)$ in (\ref{distU_N-1UStar}) and $\|\Sigma^* V^{*,T}  \|_F = \| M \|_F$, 
	\begin{align*}
	\| M - U^{N-1} \tilde{V}^{N,T} \|_F  \leq \frac{\epsilon}{2} \| M \|_F +  \| F^{N-1,T}/\sigma_k^{*} \|_F \| M \|_F.  
	\end{align*}
	By the incoherence of $U^{N-1}$ in (\ref{uiN-1Incoherence}), Theorem \ref{theorem:FtBound} implies that w.h.p. 
	$$
	\| F^{N-1,T}/\sigma_k^{*} \|_F \leq \frac{1}{5\sqrt{10 k}} \times \max \{\dist(U^{N-1}, U^*), \epsilon/2 \} \leq \frac{\epsilon}{10\sqrt{10 k}}. 
	$$
	Then w.h.p. the right hand side of the inequality (\ref{inequality:FrobeniusNormMU-N-1}) is upper bounded by 
	\begin{align*}
	& \leq \frac{\epsilon}{2} \| M \|_F +  \frac{\epsilon}{10\sqrt{10 k}}  \| M \|_F \\
	& \leq \epsilon \| M \|_F
	\end{align*}
	from which the result follows. 
\end{proof}

\subsection{Bounding the Frobenius norm $\| F^t / \sigma_k^* \|_F$. Proof of Theorem \ref{theorem:FtBound}}
We first introduce a theorem which gives an upper bound on the number of times at the $t$-th iteration of the algorithm $\mathcal{TAM}$ the operations $\mathcal{T}_2(\cdot, \beta)$ are applied to compute $\tilde{V}^{t+1}=(\tilde{v}_j^{t+1,T}, 1\leq j \leq n)$, and then use the upper bound given by this theorem to conclude the proof of Theorem \ref{theorem:FtBound}. 

Let $\beta, \delta$ be as defined in the algorithm $\mathcal{TAM}$.
Define 
\begin{align}
\label{Sbtdelta}
S_b^t(\beta) \triangleq & \Bigg\{j \in [n]: \bigg\| \frac{n}{d} \sum_{i: (i,j) \in \Omega_{t+1}} u_i^t u_i^{t, T} - I \bigg\|_2 > 1 - \beta \Bigg\}.
\end{align} 
The equivalence relation 
$$
\bigg\| \frac{n}{d} \sum_{i: (i,j) \in \Omega_{t+1}} u_i^t u_i^{t,T} - I \bigg \|_2 \leq 1 - \beta \iff \sigma_l \left(\frac{n}{d} \sum_{i: (i,j) \in \Omega_{t+1}} u_i^t u_i^{t, T} \right) \in [\beta, 2-\beta], \quad \forall l \in [k],
$$
implies that
$S_b^t(\beta)$ consists of all the `bad' indices $j \in [n]$ associated with $U_{S_j^{t+1,L}}^t$ to which the operation $\mathcal{T}_2(\cdot, \beta)$ is applied before $\tilde{v}_j^{t+1}$ is computed in (\ref{algo:stept2}). Let $\gamma_t = \dist(U^t, U^*)$, 
\begin{align}
\label{alpharhot}
\alpha = \frac{1-\beta-\delta}{12 \mu_0 k}, \quad \rho_t = \frac{2 k}{\frac{(1-\beta-\delta)^2}{24 \mu_0 k} - 3\gamma_t^2 \mu_0 k} 
\end{align}
and the function $f: \mathbb{N} \times \mathbb{R}^2 \rightarrow \mathbb{R}$  
\begin{align}
\label{equation:fd}
f(d, \mu, a) \triangleq 3 k \sqrt{\pi d} \exp \left( \frac{-a^2/2}{\mu k + \mu k a/3} d  \right).
\end{align}
For a large $C(\delta, \beta)>0$, it can be checked easily that $\rho_t > 0$ provided 
\begin{align}
\label{gamma_t_range}
\gamma_t \in \left(0, 4 \sqrt{10} \slash (\sqrt{C(\delta, \beta)} k^{1.5} \mu_0)\right).
\end{align}

The following theorem gives an upper bound on the size of $S_b^t(\beta)$. 
\begin{theorem}
	\label{theorem:UpperBoundSbt}
	Suppose Assumption $1$ holds and $U^t$ satisfies  
	\begin{align}
	\label{UtIncoherenceCondition}
	\| u_i^t \|_2 \leq \sqrt{\frac{5\mu_0 k}{n}} \quad \forall i \in [n].
	\end{align}
	Let $\delta$ and $\beta$ be as defined in $\mathcal{TAM}$. Then the following statements hold. 
	\begin{enumerate}[(a)]
		\item w.h.p. we have for any fixed $\zeta>0$,
		\begin{align}
		\label{inequality:SbtA0}
		|S_b^t(\beta)| \leq (1+\zeta) f\left(d, 5\mu_0, 1-\beta \right) n .
		\end{align}
		\item Suppose $\gamma_t$ satisfies (\ref{gamma_t_range}). w.h.p. we have for any fixed $\zeta>0$ and a large $C(\delta, \beta)>0$
		\begin{align}
		\label{inequality:SbtA1}
		|S_b^t(\beta)| \leq  \left( 1.1 \mathrm{e} \left( \frac{\mathrm{e}^2 \rho_t \gamma_t^2 }{\alpha} \right)^{\alpha d} + (1+\zeta)f(d, \mu_0, \delta) \right) n.
		\end{align}
		\item  Suppose $\gamma_t$ satisfies (\ref{gamma_t_range}) and Assumption $2$ also holds. w.h.p. we have for any fixed $\zeta>0$ and a large $C(\delta, \beta)>0$
		\begin{align}
		\label{inequality:SbtA1A2}
		|S_b^t(\beta)| \leq \left( 1.1 \mathrm{e} \left( \frac{\mathrm{e}^2 \rho_t \gamma_t^2 }{\alpha} \right)^{\alpha d} + \zeta \right) n.
		\end{align} 
	\end{enumerate}
\end{theorem}

We delay the proof of this theorem to the next subsection. We now prove Theorem \ref{theorem:FtBound}, assuming the validity of Theorem \ref{theorem:UpperBoundSbt}.
\begin{proof}[Proof of Theorem \ref{theorem:FtBound}] 
	We vectorize the rows of $F^t$ in (\ref{equation:F}) and then reassemble them one by one as a long vector $A \in \mathbb{R}^{kn \times 1}$
	\begin{align*}
	A =  \left( \begin{array}{c}
	(\hat{B}^1)^{-1} ( \hat{B}^1 D - \hat{C}^{1}) \Sigma^* v_{1}^{*} \\
	\vdots  \\
	(\hat{B}^n)^{-1} ( \hat{B}^n D - \hat{C}^{n}) \Sigma^* v_{n}^{*}  \end{array} \right).
	\end{align*}  
	Then $ \| F^t \|_F = \| A \|_2$. 
	For any $x^j \in \mathbb{R}^{ 1 \times k}$, $j \in [n]$, we have
	\begin{align*}
	(x^1, x^2, \ldots, x^n) A &= \sum_{j=1}^n x^j  (\hat{B}^j)^{-1} ( \hat{B}^j D - \hat{C}^{j}) \Sigma^* v_{j}^{*}.  
	\end{align*}
	Let 
	\begin{align}
	\label{BjCj}
	B^j = \frac{n}{d} \sum_{i: (i,j) \in \Omega_{t+1}} u_i^t u_i^{t,T} \quad \forall j \in [n] \quad \text{and} \quad C^j = \frac{n}{d} \sum_{i:(i,j) \in \Omega_{t+1}} u_i^t u_i^{*,T} \quad \forall j \in [n].    
	\end{align}
	Recall $\hat{B}^j$ and $\hat{C}^j$ defined in (\ref{equation:HatBjHatCj}), and that $S_b^t(\beta)$ consists of all the indices $j \in [n]$ associated with $U_{S_j^{t+1,L}}^t$ to which the operation $\mathcal{T}_2(\cdot, \beta)$ is applied. We have $\hat{B}^j = B^j$ and $\hat{C}^j = C^j$ for all $j \in [n] \backslash S_b^t(\beta)$. Then,
	\begin{align}
	\label{equation:xA}
	(x^1, x^2, \ldots, x^n) A &= \sum_{j=1}^n x^j  (\hat{B}^j)^{-1} ( B^j D - C^{j}) \Sigma^* v_{j}^{*} + \sum_{j \in S_b^t(\beta)}  x^j  (\hat{B}^j)^{-1} (\hat{B}^j - B^j) D \Sigma^* v_{j}^{*}   \nonumber \\
	& \quad + \sum_{j \in S_b^t(\beta)}  x^j  (\hat{B}^j)^{-1} (C^j - \hat{C}^j) \Sigma^* v_{j}^{*}.
	\end{align} 
	We will establish Theorem \ref{theorem:FtBound} from the following proposition, which gives upper bounds on the three terms on the right hand side of (\ref{equation:xA}), respectively. We delay its proof for later. 
	
	\begin{proposition}
		\label{proposition:BoundForFtFnorm}
		Suppose Assumption $1$ holds and $U^t$ satisfies 
		$$
		\| u_i^t \|_2 \leq \sqrt{\frac{5 \mu_0 k}{n}}, \; \forall i \in [n]. 
		$$
		Let $\delta$ and $\beta$ be as defined in Theorem \ref{mainTheorem} and $S_b^t(\beta)$ be as defined in (\ref{Sbtdelta}). Then for $d$ satisfying (\ref{inequality:dLowerBound2}) and all $x^j \in \mathbb{R}^{ 1 \times k}$, $j \in [n]$, satisfying $\| (x^1, x^2, \ldots, x^n) \|_2 = 1$ we have
		\begin{align}
		\label{Bound2ForFtFnorm}
		\sum_{j \in S_b^t(\beta)}  x^j  (\hat{B}^j)^{-1} (\hat{B}^j - B^j) D \Sigma^* v_{j}^{*} \leq (2-\beta + 5 \mu_0 k) \frac{\sigma_1^*}{\beta} \sqrt{\mu_0 k} \sqrt{\frac{|S_b^t(\beta)|}{n}},
		\end{align}
		\begin{align}
		\label{Bound3ForFtFnorm}
		\sum_{j \in S_b^t(\beta)}  x^j  (\hat{B}^j)^{-1} (C^j - \hat{C}^j) \Sigma^* v_{j}^{*} \leq \frac{7}{\beta} \sigma_1^* (\mu_0 k)^{1.5} \sqrt{\frac{|S_b^t(\beta)|}{n}}
		\end{align}
		and w.h.p. 
		\begin{align}
		\label{Bound1ForFtFnorm}
		\sum_{j=1}^n x^j  (\hat{B}^j)^{-1} ( B^j D - C^{j}) \Sigma^* v_{j}^{*} \leq \frac{ \sigma_k^* }{ 10 \sqrt{10 k} } \dist(U^t, U^*). 
		\end{align}
	\end{proposition}
	
	Applying Proposition \ref{proposition:BoundForFtFnorm} and then
	replacing the three terms on the right hand side of (\ref{equation:xA}) by their upper bounds provided by (\ref{Bound2ForFtFnorm}), (\ref{Bound3ForFtFnorm}) and (\ref{Bound1ForFtFnorm}), w.h.p. for $d$ satisfying (\ref{inequality:dLowerBound2}) we obtain an upper bound on $\| F^t /\sigma_k^* \|_F$
	\begin{align}
	\label{FtMidStep}
	& \| F^t /\sigma_k^* \|_F= \max_{(x^1, x^2, \ldots, x^n): \| (x^1, x^2, \ldots, x^n) \|_2 = 1} (x^1, x^2, \ldots, x^n) A/\sigma_k^*   \nonumber\\ 
	& \leq \frac{\gamma_t}{10 \sqrt{10k}} + \left(2-\beta + 5 \mu_0 k  \right) \frac{\sigma_1^*}{\beta \sigma_k^*} \sqrt{\mu_0 k} \sqrt{\frac{|S_b^t(\beta)|}{n}} + \frac{7}{\beta} \frac{\sigma_1^*}{\sigma_k^*} (\mu_0 k)^{1.5} \sqrt{\frac{|S_b^t(\beta)|}{n}}  \nonumber\\
	& \leq \frac{\gamma_t}{10 \sqrt{10k}} + \frac{14}{\beta} \frac{\sigma_1^*}{\sigma_k^*} (\mu_0 k)^{1.5} \sqrt{\frac{|S_b^t(\beta)|}{n}}. 
	\end{align}
	
	Next, we prove the upper bound on $\| F^t /\sigma_k^* \|_F$ in (\ref{inequality: FtFrobeniusNormUpperBound}) under Assumption $1$ and for $d$ satisfying (\ref{inequality:dLowerBound1}). We show this result for two cases: $\gamma_t \in [4\sqrt{10}/(\sqrt{C(\delta, \beta)} \mu_0 k^{1.5}), 1]$ and $\gamma_t \in (0, 4\sqrt{10}/(\sqrt{C(\delta, \beta)} \mu_0 k^{1.5}))$, respectively.
	Under Assumption $1$, the upper bound on $|S_b^t(\beta)|$ from (\ref{inequality:SbtA0}) in Theorem \ref{theorem:UpperBoundSbt} implies that w.h.p.
	\begin{align*}
	\frac{14}{\beta} \frac{\sigma_1^*}{\sigma_k^*} (\mu_0 k)^{1.5} \sqrt{\frac{|S_b^t(\beta)|}{n}} &\leq \frac{14}{\beta} \frac{\sigma_1^*}{\sigma_k^*} (\mu_0 k)^{1.5} \sqrt{1.1f(d, 5 \mu_0, 1-\beta)}.   
	\end{align*}
	Recall the definition of $f(d, 5 \mu_0, 1-\beta)$ in (\ref{equation:fd}). Then, 
	\begin{align*}
	& \quad \frac{14}{\beta} \frac{\sigma_1^*}{\sigma_k^*} (\mu_0 k)^{1.5} \sqrt{1.1f(d, 5 \mu_0, 1-\beta)} \\
	&= \frac{14}{\beta} \frac{\sigma_1^*}{\sigma_k^*} (\mu_0 k)^{1.5} \sqrt{3.3} \pi^{1/4} k^{1/2} d^{1/4} \exp\left( \frac{-(1-\beta)^2/4}{5\mu_0 k (1+(1-\beta)/3)} d \right) \\
	& = \frac{14 \sqrt{3.3} \pi^{1/4} }{\beta}  \exp \left( \log \left(\frac{\sigma_1^*}{\sigma_k^*} \right) + 1.5 \log \mu_0 + 2\log k + \frac{\log d}{4} -   \frac{(1-\beta)^2/4}{5\mu_0 k (1+(1-\beta)/3)} d \right)  
	\end{align*}
	For $d$ satisfying (\ref{inequality:dLowerBound1}), we observe that the last term inside $\exp(\cdot)$ above is a polynomial of $k$, $\mu_0$ and $\sigma_1^* \slash \sigma_k^* $ while other terms inside $\exp(\cdot)$ are linear combination of $\log k$, $\log \mu_0$ and $\log(\sigma_1^* \slash \sigma_k^*)$.  
	Hence we can choose a large $C(\delta, \beta)>0$ such that 
	\begin{align*}
	& \frac{14}{\beta} \frac{\sigma_1^*}{\sigma_k^*} (\mu_0 k)^{1.5} \sqrt{1.1f(d, 5 \mu_0, 1-\beta)} \leq \frac{1}{10\sqrt{10 k}} \frac{4\sqrt{10}}{\sqrt{C(\delta, \beta)} k^{1.5} \mu_0 }.
	\end{align*}
	Hence for a large $C(\delta, \beta)$ and $\gamma_t \in [4\sqrt{10}/(\sqrt{C(\delta, \beta)} \mu_0 k^{1.5}), 1]$, we have
	\begin{align*}
	\frac{\gamma_t}{10 \sqrt{10k}} + \frac{14}{\beta} \frac{\sigma_1^*}{\sigma_k^*} (\mu_0 k)^{1.5} \sqrt{\frac{|S_b^t(\beta)|}{n}} & \leq 
	\frac{\gamma_t}{10 \sqrt{10k}} + \frac{1}{10\sqrt{10 k}} \frac{4\sqrt{10}}{\sqrt{C(\delta, \beta)} k^{1.5} \mu_0 } \\
	& \leq \frac{\gamma_t}{10 \sqrt{10k}} + \frac{1}{10\sqrt{10 k}} \gamma_t = \frac{\gamma_t}{5 \sqrt{10k}},
	\end{align*}
	which, along with (\ref{FtMidStep}), gives the upper bound on $\| F^t /\sigma_k^* \|_F$ in (\ref{inequality: FtFrobeniusNormUpperBound}). 
	
	Next, we consider the case $\gamma_t \in (0, 4\sqrt{10}/(\sqrt{C(\delta, \beta)} \mu_0 k^{1.5}))$. Under Assumption $1$ and $\gamma_t \in (0, 4\sqrt{10}/(\sqrt{C(\delta, \beta)} \mu_0 k^{1.5}))$, the upper bound on $|S_b^t(\beta)|$ from (\ref{inequality:SbtA1}) in Theorem \ref{theorem:UpperBoundSbt} implies that w.h.p.  
	\begin{align*}
	\frac{14}{\beta} \frac{\sigma_1^*}{\sigma_k^*} (\mu_0 k)^{1.5} \sqrt{\frac{|S_b^t(\beta)|}{n}} &\leq \frac{14}{\beta} \frac{\sigma_1^*}{\sigma_k^*} (\mu_0 k)^{1.5} \sqrt{ 1.1 \mathrm{e} \left( \frac{\mathrm{e}^2 \rho_t \gamma_t^2 }{\alpha} \right)^{\alpha d} + 1.1f(d, \mu_0, \delta)  }  \\
	& \leq \frac{14}{\beta} \frac{\sigma_1^*}{\sigma_k^*} (\mu_0 k)^{1.5} \left( \sqrt{ 1.1 \mathrm{e}} \left( \frac{\mathrm{e}^2 \rho_t \gamma_t^2 }{\alpha} \right)^{\alpha d/2} + \sqrt{1.1f(d, \mu_0, \delta)} \right). 
	\end{align*}
	
	Our proof of Theorem \ref{theorem:FtBound} also relies on the following proposition, which gives upper bounds on the last two terms in the inequality above. The proof of this proposition, which involves heavy calculations, can be found in Appendix \ref{ap:proof_of_proposition_upperboundGammat_Epsilon}.  
	\begin{proposition}
		\label{proposition:upperboundGammat_Epsilon} {}
		Let $\alpha$ and $\rho_t$ be as defined in (\ref{alpharhot}), $f(d, \mu_0, \delta)$ be as defined in (\ref{equation:fd}),
		and $\epsilon$, $\delta$, $\beta$ be as defined in Theorem \ref{mainTheorem}. 
		Suppose $\gamma_t$ satisfies (\ref{gamma_t_range}). 
		There exists a large $C(\delta, \beta)>0$ such that if $d$ satisfies (\ref{inequality:dLowerBound2}), we have
		\begin{align}
		\label{inequality:Gamma_t_20}
		& \frac{14 }{\beta} \frac{\sigma_1^*}{\sigma_k^*} (\mu_0 k)^{1.5}    \sqrt{1.1 \mathrm{e}} \left ( \frac{\mathrm{e}^2 \rho_t \gamma_t^2}{\alpha}  \right)^{\alpha d/2} \leq \frac{\gamma_t}{20 \sqrt{10 k}},
		\end{align}
		and if $d$ satisfies (\ref{inequality:dLowerBound1}), we have
		\begin{align}
		\label{inequality:Epsilon_40}
		\frac{14}{\beta} \frac{\sigma_1^*}{\sigma_k^*}  (\mu_0 k)^{1.5}   \sqrt{1.1 f(d, \mu_0, \delta)} \leq \frac{\epsilon}{40 \sqrt{10 k}}.
		\end{align}
	\end{proposition}
	
	Since any $d$ satisfying the inequality (\ref{inequality:dLowerBound1}) also satisfies the inequality (\ref{inequality:dLowerBound2}), the two upper bounds (\ref{inequality:Gamma_t_20}) and (\ref{inequality:Epsilon_40}) in Proposition \ref{proposition:upperboundGammat_Epsilon} yield that for $d$ satisfying (\ref{inequality:dLowerBound1}), w.h.p.  
	\begin{align*}
	\frac{14}{\beta} \frac{\sigma_1^*}{\sigma_k^*} (\mu_0 k)^{1.5} \sqrt{\frac{|S_b^t(\beta)|}{n}} &\leq \frac{\gamma_t}{20 \sqrt{10 k}} + \frac{\epsilon}{40 \sqrt{10 k}},
	\end{align*}{}
	which, along with (\ref{FtMidStep}), gives the upper bound on $\| F^t /\sigma_k^* \|_F$ in (\ref{inequality: FtFrobeniusNormUpperBound})
	\begin{align*}
	\| F^t / \sigma_k^* \|_F \leq \frac{\gamma_t}{10\sqrt{10 k}} + \frac{\gamma_t}{20 \sqrt{10 k}} + \frac{\epsilon}{40 \sqrt{10 k}} \leq \frac{1}{5 \sqrt{10 k}} \max\{\gamma_t, \epsilon/2 \}. 
	\end{align*}
	This completes the proof of (\ref{inequality: FtFrobeniusNormUpperBound}) under Assumption $1$ for $d$ satisfying (\ref{inequality:dLowerBound1}). 
	
	Finally, we prove the upper bound on $\| F^t /\sigma_k^* \|_F$ in (\ref{inequality: FtFrobeniusNormUpperBound}) under Assumptions $1$, $2$ and for $d$ satisfying (\ref{inequality:dLowerBound2}). For $\gamma_t \in [4\sqrt{10}/(\sqrt{C(\delta, \beta)} \mu_0 k^{1.5}), 1]$, the upper bound on $\| F^t /\sigma_k^* \|_F$ in (\ref{inequality: FtFrobeniusNormUpperBound}) follows similarly using (\ref{inequality:SbtA0}) and (\ref{FtMidStep}). Suppose $\gamma_t \in (0, 4\sqrt{10}/(\sqrt{C(\delta, \beta)} \mu_0 k^{1.5}))$. Under Assumptions $1$ and $2$ and $\gamma_t \in (0, 4\sqrt{10}/(\sqrt{C(\delta, \beta)} \mu_0 k^{1.5}))$, the upper bound on $|S_b^t(\beta)|$ from (\ref{inequality:SbtA1A2}) in Theorem \ref{theorem:UpperBoundSbt} implies that for any fixed $\zeta>0$ w.h.p.  
	\begin{align*}
	\frac{14}{\beta} \frac{\sigma_1^*}{\sigma_k^*} (\mu_0 k)^{1.5} \sqrt{\frac{|S_b^t(\beta)|}{n}} &\leq \frac{14}{\beta} \frac{\sigma_1^*}{\sigma_k^*} (\mu_0 k)^{1.5} \sqrt{ 1.1 \mathrm{e} \left( \frac{\mathrm{e}^2 \rho_t \gamma_t^2 }{\alpha} \right)^{\alpha d} + \zeta }  \\
	& \leq \frac{14}{\beta} \frac{\sigma_1^*}{\sigma_k^*} (\mu_0 k)^{1.5} \sqrt{ 1.1 \mathrm{e}} \left( \frac{\mathrm{e}^2 \rho_t \gamma_t^2 }{\alpha} \right)^{\alpha d/2} + \frac{14}{\beta} \frac{\sigma_1^*}{\sigma_k^*} (\mu_0 k)^{1.5} \sqrt{\zeta}.
	\end{align*}
	We choose a small enough $\zeta>0$ such that 
	$$
	\frac{14}{\beta} \frac{\sigma_1^*}{\sigma_k^*} (\mu_0 k)^{1.5} \sqrt{\zeta} \leq \frac{\epsilon}{40 \sqrt{10 k}}.
	$$
	This is possible since $k$, $\sigma_1^* / \sigma_k^*$ and $\mu_0$ are assumed to be bounded from above by a constant.  
	The upper bound (\ref{inequality:Gamma_t_20}) in Proposition \ref{proposition:upperboundGammat_Epsilon} and the inequality above yield that w.h.p. 
	\begin{align*}
	\frac{14}{\beta} \frac{\sigma_1^*}{\sigma_k^*} (\mu_0 k)^{1.5} \sqrt{\frac{|S_b^t(\beta)|}{n}} &\leq \frac{\gamma_t}{20 \sqrt{10 k}} + \frac{\epsilon}{40 \sqrt{10 k}}. 
	\end{align*}
	Then similarly, the upper bound on $\| F^t /\sigma_k^* \|_F$ in (\ref{inequality: FtFrobeniusNormUpperBound}) follows. The proof of Theorem \ref{theorem:FtBound} is complete.
\end{proof}

\subsubsection{Proof of Proposition \ref{proposition:BoundForFtFnorm}}
We first prove (\ref{Bound2ForFtFnorm}). By the submultiplicative inequality for the spectral norm and Cauchy-Schwarz inequality, we have
\begin{align*}
\sum_{j \in S_b^t(\beta)}  x^j  (\hat{B}^j)^{-1} (\hat{B}^j - B^j) D \Sigma^* v_{j}^{*}  \leq& \max_{j\in [n]} \| (\hat{B}^j)^{-1} (\hat{B}^j - B^j) D \Sigma^* \|_2 \sum_{j \in S_b^t(\beta)}  \|x^j\|_2 \|v_{j}^{*}\|_2 \\
\leq & \max_{j\in [n]} \| (\hat{B}^j)^{-1} (\hat{B}^j - B^j) D \Sigma^* \|_2 \sqrt{\sum_{j \in S_b^t(\beta)}  \|x^j\|_2^2}  \sqrt{ \sum_{j \in S_b^t(\beta)} \|v_{j}^{*}\|_2^2}.
\end{align*}
By $\sum_{j \in [n]} \|x^j\|_2^2 = 1$ and Assumption $1$ on the incoherence of $V^*$, we have
$$
\sum_{j \in S_b^t(\beta)} \|x^j\|_2^2 \leq 1 \quad \text{and} \quad \sum_{j \in S_b^t(\beta)} \|v_{j}^{*}\|_2^2 \leq |S_b^t(\beta)| \frac{\mu_0 k}{n}.
$$
Next, we have
$$
\max_{j\in [n]} \| (\hat{B}^j)^{-1} (\hat{B}^j - B^j) D \Sigma^* \|_2 \leq \max_{j \in [n]} \| (\hat{B}^j)^{-1} \|_2 \max_{j \in [n]} \{\| \hat{B}^j \|_2 + \| B^j \|_2 \} \| D \|_2 \sigma_1^*.
$$ 
Recall $\hat{B}^j$ given by (\ref{equation:HatBjHatCj}). 
Then by $\sigma_l(\hat{B}^j) \in [\beta, 2-\beta]$ for all $l \in [k]$ and all $j \in [n]$ and $D = U^{t, T} U^*$ where $U^t$ and $U^*$ are both orthonormal matrices, we have
\begin{align}
\label{norm_bounds} 
\max_{j \in [n]} \| (\hat{B}^j)^{-1} \|_2  \leq  \frac{1}{\beta}, \quad \max_{j \in [n]} \| \hat{B}^j \|_2  \leq  2 - \beta, \quad \| D \|_2 \leq 1.
\end{align}
Also recall $B^j$ given by (\ref{BjCj}) and the incoherence assumption $\| u_i^t \|_2 \leq \sqrt{5 \mu_0 k / n}$, $\forall i \in [n]$. Then we have the following upper bound on $\| B^j \|_2$
$$
\quad \| B^j \|_2 \leq \frac{n}{d} d \max_{i \in [n]}\|u_i^t\|_2^2 \leq \frac{n}{d}  d \frac{ 5\mu_0 k}{n} = 5 \mu_0 k, \; \forall j \in [n]. 
$$ 
Then 
\begin{align*}
\max_{j\in [n]} \| (\hat{B}^j)^{-1} (\hat{B}^j - B^j) D \Sigma^* \|_2 \leq \frac{1}{\beta} (2-\beta + 5 \mu_0 k) \sigma_1^*.
\end{align*}
Combining the inequalities above, we obtain
\begin{align}
\label{inequlity:UpperBoundSecondTerm}
\sum_{j \in S_b^t(\beta)}  x^j  (\hat{B}^j)^{-1} (\hat{B}^j - B^j) D \Sigma^* v_{j}^{*} \leq &  \left(2-\beta + 5 \mu_0 k  \right) \frac{\sigma_1^*}{\beta} \sqrt{\mu_0 k} \sqrt{\frac{|S_b^t(\beta)|}{n}}. 
\end{align} 
This proves (\ref{Bound2ForFtFnorm}). Next, we prove (\ref{Bound3ForFtFnorm}). Similarly, we have
\begin{align*}
\sum_{j \in S_b^t(\beta)}  x^j  (\hat{B}^j)^{-1} (C^j - \hat{C}^j) \Sigma^* v_{j}^{*}  \leq   \max_{j \in [n]} \{ \| C^j\|_2 + \| \hat{C}^j \|_2\} \frac{ \sigma_1^* }{\beta} \sqrt{\mu_0 k} \sqrt{\frac{|S_b^t(\beta)|}{n}}.  
\end{align*}
It follows from $C^j$ given by (\ref{BjCj}) and $\hat{C}^j$ given by (\ref{equation:HatBjHatCj}) that 
$$
C^j = \frac{n}{d} U_{S_j^{t+1,L}}^{t,T} U_{S_j^{t+1,L}}^* \quad \text{ and } \quad \hat{C}^j = \frac{n}{d} \hat{U}_{S_j^{t+1,L}}^{t,T} U_{S_j^{t+1,L}}^*. 
$$
Also by the definition of $\mathcal{T}_2(\cdot, \beta)$ in (\ref{T2onUS}) we have $\|\hat{U}_{S_j^{t+1,L}}^t\|_2 = \| \mathcal{T}_2 (U_{S_j^{t+1,L}}^t, \beta)\|_2 \leq \sqrt{(2-\beta) d/n}$, which, together with Assumption $1$ on the incoherence of $U^*$ and the incoherence condition $\| u_i^t \|_2 \leq \sqrt{5\mu_0 k / n}$, $\forall i \in [n]$, gives 
\begin{align*}
\| C^j \|_2 &\leq \frac{n}{d} \| U_{S_j^{t+1,L}}^t \|_2 \| U_{S_j^{t+1,L}}^*\|_2  \\
&\leq \frac{n}{d} \| U_{S_j^{t+1,L}}^t \|_F \| U_{S_j^{t+1,L}}^*\|_F   \\
&\leq \frac{n}{d} \sqrt{d \frac{5 \mu_0 k}{n}} \sqrt{d \frac{\mu_0 k}{n}} = \sqrt{5} \mu_0 k, \quad \forall j \in [n],
\end{align*}
and
\begin{align*}
\| \hat{C}^j \|_2 &\leq \frac{n}{d} \| \hat{U}_{S_j^{t+1,L}}^t \|_2 \| U_{S_j^{t+1,L}}^*\|_2  \\
&\leq \frac{n}{d} \| \hat{U}_{S_j^{t+1,L}}^t\|_2 \| U_{S_j^{t+1,L}}^*\|_F   \\
&\leq \frac{n}{d} \sqrt{\frac{d}{n} (2-\beta)} \sqrt{d \frac{\mu_0 k}{n}} = \sqrt{(2-\beta) \mu_0 k}, \quad \forall j \in [n].
\end{align*}
Hence
\begin{align}
\label{inequlity:UpperBoundThirdTerm}
\sum_{j \in S_b^t(\beta)}  x^j  (\hat{B}^j)^{-1} ( C^j - \hat{C}^j ) \Sigma^* v_{j}^{*} \leq &  ( \sqrt{5} \mu_0 k + \sqrt{(2-\beta) \mu_0 k} ) \frac{ \sigma_1^* }{\beta} \sqrt{\mu_0 k} \sqrt{\frac{|S_b^t(\beta)|}{n}}  \nonumber\\  
\leq &  \frac{7}{\beta} \sigma_1^* (\mu_0 k)^{1.5} \sqrt{\frac{|S_b^t(\beta)|}{n}}.
\end{align}

Finally, we prove (\ref{Bound1ForFtFnorm}). Let 
\begin{align}
\label{definition:yjtildevj}
y^{j} = x^j (\hat{B}^j)^{-1}, \quad \tilde{v}_j^* = \Sigma^{*} v_j^{*} \text{ and } J_i = u_i^t u_i^{t,T} U^{t,T} U^{*} - u_i^t u_i^{*, T} \; \forall \; i \in [n].
\end{align}
Then 
\begin{align*}
B^j D - C^j &= \frac{n}{d} \sum_{i: (i,j) \in \Omega_{t+1}} (u_i^t u_i^{t,T} U^{t,T} U^{*} - u_i^t u_i^{*, T})   \\
&= \frac{n}{d} \sum_{i: (i,j) \in \Omega_{t+1}} J_i.
\end{align*} 
We can rewrite the left hand side of (\ref{Bound1ForFtFnorm}) then as follows 
\begin{align}
\label{equation:xA1}
\frac{n}{d} \sum_{j=1}^n  \sum_{i: (i,j) \in \Omega_{t+1}} y^j J_i \tilde{v}_j^* =  \frac{n}{d} \sum_{(i,j): \; (i,j) \in \Omega_{t+1}} y^j J_i \tilde{v}_j^*.
\end{align}
Also, let $y_h^j$, $h \in [k]$, be the $h$-th entry of $y^j \in \mathbb{R}^{1 \times k}$, $\tilde{v}_{jl}^*$, $l \in [k]$, be the $l$-th entry of $\tilde{v}_{j}^* \in \mathbb{R}^{k \times 1}$ and $(J_i)_{hl}$, $h,l \in [k]$, be the $(h,l)$ entry of the matrix $J_i \in \mathbb{R}^{k \times k}$. Then the right-hand side of (\ref{equation:xA1}) is
\begin{align}
\label{equation:xA2}
\frac{n}{d} \sum_{h,l \in [k]} \sum_{(i,j):(i,j) \in \Omega_{t+1}} y_h^j \tilde{v}_{jl}^* (J_i)_{hl}. 
\end{align}
Let $G_n \in \mathbb{R}^{n \times n}$ be the biadjacency matrix of the random bipartite $d$-regular graph associated with the index set $\Omega_{t+1}$. Also, let $\mathcal{J}^{hl} \in \mathbb{R}^{1 \times n}$, $h, l \in [k]$, be
$$
\mathcal{J}^{hl} = ((J_1)_{hl}, (J_2)_{hl}, \ldots, (J_n)_{hl}),
$$
$\mathcal{L}^{hl} \in \mathbb{R}^{1 \times n}$, $h, l \in [k]$, be 
$$
\mathcal{L}^{hl} = (y_h^1 \tilde{v}_{1l}^*, y_h^2 \tilde{v}_{2l}^*, \ldots, y_h^n \tilde{v}_{nl}^*), 
$$
$\mathcal{J} \in \mathbb{R}^{1 \times k^2 n}$ be
\begin{align*}
\mathcal{J} = (\mathcal{J}^{11}, \ldots, \mathcal{J}^{1k}, \mathcal{J}^{21}, \ldots, \mathcal{J}^{2k}, \ldots, \mathcal{J}^{k1}, \ldots, \mathcal{J}^{kk}),
\end{align*}
$\mathcal{L} \in \mathbb{R}^{1 \times k^2 n}$ be
$$
\mathcal{L} = (\mathcal{L}^{11}, \ldots, \mathcal{L}^{1k}, \mathcal{L}^{21}, \ldots, \mathcal{L}^{2k}, \ldots, \mathcal{L}^{k1}, \ldots, \mathcal{L}^{kk}),
$$
and $I_{k^2} \in \mathbb{R}^{k^2 \times k^2}$ be an identity matrix. 
Denote $\otimes$ the Kronecker product. Then we rewrite (\ref{equation:xA2}) by
{
	\begin{align}
	\label{equation:xA3}
	& \frac{n}{d} (\mathcal{J}^{11}, \ldots, \mathcal{J}^{1k}, \mathcal{J}^{21}, \ldots, \mathcal{J}^{2k}, \ldots, \mathcal{J}^{k1}, \ldots, \mathcal{J}^{kk}) \nonumber \\
	& \qquad \times \left(I_{k^2} \otimes G_n \right) (\mathcal{L}^{11}, \ldots, \mathcal{L}^{1k}, \mathcal{L}^{21}, \ldots, \mathcal{L}^{2k}, \ldots, \mathcal{L}^{k1}, \ldots, \mathcal{L}^{kk})^T  \nonumber \\
	= & \frac{n}{d} \mathcal{J} \left(I_{k^2} \otimes G_n \right) \mathcal{L}^T.
	\end{align}
}
Observe $I_{k^2} \otimes G_n$ is a block diagonal matrix in which each block is $G_n$. Let $U_1$ be the top left sigular vector of $G_n$. Then by property $P_1$ of the random bipartite $d$-regular graph, $U_1=[1/\sqrt{n}, 1/\sqrt{n}, \cdots, 1/\sqrt{n}]^T$. Hence the top $k^2$ left singular vectors of $I_{k^2} \otimes G_n$ are $\mathbf{e}_{i} \otimes U_1$, $\forall i \in [k^2]$, where $\mathbf{e}_i \in \mathbb{R}^{k^2 \times 1}$ is the $i$-th unit vector, that is, its $i$-th entry is one and all others are zero. 

Let $X_i, Y_i \in \mathbb{R}^{k^2 n \times 1}$, $i \in [k^2 n]$, be the $i$-th left singular vector and the $i$-th right singular vector of $I_{k^2} \otimes G_n$, respectively, and $\sigma_i$, $i \in [k^2 n]$, be the $i$-th singular value of $I_{k^2} \otimes G_n$. Then we can rewrite (\ref{equation:xA3}) 
\begin{align}
\label{FtFirstTermUpperBoundStep1}
& \frac{n}{d}  \left(\sum_{i=1}^{k^2}\sigma_i  (\mathcal{J} X_i) (\mathcal{L} Y_i) +  \sum_{i=k^2+1}^{k^2 n}\sigma_i (\mathcal{J} X_i) (\mathcal{L} Y_i)  \right)   \nonumber \\
=& \frac{n}{d}  \left(\sum_{i=1}^{k^2}\sigma_i \langle \mathcal{J}, \mathbf{e}_{i} \otimes U_1 \rangle (\mathcal{L} Y_i) +  \sum_{i=k^2+1}^{k^2 n}\sigma_i (\mathcal{J} X_i) (\mathcal{L} Y_i)  \right).
\end{align}   
Note that 
$$
\sum_{i \in [n]} J_i = U^{t,T} U^t U^{t,T} U^{*} - U^{t,T} U^{*} = 0.
$$
Then for all $h,l\in[k]$ we have
\begin{align}
\label{equality:JhlEntrySum}
\sum_{i \in [n]} (J_i)_{hl} = 0.
\end{align}
Hence the entry sum of $\mathcal{J}^{hl}$ for all $h,l\in [k]$ is $0$, which yields
$$
(\mathcal{J}^{11}, \ldots, \mathcal{J}^{1k}, \mathcal{J}^{21}, \ldots, \mathcal{J}^{2k}, \ldots, \mathcal{J}^{k1}, \ldots, \mathcal{J}^{kk}) \left( \mathbf{e}_{i} \otimes U_1 \right) = 0, \; \forall i \in [k^2],
$$   
that is, $\langle \mathcal{J}, \mathbf{e}_{i} \otimes U_1 \rangle=0, \; \forall i \in [k^2]$. Then the right hand side of (\ref{FtFirstTermUpperBoundStep1}) becomes 
\begin{align}
\label{equation:xA4}
\frac{n}{d}   \sum_{i=k^2+1}^{k^2 n}\sigma_i (\mathcal{J} X_i) (\mathcal{L} Y_i).  
\end{align}  
Also by the property $P_2$ of random bipartite $d$-regular graph, the top $k^2$ singular values of $I_{k^2} \otimes G_n$ are all $d$, and the remaining singular values are upper bounded by $(7\sqrt{d})/3$ w.h.p. Then w.h.p. we have
\begin{align}
\label{equation:xA5}
\frac{n}{d}   \sum_{i=k^2+1}^{k^2 n} \sigma_i (\mathcal{J} X_i) (\mathcal{L} Y_i) &\leq \frac{n}{d} \sum_{i=k^2+1}^{k^2 n} \sigma_i |\mathcal{J} X_i | | \mathcal{L} Y_i |  \nonumber\\
&\leq \frac{n}{d} \frac{7 \sqrt{d}}{3} \sqrt{\sum_{i=k^2+1}^{k^2 n} |\mathcal{J} X_i|^2 } \sqrt{\sum_{i=k^2+1}^{k^2 n} |\mathcal{L} Y_i|^2 } \nonumber\\
&\leq \frac{n}{d} \frac{7 \sqrt{d}}{3} \| \mathcal{J} \|_2 \| \mathcal{L} \|_2.
\end{align}  
Now we bound $\| \mathcal{J} \|_2 \|$ and $\| \mathcal{L} \|_2$ separately. 
Let $u_{ih}^t$, $h \in [k]$, be the $h$-th entry of $u_i^t \in \mathbb{R}^{k \times 1}$, $u_{il}^*$, $l \in [k]$, be the $l$-th entry of $u_i^* \in \mathbb{R}^{k \times 1}$ and $U_l^* \in \mathbb{R}^{n \times 1}$, $l \in [k]$, be the $l$-th column of $U^*$. Then, 
\begin{align*}
\| \mathcal{J} \|_2^2 = \sum_{h,l\in[k]}\sum_{i \in [n]} (J_i)_{h,l}^2 =& \sum_{h,l\in[k]} \sum_{i \in [n]} (u_{ih}^t  u_i^{t,T} U^{t,T} U_l^* - u_{ih}^t u_{il}^* )^2   \\
=& \sum_{l\in[k]} \sum_{i \in [n]} \sum_{h \in [k]} (u_{ih}^t)^2 ( u_i^{t,T} U^{t,T} U_l^* - u_{il}^*)^2 \\
\leq & \max_{i \in [n]} \|u_i^t\|_2^2 \sum_{l \in [k]} \sum_{i \in [n]} ( u_i^{t,T} U^{t,T} U_l^* - u_{il}^*)^2. 
\end{align*}
Since $U^t$ and $U^* \in \mathbb{R}^{n \times k}$ are both orthonormal matrices, we have
\begin{align*}
\sum_{l\in[k]} \sum_{i \in [n]} ( u_i^{t,T} U^{t,T} U_l^* - u_{il}^*)^2 & = \sum_{l\in[k]} \sum_{i\in [n]} \left(U_l^{*,T} U^t u_i^t u_i^{t,T} U^{t,T} U_l^* - 2 u_{i,l}^* u_i^{t,T} U^{t,T} U_l^* + (u_{il}^*)^2 \right)   \\
& = \sum_{l\in[k]} \left( U_l^{*,T} U^t U^{t,T} U_l^* - 2 U_l^{*,T} U^t U^{t,T} U_l^* + 1 \right) \\
& = \sum_{l\in[k]} \left(1 -  U_l^{*,T} U^t U^{t,T} U_l^* \right) \\
& = \sum_{l\in[k]} \left(1 - \| U^{t,T} U_l^* \|_2^2 \right)  \\
& \leq \sum_{l\in[k]} \left( 1 - (\sigma_{\min} (U^{t,T} U^*))^2 \right)  \\
& = k \left( 1 - (\sigma_{\min} (U^{t,T} U^*))^2 \right) 
\end{align*}
Also by the subspace distance property (\ref{dist_subspace_property3}), we have
$$
1 - (\sigma_{\min} (U^{t,T} U^*))^2 = \dist(U^t, U^*)^2
$$ 
which gives 
\begin{align*}
\sum_{l\in[k]} \sum_{i \in [n]} ( u_i^{t,T} U^{t,T} U_l^* - u_{il}^*)^2 \leq k \dist(U^t, U^*)^2. 
\end{align*}
Then using the incoherence assumption $\| u_i^t \|_2 \leq \sqrt{5 \mu_0 k / n}$, $\forall i \in [n]$, we obtain 
\begin{align*}
\| \mathcal{J} \|_2^2 \leq & \frac{5 \mu_0 k^2}{n}  \dist(U^t, U^*)^2.
\end{align*}        
Next, we bound $\| \mathcal{L} \|_2$. It follows from $y^j$ and $\tilde{v}_j$ given in (\ref{definition:yjtildevj}) and Assumption $1$ that 
$$
\sum_{l \in [k]}(\tilde{v}_{jl}^*)^2 = \| \tilde{v}_j^* \|_2^2 = \| \Sigma^* v_j^* \|_2^2 \leq (\sigma_1^*)^2 \frac{\mu_0 k}{n} $$
and
$$\quad \sum_{h \in [k]} (y_h^j)^2 = \| y^j \|_2^2 = \| x^j (\hat{B}^j)^{-1} \|_2^2 \leq \frac{\| x^j \|_2^2}{\beta^2} 
$$
where in the last inequality we used (\ref{norm_bounds}).
Then recalling $\sum_{j \in [n]} \|x^j\|_2^2 = 1$ we have
\begin{align*}
\| \mathcal{L} \|_2^2 = \sum_{h,l \in [k]} \sum_{j \in [n]} (y_{h}^j)^2 (\tilde{v}_{jl}^*)^2  \leq \sum_{j \in [n]} \frac{\|x^j\|_2^2}{\beta^2} (\sigma_1^*)^2 \frac{\mu_0 k}{n} = \frac{ (\sigma_1^*)^2 }{\beta^2} \frac{\mu_0 k}{n}. 
\end{align*} 
Finally, we obtain w.h.p.
\begin{align*}
\sum_{j=1}^n x^j (\hat{B}^j)^{-1} ( B^j D - C^j) \Sigma^* v_j^* & \leq \frac{n}{d} \frac{7 \sqrt{d}}{3} \| \mathcal{J} \|_2 \| \mathcal{L} \|_2 \\
& \leq \frac{n}{d} \frac{7\sqrt{d}}{3}  \sqrt{\frac{5 \mu_0 k^2}{n}} \dist(U^t, U^*) \frac{ \sigma_1^*}{\beta} \sqrt{\frac{\mu_0 k}{n}}  \\
& = \frac{7\sqrt{5}}{3 \beta} \frac{k^{1.5} \mu_0}{\sqrt{d} } \sigma_1^* \dist(U^t, U^*).
\end{align*}
Then for $d \geq C(\delta, \beta) k^4 \mu_0^2 (\sigma_1^* / \sigma_k^*)^2$ we can choose a large $C(\delta, \beta)>0$ such that w.h.p.
\begin{align}
\label{inequlity:UpperBoundFirstTerm}
\sum_{j=1}^n x^j (\hat{B}^j)^{-1} ( B^j D - C^j) \Sigma^* v_j^* \leq  \frac{\sigma_k^*}{10 \sqrt{10k}} \dist(U^t, U^*). 
\end{align}
The proof of Proposition \ref{proposition:BoundForFtFnorm} is complete.

\subsection{Bounding the size of $S_b^t(\beta)$. Proof of Theorem \ref{theorem:UpperBoundSbt}}

First, we claim that there exists an orthonormal matrix $R \in \mathbb{R}^{k \times k}$ such that $U^{*,T} U^t R$ is symmetric. Indeed, suppose the SVD of $U^{*,T} U^t$ is 
$$
U^{*,T} U^t = W_1 \Sigma W_2^T
$$ 
where $W_1, W_2 \in \mathbb{R}^{k \times k}$ are two orthonormal matrices. Right-multiplying both sides of the equation above by $W_2 W_1^T$, we obtain
\begin{align}
\label{equation:UstarUW}
U^{*,T} U^t W_2 W_1^T = W_1 \Sigma W_1^T. 
\end{align} 
Observe $W_2 W_1^T \in \mathbb{R}^{k \times k}$ is an orthonormal matrix and then the claim follows by taking $R = W_2 W_1^T$.

Note the definition of $S_b^t(\beta)$ in (\ref{Sbtdelta}). If we replace $U^t$ by $U^t R$, it can be checked easily that the index set $S_b^t(\beta)$, $\gamma_t$ and $\rho_t$ given in (\ref{alpharhot}) are unchanged.
In the remaining part of this subsection, we will use $U^t R$ instead of $U^t$ to derive an upper bound on $|S_b^t(\beta)|$. 
We will still denote $U^t R$ by $U^t$ for convenience.
Now $U^{*,T} U^t$ is symmetric.   

For $\tau \in (0, 1)$, let the set $Q^t(\tau)$ be
\begin{align*}
Q^t(\tau) \triangleq \left\{i \in [n]: \| u_i^t u_i^{t, T} - u_i^{*} u_i^{*,T} \|_2 > \frac{\tau}{n} \right\}. 
\end{align*}   
Our first step is to show an upper bound on the size of $Q^t(\tau)$ when $\dist(U^t, U^*)$ is small. 
\begin{lemma}
	\label{lemma:boundQ}
	Suppose Assumption $1$ holds. Let $\gamma_t = \dist(U^t, U^*)$. Then for any $\tau \in (0, 1)$,
	\begin{align}
	\label{upperBoundQdelta}
	\left(\frac{\tau^2}{6 \mu_0 k} - 3 \gamma_t^2 \mu_0 k \right) |Q^t(\tau)| \leq 2 k \gamma_t^2 n.
	\end{align}
\end{lemma}  
For $\gamma_t < \frac{\tau}{\sqrt{18} \mu_0 k}$, the coefficient of $|Q^t(\tau)|$ above is positive. Then the inequality above implies an upper bound on the size of $Q^t(\tau)$ 
\begin{align*}
|Q^t(\tau)| \leq \frac{2 k \gamma_t^2 n}{\frac{\tau^2}{6 \mu_0 k} - 3 \gamma_t^2 \mu_0 k}.
\end{align*}
Hence for small distance $\gamma_t$, most of the row vectors $u_i^t$ of $U^t$ are close to the corresponding row vectors $u_i^*$ of $U^*$. 
\begin{proof}
	For any $i \in Q^t(\tau)$, we now derive a lower bound on $\| u_i^{t,T} U^{t,T} - u_i^{*,T} U^{*,T} \|_2^2$ by considering the cases $(\| u_i^t \|_2 - \| u_i^{*} \|_2)^2 \geq \frac{\tau^2}{6 \mu_0 k n}$ and $(\| u_i^t \|_2 - \| u_i^{*} \|_2)^2 < \frac{\tau^2}{6 \mu_0 k n}$, separately. Consider the case $ (\| u_i^t \|_2 - \| u_i^{*} \|_2)^2 \geq \frac{\tau^2}{6 \mu_0 k n}$. Recall $U^t, U^* \in \mathbb{R}^{n \times k}$ are two orthonormal matrices. Then,  
	\begin{align}
	\label{lowerbound1}
	& \| u_i^{t,T} U^{t,T} - u_i^{*,T} U^{*,T} \|_2^2  \nonumber\\
	= & (u_i^{t,T} U^{t,T} - u_i^{*,T} U^{*,T}) ( U^t u_i^t - U^{*} u_i^{*} )  \nonumber\\
	= & u_i^{t,T} u_i^t  +  u_i^{*,T} u_i^{*} - u_i^{t,T} U^{t,T} U^{*} u_i^{*} - u_i^{*,T} U^{*,T} U^t u_i^t  \nonumber\\
	\geq & \| u_i^t \|_2^2 + \| u_i^{*} \|_2^2 - 2 \| u_i^t \|_2 \| u_i^{*} \|_2  \nonumber\\
	= & (\| u_i^t \|_2 - \| u_i^{*} \|_2)^2 \geq \frac{\tau^2}{6 \mu_0 k n}.  
	\end{align}
	
	Next, we consider the case
	\begin{align}
	\label{inequality:uiuistarUpperBound}
	(\| u_i^t \|_2 - \| u_i^{*} \|_2)^2 < \frac{\tau^2}{6 \mu_0 k n}.
	\end{align} 
	We first show a lower bound on $\| u_i^t - u_i^* \|_2$. By Assumption $1$ on the incoherence of $U^*$ and the inequality (\ref{inequality:uiuistarUpperBound}), we have
	\begin{align}
	\label{inequality:upperboundui}
	\| u_i^t \|_2 \leq  \| u_i^* \|_2 + \left|\| u_i^t \|_2 - \| u_i^{*} \|_2 \right| \leq & \sqrt{\frac{\mu_0 k}{n}} + \sqrt{\frac{\tau^2}{6\mu_0 k n}}. 
	\end{align}
	Then,
	\begin{align}
	\label{inequality:uiuiuistaruistar}
	\| u_i^t u_i^{t,T} - u_i^* u_i^{*,T} \|_2 &= \| u_i^t u_i^{t,T} - u_i^t u_i^{*,T} + u_i^t u_i^{*,T} - u_i^* u_i^{*,T} \|_2  \nonumber \\
	&\leq \| u_i^t \|_2 \| u_i^{t,T} - u_i^{*,T} \|_2 +  \| u_i^t - u_i^* \|_2 \| u_i^{*,T} \|_2 \nonumber \\
	&\leq  \left( \sqrt{\frac{\mu_0 k}{n}} + \sqrt{\frac{\tau^2}{6\mu_0 k n}} \right) \| u_i^{t,T} - u_i^{*,T} \|_2 + \sqrt{\frac{\mu_0 k}{n}} \| u_i^t - u_i^* \|_2 \nonumber \\
	&= \left( 2 \sqrt{\frac{\mu_0 k}{n}} + \sqrt{\frac{\tau^2}{6\mu_0 k n}} \right) \| u_i^t - u_i^* \|_2.
	\end{align}
	Also by the definition of $Q^t(\tau)$, for any $i\in Q^t(\tau)$, we have
	\begin{align}
	\label{equation:uiuistarDistance}
	\| u_i^t u_i^{t,T} - u_i^{*} u_i^{*,T} \|_2 > \frac{\tau}{n}.
	\end{align}
	Recall $\tau \in (0, 1)$, $k \geq 1$ and $\mu_0 \geq 1$. Hence, 
	\begin{align}
	\label{inequality:LowerBounduiuistar}
	\| u_i^t - u_i^* \|_2 > \frac{\tau/n}{ 2 \sqrt{\frac{\mu_0 k}{n}} + \sqrt{\frac{\tau^2}{6\mu_0 k n}} } \geq \frac{\tau/n}{ 2 \sqrt{\frac{\mu_0 k}{n}} + \sqrt{\frac{1}{6} \frac{\mu_0 k}{ n}} }  \geq \frac{\tau}{\sqrt{6 \mu_0 k n}}.
	\end{align}
	Now,
	\begin{align*}
	& \| u_i^{t,T} U^{t,T} - u_i^{*,T} U^{*,T} \|_2^2  \nonumber\\
	=& u_i^{t,T} u_i^t  +  u_i^{*,T} u_i^{*} - u_i^{t,T} U^{t,T} U^{*} u_i^{*} - u_i^{*,T} U^{*,T} U^t u_i^t   \\
	=& \| u_i^t -u_i^* \|_2^2  - u_i^{t,T} (U^{t,T} U^{*} - I) u_i^{*} - u_i^{*,T} (U^{*,T} U^t - I) u_i^t \\
	\geq & \| u_i^t -u_i^* \|_2^2 - 2 \| I - U^{*,T} U^t  \|_2 \| u_i^t \|_2 \| u_i^* \|_2. 
	\end{align*}
	Since $U^{*,T} U^t$ is symmetric, $U^{*,T} U^t$ has SVD $U^{*,T} U^t = W \Sigma W^T$ for some orthonormal matrix $W \in \mathbb{R}^{k \times k}$, 
	\begin{align*}
	\| I - U^{*,T} U^t  \|_2 = \| W (I - \Sigma) W^T \|_2 = \| I - \Sigma \|_2.
	\end{align*} 
	By the property (\ref{dist_subspace_property3}) of subspace distance, the least singular value of $U^{*,T} U^t$ is $\sqrt{1-\gamma_t^2}$ and thus all the singular values in $\Sigma$ are in $[\sqrt{1-\gamma_t^2} , 1]$. Then, 
	\begin{align*}
	\| I - U^{*,T} U^t  \|_2 \leq  1 - \sqrt{1-\gamma_t^2} \leq \gamma_t^2. 
	\end{align*} 
	Hence, 
	\begin{align*}
	\| u_i^{t,T} U^{t,T} - u_i^{*,T} U^{*,T} \|_2^2  \geq & \| u_i^t -u_i^* \|_2^2 - 2 \gamma_t^2 \| u_i^t \|_2 \| u_i^* \|_2.
	\end{align*}
	By the lower bound on $\| u_i^t -u_i^* \|_2$ in (\ref{inequality:LowerBounduiuistar}), the upper bound of $\| u_i^t \|_2$ in (\ref{inequality:upperboundui}) and the incoherence Assumption $1$ on $U^*$, we have
	\begin{align*}
	\| u_i^{t,T} U^{t,T} - u_i^{*,T} U^{*,T} \|_2^2  \geq & \frac{\tau^2}{6 \mu_0 k n} - 2 \gamma_t^2 \left(\sqrt{\frac{\mu_0 k}{n}} + \sqrt{\frac{\tau^2}{6\mu_0 k n}}  \right) \sqrt{\frac{\mu_0 k}{n}} \\
	\geq & \frac{\tau^2}{6 \mu_0 k n} - 3 \gamma_t^2 \frac{\mu_0 k}{n},
	\end{align*}
	which, along with the lower bound on $\| u_i^{t,T} U^{t,T} - u_i^{*,T} U^{*,T} \|_2^2$ in (\ref{lowerbound1}) for the first case, implies that the inequality above holds for all $i \in Q^t(\tau)$. 
	Hence
	\begin{align}
	\label{inequality:FrobeniusNormUUstar}
	\| U^t U^{t,T} - U^* U^{*,T} \|_F^2 = \sum_{i=1}^n \| u_i^{t,T} U^{t,T} - u_i^{*,T} U^{*,T} \|_2^2 \geq |Q^t(\tau)| \left( \frac{\tau^2}{6 \mu_0 k n} - 3 \gamma_t^2 \frac{\mu_0 k}{n}   \right).
	\end{align}
	Since $U^t, U^* \in \mathbb{R}^{n \times k}$ are both orthonormal matrices, the ranks of $U^t U^{t,T}$ and $U^* U^{*, T}$ are both $k$. Then the rank of $U^t U^{t,T} - U^* U^{*,T}$ is at most $2k$, since the rank of the sum of two matrices is at most the sum of the ranks of two matrices. 
	Then by property (\ref{dist_subspace_property4}) of subspace distance, namely,  
	$$
	\gamma_t = \dist(U^t, U^*) = \|U^t U^{t,T} - U^* U^{*,T} \|_2
	$$
	and the inequality (\ref{inequality: matrix_inequality1}) where $l = 2k$, we have
	\begin{align*}
	\| U^t U^{t,T} - U^* U^{*,T}\|_F \leq \sqrt{2k} \| U^t U^{t,T} - U^* U^{*,T} \|_2 = \sqrt{2k} \gamma_t.
	\end{align*}
	Then from the inequality (\ref{inequality:FrobeniusNormUUstar}) we have
	\begin{align*}
	2k \gamma_t^2 \geq |Q^t(\tau)| \left( \frac{\tau^2}{6 \mu_0 k n} - 3 \gamma_t^2 \frac{\mu_0 k}{n}   \right),
	\end{align*}
	from which the result (\ref{upperBoundQdelta}) follows.   
\end{proof}

For $\tau, \alpha \in (0,1)$, let the set $S_{b,1}^t(\tau, \alpha)$ be 
\begin{align*}
S_{b,1}^t(\tau, \alpha) \triangleq \{ j \in [n]: \;  \left| \{ i\in [n]: (i,j) \in \Omega_{t+1} \text{ and } i \in Q^t(\tau)\} \right| \geq \alpha d \}.
\end{align*} 
That is, $S_{b,1}^t (\tau, \alpha)$ is the set of the vertices on the right in the random bipartite $d$-regular graph associated with $\Omega_{t+1}$ such that each vertex in $S_{b,1}^t (\tau, \alpha)$ has at least $\alpha d$ neighbors in the index set $Q^t(\tau)$. 
Let $W \in \mathbb{R}^{n \times k}$ be any orthonormal matrix with its $i$th row $w_i^T$ satisfying 
$$
\| w_i \|_2^2 \leq \frac{\mu k}{n}, \forall \; i \in [n]
$$
for some $\mu>0$. 
In our application, matrices $U^*$ and $U^t$ will play the role of $W$. 
For $a \in (0,1)$,  define the set
$$
S_{b,2}^t(W, a) \triangleq \{j \in [n]: \| \frac{n}{d} \sum_{i:(i,j) \in \Omega^{t+1}} w_i w_i^T - I \|_2 > a \}.
$$
Roughly speaking, 
$S_{b,2}^t(W, a)$ contains all the vertices $j \in [n]$ on the right in the random bipartite $d$-regular graph associated with $\Omega_{t+1}$ for which the corresponding matrix $\frac{n}{d} \sum_{i:(i,j) \in \Omega^{t+1}} w_i w_i^{T}$ deviates from $I$ by a certain threshold. 
The next lemma shows that the size of $S_b^t(\beta)$ is bounded from above by the sum of $|S_{b,1}^t(\tau, \alpha)|$ and $|S_{b,2}^t(W, a)|$ for a certain choice of $\tau$, $\alpha$, $W$ and $a$.  
\begin{lemma}
	\label{bound_Sbt_beta}
	Let $\delta$ and $\beta$ be as defined in $\mathcal{TAM}$. 
	Also, let $\tau = (1 - \beta - \delta)/2$ and $\alpha = (1 - \beta - \delta)/(12 \mu_0 k)$. 
	Then, 
	\begin{align}
	\label{upperbound:Sbt_beta}
	|S_b^t(\beta)| \leq |S_{b,1}^t(\tau, \alpha)| + |S_{b,2}^t(U^*, \delta)|. 
	\end{align}
\end{lemma}

\begin{proof}
	It suffices to show $S_b^t(\beta) \subseteq S_{b,1}^t(\tau, \alpha) \cup S_{b,2}^t(U^*, \delta)$. 
	For $j \notin S_{b,1}^t(\tau, \alpha) \cup S_{b,2}^t(U^*, \delta)$, it follows from the definition of $S_{b,1}^t(\tau, \alpha)$ and $S_{b,2}^t(U^*, \delta)$ that  
	\begin{align}
	\label{upperBoundTildeAlphad}
	|\{i \in [n]: (i,j) \in \Omega^{t+1} \text{ and } i \in Q^t(\tau)\}| < \alpha d  \quad \text{and} \quad \bigg\| \frac{n}{d} \sum_{i: (i,j) \in \Omega^{t+1}} u_i^* u_i^{*,T} - I \bigg\|_2 \leq \delta. 
	\end{align} 
	Then,
	\begin{align}
	\label{inequality:NormUiIk1}
	 & \bigg\| \frac{n}{d} \sum_{i: (i,j) \in \Omega^{t+1}} u_i^t u_i^{t, T} - I \bigg\|_2 \nonumber \\ 
	 \leq \quad & \bigg\| \frac{n}{d} \sum_{i: (i,j) \in \Omega^{t+1}} u_i^* u_i^{*, T} - I \bigg\|_2 + \frac{n}{d} \bigg\| \sum_{i: (i,j) \in \Omega^{t+1}} u_i^t u_i^{t,T} - \sum_{i: (i,j) \in \Omega^{t+1}} u_i^{*} u_i^{*,T} \bigg\|_2  \nonumber \\
	\leq \quad & \delta + \frac{n}{d} \bigg\| \sum_{i: (i,j) \in \Omega^{t+1}} u_i^t u_i^{t,T} - \sum_{i: (i,j) \in \Omega^{t+1}} u_i^{*} u_i^{*,T} \bigg\|_2.
	\end{align}
	Divide vertex $j$'s neighbors $\{i\in [n]: (i,j) \in \Omega^{t+1} \}$ into two parts: neighbors in $[n] \backslash Q^t(\tau)$ and neighbors in $Q^t(\tau)$, that is, 
	$$
	S_1 = \{i\in [n]: (i,j) \in \Omega^{t+1} \text{ and } i \notin Q^t(\tau) \} \; \text{ and } \; S_2 = \{ i \in [n]: (i,j) \in \Omega^{t+1}, i \in Q^t(\tau)\}.
	$$
	Then we have
	\begin{align*}
	& \bigg\| \sum_{i: (i,j) \in \Omega^{t+1}} u_i^t u_i^{t,T} - \sum_{i: (i,j) \in \Omega^{t+1}} u_i^{*} u_i^{*,T} \bigg\|_2  \leq  \sum_{i \in S_1} \big\| u_i^t u_i^{t,T} -  u_i^{*} u_i^{*,T} \big\|_2 +  \sum_{i \in S_2} \big\|  u_i^t u_i^{t,T} -  u_i^{*} u_i^{*,T} \big\|_2. 
	\end{align*}
	$i \in S_1$ implies $i \notin Q^t(\tau)$ and thus $\big\| u_i^t u_i^{t,T} - u_i^{*} u_i^{*,T} \big\|_2 \leq \tau/n$. Then the right hand side of the inequality above is 
	$$
	\leq \frac{\tau}{n} |S_1| + |S_2| (\big\|u_i^t u_i^{t,T} \big\|_2 + \big\|  u_i^{*} u_i^{*,T} \big\|_2).
	$$
	From the first inequality of (\ref{upperBoundTildeAlphad}), we have $|S_2| < \alpha d$, which, together with the incoherence assumption of $u_i^t$ in (\ref{UtIncoherenceCondition}) and $\alpha = (1-\beta-\delta)/(12 \mu_0 k)$, implies the inequality above 
	\begin{align*}
	& \leq \frac{\tau}{n} d + \alpha d \left( \frac{5\mu_0 k}{n} + \frac{\mu_0 k}{n}  \right) \\
	& = \frac{d}{n} \left(\frac{1-\beta-\delta}{2} + \frac{1-\beta-\delta}{2} \right) = \frac{d}{n} (1 - \beta - \delta).
	\end{align*}
	Then (\ref{inequality:NormUiIk1}) becomes
	\begin{align}
	\label{inequality:NormUiIk2}
	\bigg\| \frac{n}{d} \sum_{i: (i,j) \in \Omega^{t+1}} u_i^t u_i^{t, T} - I \bigg\|_2  & \leq 1-\beta.
	\end{align}
	It follows from the definition of $S_b^t(\beta)$ in (\ref{Sbtdelta}) that $j \notin S_b^t(\beta)$ and thus $S_b^t(\beta) \subseteq S_{b,1}^t(\tau, \alpha) \cup S_{b,2}^t(U^*, \delta)$.  
\end{proof}

We will establish Theorem \ref{theorem:UpperBoundSbt} from the following two lemmas, which gives upper bounds on $|S_{b,1}^t(\tau, \alpha)|$ and $|S_{b,2}^t(W, a)|$, respectively. We delay their proof for later.   

\begin{proposition}
	\label{proposition:boundSb1}
	Suppose Assumption $1$ holds. Let $\alpha$ and $\rho_t$ be as defined in (\ref{alpharhot}). Without loss of generality, let $\alpha d$ be an integer. Also, let 
	\begin{align*}
	\lambda = \frac{1}{\alpha k \mu_0} \text{ and } \nu = \frac{\rho_t}{k^2 \mu_0}. 
	\end{align*}
	For a $C>0$, suppose  
	\begin{align*}
	C \geq \mathrm{e} \sqrt{\nu \lambda}, \; \gamma_t \in (0, 1/(C \mu_0 k^{1.5})), \;  |Q^t(\tau)| \leq \rho_t \gamma_t^2 n   \text{ and } \rho_t \gamma_t^2 < 1.
	\end{align*}
	The following inequality 
	\begin{align}
	\label{inequality:UpperBoundOnSb1}
	|S_{b,1}^t (\tau, \alpha)| \leq 1.1 \mathrm{e} \left(\frac{\mathrm{e}^2 \rho_t \gamma_t^2}{\alpha} \right)^{\alpha d} n
	\end{align} 
	holds w.h.p. 
\end{proposition}

\begin{proposition}
	\label{Sb2_configurationA1} 
	For $\mu > 0$ and $a \in (0, 1)$, let $f(d, \mu, a)$ be as defined in (\ref{equation:fd}). 
	Then w.h.p. for any $\zeta>0$
	\begin{align}
	\label{inequality:Sb2Bound1}
	|S_{b,2}^t(W, a)| \leq (1+\zeta) f(d, \mu, a) n.
	\end{align}
	Suppose Assumptions $1$ and $2$ hold for $U^*$. Let $\delta$ be as given in Assumptions $2$.      
	w.h.p. for any $\zeta>0$ 
	\begin{align}
	\label{inequality:Sb2Bound2}
	|S_{b,2}^t(U^*, \delta)| \leq \zeta n.
	\end{align}       
\end{proposition} 

\begin{proof}[Proof of Theorem \ref{theorem:UpperBoundSbt} ]
	The first result (\ref{inequality:SbtA0}) directly follows from (\ref{inequality:Sb2Bound1}) in Lemma \ref{Sb2_configurationA1} where we choose $W=U^t$, $\mu = 5 \mu_0$ and $a = 1 - \beta$. 
	
	Now we prove the second result (\ref{inequality:SbtA1}). Let $\tau = (1 - \beta - \delta)/2$. By Lemma \ref{bound_Sbt_beta}, $|S_b^t(\beta)|$ is bounded from above by
	$$
	|S_b^t(\beta)| \leq |S_{b, 1}^t(\tau, \alpha)| + |S_{b,2}^t(U^*, \delta)|.
	$$ 
	Next we rely on Proposition \ref{proposition:boundSb1} and Proposition \ref{Sb2_configurationA1} to derive upper bounds on $|S_{b, 1}^t(\tau, \alpha)|$ and $|S_{b,2}^t(U^*, \delta)|$, respectively.  
	
	First, we verify the assumptions of Proposition \ref{proposition:boundSb1}. 
	We have
	$$
	\lambda = \frac{12}{1-\beta-\delta} \text{ and } \nu = \frac{2}{\frac{(1-\beta-\delta)^2}{24} - 3 \gamma_t^2 \mu_0^2 k^2 }. 
	$$
	Let $C = \sqrt{C(\delta, \beta)}/(4\sqrt{10})$. 
	By the assumption $\gamma_t \in (0, 1/(C k^{1.5} \mu_0))$, it can be easily checked that for a large $C(\delta, \beta)>0$, we have
	\begin{align*}
	C \geq \mathrm{e} \sqrt{\nu \lambda} \text{ and } \rho_t \gamma_t^2 < 1.  
	\end{align*}
	Also, Lemma \ref{lemma:boundQ} implies $|Q^t(\tau)| \leq \rho_t \gamma_t^2 n$. The verification is completed.  
	Then it follows from Proposition \ref{proposition:boundSb1} that w.h.p.
	\begin{align}
	\label{ineqaulity:Sb1t_delta1_alpha}
	|S_{b,1}^t(\tau, \alpha)| \leq  1.1 \mathrm{e} \left( \frac{\mathrm{e}^2 \rho_t \gamma_t^2}{\alpha}  \right)^{\alpha d} n.
	\end{align}
	Also, (\ref{inequality:Sb2Bound1}) in Proposition \ref{Sb2_configurationA1} implies that under Assumption $1$ w.h.p. for any $\zeta>0$ 
	$$
	|S_{b,2}^t(U^*, \delta)| \leq (1+\zeta) f(d, \mu_0, \delta) n.  
	$$
	Therefore w.h.p.
	$$
	|S_b^t(\beta)| \leq |S_{b,1}^t(\tau, \alpha)| + |S_{b,2}^t(U^*, \delta)| \leq 1.1 \mathrm{e} \left( \frac{\mathrm{e}^2 \rho_t \gamma_t^2}{\alpha}  \right)^{\alpha d} n + (1+\zeta) f(d, \mu_0, \delta) n
	$$
	from which (\ref{inequality:SbtA1}) follows. 
	
	Suppose that Assumption $2$ is also satisfied, (\ref{inequality:Sb2Bound2}) in Proposition \ref{Sb2_configurationA1} implies that w.h.p. for any $\zeta>0$  
	$$
	|S_{b,2}^t(U^*, \delta)| \leq \zeta n
	$$
	which, together with the bound in (\ref{ineqaulity:Sb1t_delta1_alpha}), implies the third result (\ref{inequality:SbtA1A2}) similarly. 
\end{proof}

\subsubsection{Bounding the size of $S_{b,1}^t(\tau, \alpha)$. Proof of Proposition \ref{proposition:boundSb1}}

We will rely on the configuration model of random regular graphs and its extension to the random bipartite regular graphs \cite{Bollobas1985, Janson2000}, which we now introduce.  

A configuration model of $\mathbb{G}_d(n,n)$ is obtained by replicating each of the $2n$ vertices of the graph $d$ times, and then creating a uniform random bipartite matching between $dn$ replicas on the left and the other $dn$ replicas on the right. 
Then for every two vertices $u \in [n]$ and $v \in [n]$ on the opposite sides, an edge is created between $u$ and $v$, for each edge between any of the replicas of $u$ and any of the replicas of $v$.  The step of creating edges between vertices belonging to different sides from the matching on $dn$ replicas on the left and the other $dn$ replicas on the right we call projecting. It is known that, conditioned on the absence of parallel edges, this procedure gives a bipartite $d$-regular graph generated uniformly at random from the set of all bipartite $d$-regular graphs on $2n$ vertices. 
It is also known that the probability of no parallel edges after projecting is bounded away from zero when $d$ is bounded. 
More detailed results on this fact can be found in the introduction section of \cite{Cook2014}.
Since we are only concerned with events holding w.h.p., such a conditioning is irrelevant to us and thus we assume that $\mathbb{G}_d(n,n)$ is generated simply by first choosing a unifrom random bipartite matching and then projecting. Denote the configuration model by $\bar{\mathbb{G}}_d(n,n)$, with vertices 
denoted by $(i,r,L)$ for the vertices on the left and $(i,r,R)$ for the vertices on the right where $i \in [n]$ and $r \in [d]$. Namely, $(i, r, L(R))$ is the $r$-th replica of vertex $i$ on the left (right) in the configuration model. Given any set $A \subset [n]$ on the left (right), let $\bar{A}$ be the extension of $A$ to the configuration model, namely, $\bar{A} = \{(i, r, L(R)): i \in A, r \in [d] \}$. We will use $A$ and $\bar{A}$ interchangeably.    

\begin{proof}[Proof of Proposition \ref{proposition:boundSb1}]        
	By the assumption $|Q^t(\tau)| \leq \rho_t \gamma_t^2 n$, let $|Q^t(\tau)| = \hat{\rho} \gamma_t^2 n$ for some $\hat{\rho} \in [0,\rho_t]$. Let $\mathcal{E}(\beta n, \alpha d)$ be the event that there are exactly $|S_{b,1}^t (\tau, \alpha)| = \beta n$ vertices on the right such that each of them has at least $\alpha d$ neighbors in the vertex set $Q^t(\tau)$ on the left. Also, let $\mathcal{R}(\beta n, \alpha d, l) \subset \mathcal{E}(\beta n, \alpha d)$ be the event that there are exactly $l$ edges between the vertex set $Q^t(\tau)$ on the left and the vertex set $S^t_{b,1}(\tau, \alpha)$ on the right. 
	Since under the event $\mathcal{E}(\beta n, \alpha d)$ each vertex in $S_{b,1}^t (\tau, \alpha)$ has at least $\alpha d$ neighbors in $Q^t(\tau)$ and the number of edges originating from $S_{b,1}^t (\tau, \alpha)$ is $d \beta n $, the number of edges between the vertex set $Q^t(\tau)$ and the vertex set $S^t_{b,1}(\tau, \alpha)$ is within $[\alpha d \beta n, d \beta  n]$. 
	Then $l$ is at least $\alpha d \beta n$, at most $\beta d n$ 
	and $\cup_{l=\alpha d \beta n}^{\beta d n}\mathcal{R}(\beta n, \alpha d, l) = \mathcal{E}(\beta n, \alpha d)$. In what follows we bound the probability $\mathbb{P}(\mathcal{R}(\beta n, \alpha d, l))$ in the configuration model $\bar{G}_d(n,n)$ for $l \in [\alpha d \beta n, \beta dn]$, and thus the probability $\mathbb{P}(\mathcal{E}(\beta n, \alpha d))$ in the configuration model $\bar{G}_d(n,n)$ by the union bound. 
	
	It follows from $S_{b,1}^t (\tau, \alpha) = \beta n$ and $|Q^t(\tau)| = \hat{\rho} \gamma_t^2 n$ that their counterparts in the configuration model are  
	$$
	|\bar{S}^t_{b,1}(\tau, \alpha)| = \beta d n \quad \text{ and } \quad |\bar{Q}^t(\tau)| = \hat{\rho} \gamma_t^2 d n. 
	$$
	Let $\theta \in [\alpha, 1]$ be defined by $l = \theta \beta d n$. 
	Then as shown in Figure \ref{fig:randomRegularGraphComputeIllustration}, the number of edges between $\bar{Q}^t(\tau)$ and $\overline{[n]\backslash S_{b,1}^t(\tau, \alpha)}$ is 
	\begin{align}
	\label{number_edges_cross1}
	\hat{\rho} \gamma_t^2 d n - \theta \beta d n,
	\end{align}
	the number of edges between $\bar{S}_{b,1}^t(\tau, \alpha)$ and $\overline{[n]\backslash Q^t(\tau)}$ is $\beta d  n - \theta \beta d n$, and the number of edges between $ \overline{[n] \backslash Q^t(\tau)}$ and $ \overline{[n] \backslash S_{b,1}^t(\tau, \alpha)}$ is 
	\begin{align}
	\label{number_edges_cross2}
	(1-\hat{\rho} \gamma_t^2)dn - (\beta d n - \theta \beta d n) = (1-\beta) dn - \hat{\rho}\gamma_t^2 dn + \theta \beta dn. 
	\end{align}
	
	\begin{figure}[!ht]
		\centering
		\includegraphics[width=0.6\textwidth]{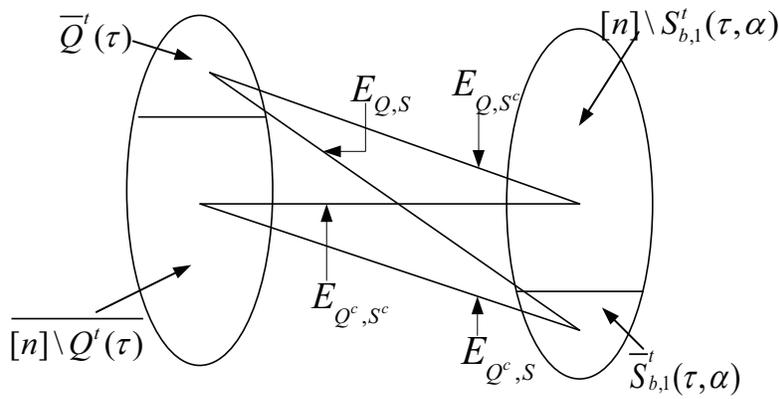}
		\caption{Illustration of the event $\mathcal{R}(\beta n, \alpha d, \theta \beta d n)$ where $E_{Q,S} = \theta \beta d n$ represents the number of edges between two vertex sets sitting at the ends of the line corresponding to $E_{Q,S}$.  
		$E_{Q,S^c} = \hat{\rho} \gamma_t^2 d n - \theta \beta d n$, $E_{Q^c, S} = \beta d n - \theta \beta d n$ and $E_{Q^c, S^c} = (1-\beta)dn - \hat{\rho} \gamma_t^2 d n + \theta \beta d n$ are defined accordingly. } \label{fig:randomRegularGraphComputeIllustration}
	\end{figure}
	
	Let $X_{ij}$, $i\in [\beta n], j \in [d]$ be i.i.d. Bernoulli random variables with $\mathbb{P}(X_{ij} = 1) = \theta $, and $Y_{ij}$, $i\in [(1-\beta) n], j \in [d]$ be another set of i.i.d. Bernoulli random variables with $\mathbb{P}(Y_{ij} = 1) = \frac{\hat{\rho} \gamma_t^2 - \theta \beta}{1-\beta}$. 
	Define two conditional probabilities
	{\small
		\begin{align*}
		f_1 =& \mathbb{P} \left(\sum_{j=1}^d X_{ij} \geq \alpha d, \; \forall i \in [\beta n] \; \middle| \; \sum_{i=1}^{\beta n} \sum_{j=1}^d X_{ij} = \theta \beta d n \right), \\
		f_2 =& \mathbb{P} \left(\sum_{j=1}^d Y_{ij} < \alpha d, \; \forall i \in [(1-\beta) n] \; \middle| \; \sum_{i=1}^{(1-\beta) n} \sum_{j=1}^d Y_{ij} = \hat{\rho} \gamma_t^2 d n - \theta \beta d n \right). 
		\end{align*}
	}
	Then we claim that ${{\beta dn} \choose {\theta \beta d n}} f_1$ is the number of ways of choosing $\theta \beta d n$ replicas from $\beta d n$ replicas in $\bar{S}_{b,1}^t(\tau, \alpha)$ such that each vertex in $S_{b,1}^t(\tau, \alpha)$ has at least $\alpha d$ replicas chosen. Define the set 
	$$
	L \triangleq \bigg\{
	(r_1,\ldots,r_{\beta n}) \in [d]^{\beta n}: \; \sum_{i=1}^{\beta n} r_i = \theta \beta d n; \; r_i \geq \alpha d, \; \forall i \in [\beta n] \bigg\}.
	$$
	Then we expand $f_1$ by Bayes' formula  
	\begin{align*}
	f_1 &= \frac{\sum_{(r_1,\ldots,r_{\beta n}): \; (r_1,\ldots,r_{\beta n}) \in L}  \prod_{i=1}^{\beta n} {d \choose r_i} \theta^{r_i} (1 - \theta)^{d - r_i} }{{{\beta dn} \choose {\theta \beta d n}} \theta^{\theta \beta d n} (1 - \theta)^{(1-\theta) \beta d n}}  \\ 
	&= \frac{\sum_{(r_1,\ldots,r_{\beta n}): \; (r_1,\ldots,r_{\beta n}) \in L}  \prod_{i=1}^{\beta n} {d \choose r_i} }{{{\beta dn} \choose {\theta \beta d n}} }.
	\end{align*}  
	Observe that the numerator of the expression above is exactly the number of ways of choosing $\theta \beta d n$ replicas from $\beta d n$ replicas in $\bar{S}_{b,1}^t(\tau, \alpha)$ such that each vertex in $S_{b,1}^t(\tau, \alpha)$ has at least $\alpha d$ replicas chosen. 
	Hence the claim follows. Similarly, we have that ${{(1-\beta)dn}  \choose {\hat{\rho} \gamma_t^2 d n - \theta \beta d n}} f_2$ is the number of ways of choosing $\hat{\rho} \gamma_t^2 d n - \theta \beta d n$ replicas from $(1 - \beta) d n$ replicas in $\overline{[n] \backslash S_{b,1}^t(\tau, \alpha)}$ such that each vertex in $[n] \backslash S_{b,1}^t(\tau, \alpha)$ has less than $\alpha d$ replicas chosen.
	
	Now we claim that the probability $\mathbb{P}(\mathcal{R}(\beta n, \alpha d, \theta \beta d n))$ is given by 
	\begin{align}
	\label{equation:probR}
	\mathbb{P}(\mathcal{R}(\beta n, \alpha d, \theta \beta d n)) =  \frac{ { n \choose \beta n} I_1 I_2 I_3 I_4 }{(dn)!}
	\end{align}
	where 
	\begin{align*}
	I_1 &=  { {\beta dn} \choose {\theta \beta d n}} f_1 {{\hat{\rho} \gamma_t^2 d n} \choose {\theta \beta d n}} (\theta \beta d n)!,   \\ 
	I_2 &= { {(1-\hat{\rho} \gamma_t^2) d n} \choose  {\beta d n - \theta \beta d n}} (\beta d n - \theta \beta d n)!,  \\
	I_3 &= {{(1-\beta)dn}  \choose {\hat{\rho} \gamma_t^2 d n - \theta \beta d n}} f_2 (\hat{\rho} \gamma_t^2 d n - \theta \beta d n)!,  \\
	I_4 &= ((1-\beta)dn - \hat{\rho} \gamma_t^2 d n + \theta \beta d n)!.
	\end{align*}
	Indeed, the term ${ n \choose \beta n}$ is the number of ways of selecting $|S_{b,1}^t(\tau, \alpha)| = \beta n$ vertices from $n$ vertices on the right. 
	The term $I_1$ is the number of matching choices between $\theta \beta d n$ vertices chosen from $\bar{S}_{b,1}^t(\tau, \alpha)$ and $\theta \beta d n$ vertices chosen from $\bar{Q}^t(\tau)$ such that any vertex in $S_{b,1}^t(\tau, \alpha)$ has at least $\alpha d$ neighbors in $Q^t(\tau)$. The term $I_2$ is the number of matching choices between the remaining vetices in $\bar{S}_{b,1}^t(\tau, \alpha)$ and $\beta d n - \theta \beta d n$ vertices chosen from $\overline{[n]\backslash Q^t(\tau)}$. 
	The term $I_3$ is the number of matching choices between the remaining vetices in $\bar{Q}^t(\tau)$ and $\hat{\rho} \gamma_t^2 d n - \theta \beta d n$ vertices chosen from $\overline{[n] \backslash S_{b,1}^t(\tau, \alpha)}$ such that any vertex in $[n] \backslash S_{b,1}^t(\tau, \alpha)$ has less than $\alpha d$ neighbors in $Q^t(\tau)$. The term $I_4$ is the number of matching choices between the remaining vetices in $\overline{[n] \backslash Q^t(\tau)}$ and the remaining vertices in $\overline{[n] \backslash S_{b,1}^t(\tau, \alpha)}$. 
	Thus ${n \choose \beta n} I_1 I_2 I_3 I_4$ is the number of configuration graphs such that there are exactly $\beta n$ vertices on the right each of which has at least $\alpha d$ neighbors in $Q^t(\tau)$, and the number of edges between $Q^t(\tau)$ and $S_{b,1}^t(\tau, \alpha)$ is exactly $\theta \beta d n$. $(dn)!$ is precisely the total number of configuration graphs. Hence (\ref{equation:probR}) follows. 
	
	By expanding the terms in (\ref{equation:probR}), we have the following lemma. The proof of this lemma, which involves heavy asymptotic expansions, can be found in Appendix \ref{ap:proof_of_lemma_probability_rate}.
	\begin{lemma}
		\label{probabilityRate}
		Given $\beta \in (1.1 \mathrm{e} (\mathrm{e}^2 \rho_t \gamma_t^2 /\alpha )^{\alpha d}, 1]$, there exists an $\eta>0$ such that     
		\begin{align}
		\label{inequality:eta}
		\limsup_{n \rightarrow \infty} \frac{1}{n} \log \mathbb{P}(\mathcal{R}(\beta n, \alpha d, \theta \beta d n)) \leq -\eta.
		\end{align}
	\end{lemma} 
	Applying Lemma \ref{probabilityRate}, for any $\beta \in (1.1 \mathrm{e} (\mathrm{e}^2 \rho_t \gamma_t^2 /\alpha )^{\alpha d}, 1]$ we have by the union bound 
	\begin{align*}
	\mathbb{P}(\mathcal{E}(\beta n, \alpha d)) \leq \sum_{l=\alpha d \beta n}^{\beta d n} \mathbb{P}(\mathcal{R}(\beta n, \alpha d, l)) = \exp(-\Omega(n)).  
	\end{align*} 
	Thus in the configuration model $\bar{\mathbb{G}}_d(n, n)$, we have 
	\begin{align*}
	\mathbb{P}(|S_{b,1}^t(\tau, \alpha)| >  1.1 \mathrm{e} (\mathrm{e}^2 \rho_t \gamma_t^2 /\alpha )^{\alpha d}  n ) \leq &  \sum_{h= \lfloor 1.1 \mathrm{e} (\mathrm{e}^2 \rho_t \gamma_t^2 /\alpha )^{\alpha d} n \rfloor + 1}^{n}\mathbb{P}(\mathcal{E}(h, \alpha d)) = \exp(-\Omega(n)).     
	\end{align*} 
\end{proof}

\subsubsection{Bounding the size of $S_{b,2}^t(a)$. Proof of Proposition \ref{Sb2_configurationA1}}

We first introduce Matrix Bernstein inequality.
\begin{theorem}
	\label{MatrixBernstein}
	{\upshape \cite[Theorem 6.1.1]{Tropp2015}}
	Consider a finite sequence $\{\mathbf{S}_k\}$ of independent, random matrices with common dimension $d_1 \times d_2$. Assume that
	$$
	\mathbb{E}(\mathbf{S}_k) = \mathbf{0} \quad \text{and} \quad \| \mathbf{S}_k \|_2 \leq L \quad \text{for each index } k. 
	$$
	Introduce the random matrix $\mathbf{Z} = \sum_{k} \mathbf{S}_k$. Let 
	\begin{align*}
	\nu(\mathbf{Z}) &= \max \{ \|\mathbb{E}(\mathbf{Z}^T \mathbf{Z})\|_2, \|\mathbb{E}( \mathbf{Z} \mathbf{Z}^T)\|_2 \}  \nonumber \\
	&= \max \{ \|\mathbb{E}(\sum_{k} \mathbf{S}_k^T \mathbf{S}_k)\|_2, \|\mathbb{E}(\sum_{k} \mathbf{S}_k \mathbf{S}_k^T )\|_2 \}.
	\end{align*}
	Then for all $t \geq 0$,
	\begin{align}
	\label{inequality:MatrixBernstein}
	\mathbb{P}(\| \mathbf{Z}\|_2 \geq t) \leq (d_1 + d_2) \exp \left( \frac{-t^2/2}{ \nu(\mathbf{Z}) + L t/3} \right).
	\end{align}
\end{theorem}

We rely on the configuration model $\bar{\mathbb{G}}_d(n,n)$ to prove Proposition \ref{Sb2_configurationA1}. 
To state the generation of the configuration model $\bar{\mathbb{G}}_d(n,n)$ more precisely, we first introduce an ordering for the replicas on the right side of $\bar{\mathbb{G}}_d(n,n)$. 
For $j_1, j_2 \in [n]$ and $r_1, r_2 \in [d]$, we say $(j_1,r_1,R) > (j_2,r_2,R)$ if $j_1 > j_2$. 
For $j \in [n]$ and $r_1, r_2 \in [d]$, we say $(j,r_1,R) > (j,r_2,R)$ if $r_1 > r_2$. 
Here we use the following procedure to generate a random bipartite $d$-regular multigraphs $\bar{\mathbb{G}}_d(n,n)$ on $[n] \times [n]$ vertices \cite{Cook2014, WormaldModelsRandomGraphs}. 
Replicate each vetex in $[n]$ on both sides of the graph $d$ times. 
Then on the left side, the replicas are $(i,r,L)$ for all $i \in [n]$ and all $r\in [d]$. 
Similarly on the right side, the replicas are $(i,r,R)$ for all $i \in [n]$ and all $r \in [d]$. 
Always choose the replica on the right of the least order and pair it uniformly at random with one unpaired replica on the left until all the replicas are paired. Finally for each pair, create an edge between the two replicas in the pair. 

\begin{proof}[Proof of Proposition \ref{Sb2_configurationA1}]  
	Denote $(i,r_1, L) \stackrel{\bar{\mathbb{G}}_d(n, n)}{\thicksim} (j,r_2,R)$ if the vertex replica $(i,r_1, L)$ on the left is paired with the vertex replica $(j,r_2,R)$ on the right in the graph $\bar{\mathbb{G}}_d(n, n)$. Then for each $j \in [n]$, the vertex replicas on the left pairing with the replicas $(j, r_2, R)$, $\forall r_2 \in [d]$, on the right in $\bar{\mathbb{G}}_d(n, n)$ are included in    
	$$
	H_{j} \triangleq \bigg\{(i,r_1,L): \exists \; r_2 \in [d] \text{ such that } (i,r_1, L) \stackrel{\bar{\mathbb{G}}_d(n, n)}{\thicksim} (j,r_2,R) \;  \bigg\}.
	$$ 
	Recall $W \in \mathbb{R}^{n \times k}$ is an orthonormal matrix with incoherence parameter $\mu>0$. 
	For the tuple $(i, r)$, $i \in [n]$ and $r \in [d]$, let $g((i, r)) \triangleq i$ and correspondingly
	$$
	\hat{S}_{b,2}(W, a) \triangleq \bigg\{j \in [n]: \bigg\| \frac{n}{d} \sum_{(i, r): (i,r,L) \in H_j} w_{g((i, r))} w_{g((i, r))}^{T} - I \bigg\|_2 > a \bigg\}.
	$$
	Observe that conditional on $\bar{G}_d(n,n)$ being a simple graph, $\hat{S}_{b,2}(W, a)$ has the same distribution as $S_{b,2}^t(W, a)$. For bounded $d$, the probability that the configuration model produces a simple graph is bounded away from zero. Since we are only concerned with events holding w.h.p., in the following we derive an upper bound on $|\hat{S}_{b,2}(W, a)|$ instead. 
	
	Let $Z_{ir}$, $i \in [n]$ and $r \in [d]$, be a sequence of i.i.d. Bernoulli random variable with $\mathbb{P}(Z_{ir} = 1) = 1/n$. 
	$H_1$ consists of $d$ replicas on the left which are paired with the $d$ least ordered replicas on the right in $\bar{\mathbb{G}}_d(n,n)$. 
	$H_1$ can be also seen as $d$ replicas chosen uniformly at random from $nd$ replicas on the left. 
	Then we have  
	\begin{align}
	\label{concentration_conditional}
	& \mathbb{P} \bigg( \bigg\| \frac{n}{d} \sum_{(i, r): (i, r, L) \in H_{1}} w_{g((i, r))} w_{g((i, r))}^T - I \bigg\|_2 > a \bigg) \nonumber \\
	=& \mathbb{P} \left( \bigg\|\frac{n}{d} \sum_{(i,r) \in \{(i,r) \in [n] \times [d]: \; Z_{ir}=1 \}} w_{g((i, r))} w_{g((i, r))}^T - I \bigg\|_2 > a \; \middle | \; \sum_{i=1}^n \sum_{r=1}^d Z_{ir} = d  \right).
	\end{align}
	It follows from the Local Limit Theorem \cite[Theorem 9.1]{Lesigne2005} that
	$$
	\mathbb{P}\left(\sum_{i=1}^n \sum_{r=1}^d Z_{ir} = d \right) = \frac{1}{\sqrt{2\pi d}}(1+o(1)).
	$$
	Then we have an upper bound on the right hand side of (\ref{concentration_conditional})  
	\begin{align}
	\label{inequality:concentrationuistar}
	& \mathbb{P} \bigg( \bigg\|\frac{n}{d} \sum_{(i,r) \in \{(i,r) \in [n] \times [d]: Z_{ir}=1\}} w_{g((i, r))} w_{g((i, r))}^T - I \bigg\|_2 > a \; \bigg | \; \sum_{i=1}^n \sum_{r=1}^d Z_{ir} = d  \bigg) \nonumber \\ 
	& \leq \sqrt{2\pi d} (1+o(1))   \mathbb{P} \bigg( \bigg\|\frac{n}{d} \sum_{(i,r) \in \{(i,r) \in [n] \times [d]: Z_{ir}=1\}} w_{g((i, r))} w_{g((i, r))}^T - I \bigg\|_2 > a \bigg).
	\end{align}
	We claim
	\begin{align*}
	\mathbb{P} \bigg( \bigg\|\frac{n}{d} \sum_{(i,r) \in \{(i,r) \in [n] \times [d]: Z_{ir}=1\}} w_{g((i, r))} w_{g((i, r))}^T - I \bigg\|_2 > a \bigg)  \leq 2k \exp \left( \frac{-a^2/2}{\mu k + \mu k a/3} d  \right).
	\end{align*}
	Now we use Matrix Bernstein inequality (Theorem \ref{MatrixBernstein}) to establish this claim. Let $S_{ir}$, $i \in [n]$ and $r \in [d]$, be 
	$$
	S_{ir} = \frac{n}{d} \left(Z_{ir} w_i w_i^T - \frac{1}{n} w_i w_i^T \right).
	$$ 
	Then by $\sum_{i=1}^n  w_i w_i^T = W^{T} W = I$,  
	$$
	\sum_{i=1}^n \sum_{r=1}^d S_{ir} = \frac{n}{d} \sum_{(i,r) \in \{(i,r)\in [n] \times [d]: Z_{ir}=1\}} w_{g((i, r))} w_{g((i, r))}^T - I
	$$
	and $\mathbb{E} (S_{ir}) = \mathbf{0}$. Using $\| w_i \|_2^2 \leq \mu k /n$ for all $i \in [n]$, we have 
	$$
	\| S_{ir} \|_2 \leq \frac{n}{d} \times (1-\frac{1}{n}) \| w_i \|_2^2 \leq  \frac{\mu k}{d}, \quad \forall i \in [n] \text{ and } \forall r\in [d]. 
	$$
	Observe $S_{ir} \in \mathbb{R}^{k \times k}$ is a symmetric matrix and $w_i w_i^T$ is positive semidefinite. Then, 
	\begin{align*}
	\bigg\|\sum_{i=1}^n \sum_{r=1}^d \mathbb{E}(S_{ir} S_{ir}) \bigg\|_2 &= \left( \frac{n}{d} \right)^2 \bigg\|\sum_{i=1}^n \sum_{r=1}^d \mathbb{E}(Z_{ir} - \frac{2}{n} Z_{ir} + \frac{1}{n^2} ) w_i w_i^T w_i w_i^T \bigg\|_2    \\
	&= \left( \frac{n}{d} \right)^2 \bigg\|\sum_{i=1}^n  d \left( \frac{1}{n} - \frac{1}{n^2}\right) w_i w_i^T w_i w_i^T \bigg \|_2  \\
	&\leq \frac{n}{d} \Big\|\sum_{i=1}^n (w_i^T w_i)  w_i w_i^T \Big\|_2 \\
	&\leq \frac{n}{d} \max_{i \in [n]} \{w_i^T w_i\} \Big\|\sum_{i=1}^n  w_i w_i^T \Big\|_2.
	\end{align*}
	By $\sum_{i=1}^n  w_i w_i^T = I$ and the incoherence parameter $\mu$ of $W$, we have 
	\begin{align*}
	\bigg\|\sum_{i=1}^n \sum_{r=1}^d \mathbb{E}(S_{ir} S_{ir}) \bigg\|_2 \leq \frac{n}{d}  \times \frac{\mu k}{n} =  \frac{\mu k }{d}. 
	\end{align*}
	The claim then follows from choosing $t = a$ in (\ref{inequality:MatrixBernstein}) in Theorem \ref{MatrixBernstein}. 
	Then from the inequality (\ref{inequality:concentrationuistar}), we have
	\begin{align*}
	\mathbb{P}\bigg( \Big\| \frac{n}{d} \sum_{(i, r): (i, r, L) \in H_{1}} w_{g((i, r))} w_{g((i, r))}^T - I \Big\|_2 > a \bigg) \leq 2 k \sqrt{2\pi d} (1+o(1))    \exp \left( \frac{-a^2/2}{\mu k + \mu k a/3} d  \right).
	\end{align*}
	In the configuration model $\bar{\mathbb{G}}_d(n,n)$, $H_j$ for $2 \leq j \leq n$ has the same distribution as $H_1$. Hence for a large $n$ we have 
	\begin{align}
	\label{inequality:Expected_hatSb2}
	\mathbb{E}(|\hat{S}_{b,2}(W, a)|) \leq 2k \sqrt{2\pi d} (1+o(1))    \exp \left( \frac{-a^2/2}{\mu k + \mu k a/3} d  \right) n < f(d, \mu, a) n. 
	\end{align}
	Next we apply the following concentration result.
	\begin{theorem}
		\label{wormald:concentration} 
		{\upshape \cite[Theorem 2.19]{WormaldModelsRandomGraphs}}
		If $X_n$ is a random variable defined on $\bar{\mathbb{G}}_d(n,n)$ such that $|X_n(P) - X_n(P')| \leq c$ whenever $P$ and $P'$ differ by a simple switching of two edges, then 
		$$
		\mathbb{P}(|X_n - \mathbb{E}(X_n)| \geq t) \leq 2 \exp \left( \frac{-t^2}{d n c^2} \right)
		$$
		for all $t>0$. 
	\end{theorem}
	
	Although this result is established for the configuration model of a random regular graph, the same result for the configuration model of a random bipartite regular graph can be established in the obvious manner. Choosing the constant $c=2$ in this theorem, we have
	\begin{align*}{}
	& \mathbb{P}\bigg(\left| |\hat{S}_{b,2}(W, a)| - \mathbb{E}(|\hat{S}_{b,2}(W, a)|)\right| \geq \zeta f(d, \mu, a) n \bigg) \\
	\leq ~&  2 \exp\left( -\frac{\zeta^2 f(d, \mu, a)^2 n^2}{4dn} \right) = 2 \exp\left( -\frac{\zeta^2 f(d, \mu, a)^2}{4d} n \right).
	\end{align*}
	Then it follows from the inequality above and the inequality (\ref{inequality:Expected_hatSb2}) that 
	$$
	\mathbb{P}(|\hat{S}_{b, 2}^t (W, a)| > (1+\zeta)f(d, \mu, a)n) \leq 2 \exp\left( -\frac{\zeta^2 f(d, \mu, a)^2}{4d} n \right)
	$$
	from which the first result (\ref{inequality:Sb2Bound1}) follows. 
	
	Now we establish the second result (\ref{inequality:Sb2Bound2}). 
	Similarly, we have for any $\zeta>0$ 
	\begin{align*}
	\mathbb{P} \left(\left| |\hat{S}_{b,2}^t(U^*, \delta)| - \mathbb{E}(|\hat{S}_{b,2}^t(U^*, \delta)|)\right| \geq \frac{\zeta}{2} n \right) \leq 2 \exp\left( -\frac{(\zeta/2)^2 n^2}{4dn} \right) = 2 \exp\left( -\frac{\zeta^2 }{16d} n \right).
	\end{align*}
	Recall that the probability that the configuration model produces a simple graph is bounded away from zero and does not depend on $n$. Then, 
	\begin{align*}
	\mathbb{P} \left(\left| |S_{b,2}^t(U^*, \delta)| - \mathbb{E}(|S_{b,2}^t(U^*, \delta)|)\right| \geq \frac{\zeta}{2} n \right) \leq 2 \exp\left( -\frac{\zeta^2 }{16d} n + O(1) \right).
	\end{align*}
	It follows from Assumption $2$ that $\mathbb{E}(|S_{b,2}^t(U^*, \delta)|) = o(n)$ and thus
	$$
	\mathbb{P}(|S_{b,2}^t(U^*, \delta)| > \zeta n) \leq 2 \exp\left( -\frac{\zeta^2 }{16d} n + O(1) \right)
	$$
	from which the second result follows.
\end{proof}

\section{Conclusions and Open Questions} \label{section:conclusion} 
We close this paper with several open questions for further research. 
In light of the new algorithm $\mathcal{TAM}$ which improves the sample complexity for the alternating minimization algorithm by a factor $\log n$ for the case of matrix $M$ with bounded rank, condition number and incoherence parameter, a natural direction is to extend this result to the cases when the rank, condition number and incoherence parameter are possibly growing functions of the dimension of $M$. 
In this situation we would be considering the case of growing $d$ for which Assumption $2$ is satisfied automatically by applying Matrix Bernstein inequality. 
On the other hand, under uniform sampling and for the case of growing (average degree) $d$, Hardt \cite{Hardt2014} proposed an augmented alternating minimization algorithm by adding extra smoothing steps typically used in smoothed analysis of the QR factorization. This reduced the dependence of the sample complexity on the rank, condition number and incoherence parameter.  
Perhaps such smoothing steps can be incorporated into the $\mathcal{TAM}$ algorithm as well, possibly leading to a reduced sample and computational complexity when
compared to the one achieved in~\cite{Hardt2014}.

Studying $\mathcal{TAM}$ under i.i.d. uniform sampling, which corresponds to a bipartite \ER graph, is another interesting problem. Instead of using the configuration model, possibly Poisson cloning model can be employed to carry out a similar analysis for the case of a bipartite \ER graph. We conjecture that the same sample complexity of $\mathcal{TAM}$ holds under such uniform sampling. 

Finally, another challenge is to achieve the information theoretic lower bound of sample complexity $O(\mu_0 k n \log n)$ \cite{CandesTao2010} for exact low-rank matrix completion when $k$ is growing. The technique developed in this paper for reducing sample complexity by a $\log n$ factor might be of interest for achieving this goal via more careful analysis of the trace-norm based minimization.     

\bibliographystyle{amsalpha}
\bibliography{bibliography}

\appendix

\section{Proof of Lemma \ref{lemma:distanceBarU0UStar}} \label{ap:proof_of_lemma_distanceBarU0UStar}
\begin{proof}[Proof of Lemma \ref{lemma:distanceBarU0UStar}]
	Let $\bar{U}^0 \Sigma V^T$ be the top-$k$ singular components of $\frac{n}{d} P_{\Omega_0}(M)$. Recall from the property $P_2$ that w.h.p. the second largest singular value of the biadjacency matrix of the random bipartite $d$-regular graph associated with the index set $\Omega_0$ is at most $(7\sqrt{d})/3$. By Theorem 4.1 in \cite{Bhojanapalli2014} where we choose $7/3$ as the constant in this theorem, w.h.p. we have
	\begin{align}
	\label{equation:kComponentError}
	\| M - \bar{U}^0 \Sigma V^T  \|_2 \leq \frac{14 \mu_0 k}{3\sqrt{d}} \| M \|_2.
	\end{align}
	Also we have
	\begin{align*}
	\| M - \bar{U}^0 \Sigma V^T  \|_2 &= \| U^* \Sigma^* V^{*,T} - \bar{U}^0 \bar{U}^{0,T} U^* \Sigma^* V^{*,T} + \bar{U}^0 \bar{U}^{0,T} U^* \Sigma^* V^{*,T} - \bar{U}^0 \Sigma V^T  \|_2  \\
	&= \|(I-\bar{U}^0 \bar{U}^{0,T}) U^* \Sigma^* V^{*,T} + \bar{U}^0 (\bar{U}^{0,T} U^* \Sigma^* V^{*,T} - \Sigma V^T) \|_2.
	\end{align*}
	Since $I - \bar{U}^0 \bar{U}^{0,T}$ is orthogonal to $\bar{U}^0$, we have the right hand side of the equation above
	\begin{align*}
	& \geq \|(I-\bar{U}^0 \bar{U}^{0,T}) U^* \Sigma^* V^{*,T} \|_2 = \| \bar{U}_{\perp}^0 \bar{U}_{\perp}^{0, T} U^* \Sigma^* V^{*,T} \|_2.  
	\end{align*}
	Suppose the SVD of $\bar{U}_{\perp}^{0, T} U^* \Sigma^*$ is $\hat{U} \hat{\Sigma} \hat{V}^T$. 
	Then  
	$$
	\bar{U}_{\perp}^0 \bar{U}_{\perp}^{0, T} U^* \Sigma^* V^{*,T} = \bar{U}_{\perp}^0 \hat{U} \hat{\Sigma} \hat{V}^T V^{*,T}.
	$$ 
	Observe that $\bar{U}_{\perp}^0 \hat{U}$ and $V^{*} \hat{V}$ are both orthonormal matrices. Then $\hat{U} \hat{\Sigma} \hat{V}^T$ has the same singular values as the ones in $\bar{U}_{\perp}^0 \bar{U}_{\perp}^{0, T} U^* \Sigma^* V^{*,T}$, i.e.   
	$$
	\| \bar{U}_{\perp}^0 \bar{U}_{\perp}^{0, T} U^* \Sigma^* V^{*,T} \|_2 = \| \bar{U}_{\perp}^{0, T}  U^* \Sigma^*\|_2.  
	$$
	Let $y \in \mathbb{R}^{n-k}$ be the top left singular vector of $\bar{U}_{\perp}^{0, T} U^*$.
	In particular, $\| y \bar{U}_{\perp}^{0, T} U^* \|_2 = \| \bar{U}_{\perp}^{0, T} U^* \|_2$.  
	Then 
	\begin{align*}
	\| \bar{U}_{\perp}^{0, T}  U^* \Sigma^*\|_2 &= \sup_{x \in \mathbb{R}^{n-k} : \| x \|_2 =1} \| x \bar{U}_{\perp}^{0, T}  U^* \Sigma^*\|_2  \\
	& \geq \| y \bar{U}_{\perp}^{0, T}  U^* \Sigma^*\|_2 \\
	& = \| \bar{U}_{\perp}^{0, T}  U^* \|_2 \| \frac{y \bar{U}_{\perp}^{0, T}  U^* }{\| \bar{U}_{\perp}^{0, T}  U^* \|_2}  \Sigma^*\|_2 \\
	& \geq \| \bar{U}_{\perp}^{0, T}  U^* \|_2 \inf_{z \in \mathbb{R}^{k}: \| z \|_2=1} \| z  \Sigma^*\|_2 \\
	& = \| \bar{U}_{\perp}^{0, T}  U^* \|_2 \sigma_k^*,
	\end{align*}
	which, together with (\ref{equation:kComponentError}), gives 
	\begin{align*}
	\| \bar{U}_{\perp}^{0,T}  U^* \|_2 \leq \frac{14 \mu_0 k}{3 \sqrt{d}} \frac{\sigma_1^*}{\sigma_k^*}. 
	\end{align*}
	The result then follows from $d \geq C k^4 \mu_0^2 (\sigma_1^*/\sigma_k^*)^2$. 
\end{proof}

\section{Proof of Lemma \ref{lemma:TruncateMaintainIncoherence}} \label{ap:proof_of_lemma_TruncateMaintainIncoherence}

\begin{proof}[Proof of Lemma \ref{lemma:TruncateMaintainIncoherence}]
	First, we claim that there exists an orthonormal matrix $R \in \mathbb{R}^{k \times k}$ such that 
	\begin{align}
	\label{inequality:UstarRbarU}
	\|U^* R - \bar{U} \|_F \leq \frac{\sqrt{2}}{\phi}.
	\end{align}
	Consider the SVD of $U^{*,T} \bar{U} = W_1 \Sigma W_2^T$ where $W_1, W_2 \in \mathbb{R}^{k \times k}$ are two orthonormal matrices.
	Since $\| U^{*,T} \bar{U} \|_2 \leq \| U^{*,T} \|_2 \| \bar{U} \|_2 = 1$, all the singular values in $\Sigma$ are within $[0,1]$ and $W_2 W_1^T$ is also an orthonormal matrix. Let $R = W_1 W_2^T$. Then we have
	\begin{align*}
	\|U^* R - \bar{U} \|_2^2 &= \| (U^* R - \bar{U})^T (U^* R - \bar{U})  \|_2  \\
	&= \| 2 I - R^T U^{*,T} \bar{U} - \bar{U}^T U^* R \|_2 \\
	&= \| 2 I - 2 W_2 \Sigma W_2^T  \|_2  \\
	&= 2 \| W_2 (I - \Sigma) W_2^T \|_2 \\
	&= 2 \| I - \Sigma \|_2
	\end{align*}
	Let $\gamma = \dist( \bar{U}, U^* )$. 
	By the property (\ref{dist_subspace_property3}) of subspace distance and $\gamma = \| U_{\perp}^{*,T} \bar{U}  \|_2$, the least singular value of $ U^{*,T} \bar{U} $ is $\sqrt{1-\gamma^2}$. Then the inequality above becomes  
	\begin{align*}
	\|U^* R - \bar{U} \|_2^2 = 2 (1 - \sqrt{1 - \gamma^2}) \leq 2 \gamma^2 
	\end{align*}
	which, together with the inequality (\ref{inequality: matrix_inequality1}) where $l = k$, implies $\|U^* R - \bar{U} \|_F \leq \sqrt{2k} \gamma$. Then the claim follows from $\gamma = \dist( \bar{U}, U^* ) \leq 1/(\phi k^{1/2})$.   
	
	Let $\bar{u}_i^{T}, \hat{u}_i^T, u_i^{*,T}$, $ i \in [n]$, be the $i$-th row of the matrices $\bar{U}$, $\hat{U}$ and $U^*$. 
	We claim
	\begin{align}
	\label{inequality:UbarUhat}
	\| \bar{u}_i^T - \hat{u}_i^T \|_2 \leq \| \bar{u}_i^T - u_i^{*,T} R \|_2 \quad \forall  i \in [n].
	\end{align}
	We will establish this claim by considering the case $\| \bar{u}_i^T \|_2 \geq  2\sqrt{ \mu_0 k / n}$ and the case $\|\bar{u}_i^T \|_2 < 2 \sqrt{\mu_0 k / n}$, respectively. Consider the case $\| \bar{u}_i^T \|_2 \geq  2\sqrt{ \mu_0 k / n}$. Applying the operator $\mathcal{T}_1$ on $\bar{u}_i$ truncates $\bar{u}_i$ to $\hat{u}_i$ of the same direction and of length $\sqrt{\mu_0 k / n}$, which gives $\|\hat{u}_i\|_2 = \sqrt{\mu_0 k / n}$ and thus
	\begin{align*}
	\| \bar{u}_i^T - \hat{u}_i^T  \|_2 = \| \bar{u}_i^T \|_2 - \sqrt{\frac{ \mu_0 k}{n}}.
	\end{align*}
	Notice that the orthonormal transformation does not change the length of $u_i^{*}$, that is, 
	$$
	\| u_i^{*,T} R \|_2 = \| u_i^{*,T} \|_2  \leq \sqrt{\mu_0 k / n}.
	$$
	The triangle inequality gives 
	$$
	\| \bar{u}_i^T - \hat{u}_i^T  \|_2 = \| \bar{u}_i^T \|_2 - \sqrt{\frac{ \mu_0 k}{n}} \leq \| \bar{u}_i^T \|_2 - \| u_i^{*, T} R \|_2 \leq \| \bar{u}_i^T - u_i^{*, T} R \|_2,
	$$
	and the claim is established for the case $\| \bar{u}_i^T \|_2 \geq  2\sqrt{ \mu_0 k / n}$.

	Suppose now $\bar{u}_i^T$ satisfies $\| \bar{u}_i^T \|_2 < 2 \sqrt{ \mu_0 k / n }$. It follows from (\ref{definition:T1}) that $\hat{u}_i^T =\mathcal{T}_1(\bar{u}_i^T) =\bar{u}_i^T $ and thus $\| \bar{u}_i^T - \hat{u}_i^T \|_2 = 0$. Then the claim follows. Thus it follows from (\ref{inequality:UbarUhat}) and (\ref{inequality:UstarRbarU}) that  
	\begin{align}
	\label{inequality:upperbound_barU_hatU}
	\| \bar{U} - \hat{U}\|_F = \sqrt{\sum_{i=1}^n \| \bar{u}_i^{T} - \hat{u}_i^{T}\|_2^2 } \leq \sqrt{\sum_{i=1}^n \| \bar{u}_i^T - u_i^{*,T} R \|^2 } = \| \bar{U} - U^{*} R \|_F \leq \frac{\sqrt{2}}{\phi}. 
	\end{align}
	Applying Ky Fan singular value inequality (\ref{KyFanSVInequality}) to $\bar{U} = \hat{U} + (\bar{U}-\hat{U})$ gives   
	$$ 
	\sigma_k(\bar{U}) \leq   \sigma_{k}(\hat{U}) + \sigma_{1}(\bar{U}-\hat{U}).
	$$ 
	Since $\bar{U} \in \mathbb{R}^{n \times k}$ is an orthonormal matrix, we have $\sigma_{k}(\bar{U}) = 1$. Also we have $\sigma_{1}(\bar{U}-\hat{U}) \leq \| \bar{U}-\hat{U} \|_F \leq \sqrt{2}/\phi$. Then using $\phi \geq \sqrt{10}/(\sqrt{5} - 2)$, we obtain
	$$
	\sigma_{k}(\hat{U}) \geq \sigma_{k}(\bar{U})-\sigma_{1}(\bar{U}-\hat{U}) \geq 1 - \frac{\sqrt{2}}{\phi} \geq \frac{2}{\sqrt{5}}.
	$$
	We can write $U = \hat{U} Q^{-1}$ where $Q$ is an invertible matrix with the same singular values as $\hat{U}$. This, together with the inequality above, implies
	$$
	\| Q^{-1} \|_2 = \frac{1}{\sigma_{k}(\hat{U})} \leq \frac{\sqrt{5}}{2}. 
	$$
	Since $\hat{u}_i^T$ is obtained by applying the operations $\mathcal{T}_1$ on $\bar{u}_i^T$, we have $\| \hat{u}_i \| < 2 \sqrt{\mu_0 k/n}$ for all $i \in [n]$. Therefore for all $i \in [n]$
	\begin{align*}
	\|u_i^{T} \|_2 = \|\hat{u}_i^T Q^{-1} \|_2 \leq \|\hat{u}_i\|_2 \| Q^{-1} \|_2  \leq \frac{\sqrt{5}}{2} \times 2 \sqrt{\frac{  \mu_0 k}{n}}  = \sqrt{\frac{ 5\mu_0 k}{n}},
	\end{align*}
	and (\ref{incoherenceui}) is established. 
	Finally,
	\begin{align*}
	\dist(U, U^*) = \|U_{\perp}^{*,T} U \|_2 = & \|U_{\perp}^{*,T} \hat{U} Q^{-1}\|_2 \\
	\leq &  \|U_{\perp}^{*,T} \hat{U} \|_2 \| Q^{-1} \|_2 \\
	\leq & \frac{\sqrt{5}}{2} \|U_{\perp}^{*,T} \hat{U} \|_2  \\
	\leq & \frac{\sqrt{5}}{2} \left( \|U_{\perp}^{*,T} \bar{U} \|_2 + \|U_{\perp}^{*,T} (\hat{U} - \bar{U}) \|_2 \right)  \\
	\leq & \frac{\sqrt{5}}{2} (\|U_{\perp}^{*,T} \bar{U} \|_2 + \| \bar{U} - \hat{U}\|_2).
	\end{align*}
	Recall from the assumptions of this lemma that $\|U_{\perp}^{*,T} \bar{U} \|_2 = \dist(\bar{U}, U^*) \leq 1/(\phi k^{1/2})$ and from (\ref{inequality:upperbound_barU_hatU}) that $\| \bar{U} - \hat{U} \|_2 \leq \| \bar{U} - \hat{U} \|_F \leq \sqrt{2}/\phi$. (\ref{distU}) then follows from 
	\begin{align*}
	\dist(U, U^*) \leq \frac{\sqrt{5}}{2} \left( \frac{1}{\phi k^{1/2}}   + \frac{\sqrt{2}}{\phi} \right) \leq \frac{\sqrt{10}}{\phi}. 
	\end{align*}
\end{proof}

\section{Proof of Proposition \ref{proposition:upperboundGammat_Epsilon}} \label{ap:proof_of_proposition_upperboundGammat_Epsilon}
\begin{proof}[Proof of Proposition \ref{proposition:upperboundGammat_Epsilon}]
We prove (\ref{inequality:Gamma_t_20}) for the cases $\mathrm{e}^2 \rho_t \gamma_t/\alpha > 1$ and $\mathrm{e}^2 \rho_t \gamma_t/\alpha \leq 1$, separately. 
Consider the case $\mathrm{e}^2 \rho_t \gamma_t/\alpha > 1$. 
Recall $\rho_t$ given in (\ref{alpharhot}).   
We first derive an upper bound on $\rho_t$. It follows from $\gamma_t \in (0, 4\sqrt{10}/(\sqrt{C(\delta, \beta) } \mu_0 k^{1.5}))$ that
\begin{align*}
\frac{(1-\beta-\delta)^2}{24 \mu_0 k} - 3 \gamma_t^2 \mu_0 k \geq \frac{(1-\beta-\delta)^2}{24 \mu_0 k} - 3 \frac{160}{C(\delta, \beta) \mu_0^2 k^3} \mu_0 k = \frac{(1-\beta-\delta)^2}{24 \mu_0 k} - \frac{480 }{C(\delta, \beta) \mu_0 k^2 }. 
\end{align*}
Then we can choose a large enough $C(\delta, \beta)>0$ such that
\begin{align}
\label{inequality:UpperBoundRho_t}
\rho_t = \frac{2 k}{\frac{(1-\beta-\delta)^2}{24 \mu_0 k} - 3 \gamma_t^2 \mu_0 k} \leq  \frac{2 k}{\frac{(1-\beta-\delta)^2}{48 \mu_0 k}}  =\frac{96 \mu_0 k^2}{(1-\beta-\delta)^2}.
\end{align} 
Recall $\alpha=(1-\beta-\delta)/(12\mu_0 k)$ in (\ref{alpharhot}). Then,
\begin{align*}
\frac{ \mathrm{e}^2 \rho_t \gamma_t^2}{\alpha}  \leq  \frac{12\mathrm{e}^2 \mu_0 k}{1-\beta-\delta} \frac{96 \mu_0 k^2 }{(1-\beta-\delta)^2} \frac{160}{C(\delta, \beta) k^3 \mu_0^2} = \frac{12 \times 96 \mathrm{e}^2 }{(1-\beta-\delta)^3} \frac{160 }{C(\delta, \beta)} .
\end{align*} 
Now we have the left hand side of (\ref{inequality:Gamma_t_20})
\begin{align*}
& \frac{14}{\beta}  \frac{\sigma_1^*}{\sigma_k^*} (\mu_0 k)^{1.5} \sqrt{1.1 \mathrm{e}} \left(\frac{\mathrm{e}^2 \rho_t \gamma_t^2}{\alpha} \right)^{\alpha d /2} \\
 \leq \quad & \frac{14}{\beta} (\mu_0 k)^{1.5} \frac{\sigma_1^*}{\sigma_k^*} \sqrt{1.1 \mathrm{e}}  \left( \frac{12 \times 96 \mathrm{e}^2 }{(1-\beta-\delta)^3} \frac{160}{C(\delta, \beta)} \right)^{\frac{ 1-\beta-\delta }{24 \mu_0 k} d}   \\
 = \quad& \exp \Bigg(\log \bigg( \frac{14 \sqrt{1.1 \mathrm{e}}}{\beta} \bigg) + 1.5 \log \mu_0 + 1.5 \log k + \log \left( \frac{\sigma_1^*}{\sigma_k^*} \right) \\
 	&  \quad +  d \frac{ 1-\beta-\delta }{24 \mu_0 k}  \log \left( \frac{12 \times 96 \mathrm{e}^2 }{(1-\beta-\delta)^3} \frac{160}{C(\delta, \beta)} \right) \Bigg).
\end{align*}
Then for $d \geq C(\delta, \beta) k^4 \mu_0^2 (\sigma_1^* / \sigma_k^*)^2$, the last term in the exponent above is a polynomial of $\mu_0$, $k$ and $\sigma_1^* \slash \sigma_k^* $ while the rest terms are the linear combination of $\log \mu_0$, $\log k$ and $\log(\sigma_1^* \slash \sigma_k^*)$. Also observe that a large $C(\delta, \beta)$ leads to a negative coefficient of the last term in the exponent above.       
Hence the following inequality holds for a large $C(\delta, \beta)$ and $d \geq C(\delta, \beta) k^4 \mu_0^2 (\sigma_1^* / \sigma_k^*)^2$
\begin{align*}
\frac{14}{\beta}  \frac{\sigma_1^*}{\sigma_k^*} (\mu_0 k)^{1.5} \sqrt{1.1 \mathrm{e}} \left(\frac{\mathrm{e}^2 \rho_t \gamma_t^2}{\alpha} \right)^{\alpha d /2}  \leq \frac{1}{20 \sqrt{10 k} } \frac{(1-\beta-\delta)^2}{ 96 \mu_0 k^2} \frac{1-\beta-\delta}{12 \mathrm{e}^2 \mu_0 k}. 
\end{align*}
Finally by the upper bound on $\rho_t$ in (\ref{inequality:UpperBoundRho_t}) and $\mathrm{e}^2 \rho_t \gamma_t/\alpha > 1$, we have
\begin{align*}
\frac{1}{20 \sqrt{10 k} } \frac{(1-\beta-\delta)^2}{ 96 \mu_0 k^2} \frac{1-\beta-\delta}{12 \mathrm{e}^2 \mu_0 k} &\leq  \frac{1}{20 \sqrt{10 k}} \frac{1}{\rho_t} \frac{\alpha}{\mathrm{e}^2}  \leq \frac{\gamma_t}{20 \sqrt{10 k} },
\end{align*}
which gives (\ref{inequality:Gamma_t_20}) for the cases $\mathrm{e}^2 \rho_t \gamma_t/\alpha > 1$. 

Consider the case $\mathrm{e}^2 \rho_t \gamma_t/\alpha \leq 1$. Then we have the following upper bound on the left hand side of (\ref{inequality:Gamma_t_20}) 
\begin{align*}
\frac{14}{\beta}   \frac{\sigma_1^*}{\sigma_k^*} (\mu_0 k)^{1.5} \sqrt{1.1 \mathrm{e}} \left(\frac{\mathrm{e}^2 \rho_t \gamma_t^2}{\alpha} \right)^{\alpha d /2} &\leq \frac{14}{\beta} \frac{\sigma_1^*}{\sigma_k^*} (\mu_0 k)^{1.5} \sqrt{1.1 \mathrm{e}} \gamma_t^{\frac{ 1-\beta-\delta }{24 \mu_0 k} d}.
\end{align*}
Similarly it follows from $\gamma_t \leq 4\sqrt{10}/(\sqrt{C(\delta, \beta) } \mu_0 k^{1.5})$ and $d \geq C(\delta, \beta) k^4 \mu_0^2 (\sigma_1^* / \sigma_k^*)^2$ that for a large $C(\delta, \beta)>0$, the following inequality holds
$$
\frac{14}{\beta} \frac{\sigma_1^*}{\sigma_k^*} (\mu_0 k)^{1.5} \sqrt{1.1 \mathrm{e}} \gamma_t^{\alpha d /2} \leq \frac{\gamma_t}{20 \sqrt{10 k} }.
$$
Hence the inequality (\ref{inequality:Gamma_t_20}) follows.

Now we show the inequality (\ref{inequality:Epsilon_40}). Recall the definition of $f(d, \mu_0, \delta)$ in (\ref{equation:fd}). Then the left hand side of (\ref{inequality:Epsilon_40}) becomes 
\begin{align}
\label{equation:boundEpsilon_40_1}
& \exp\left(\log \left(\frac{14}{\beta} \frac{\sigma_1^*}{\sigma_k^*} (\mu_0 k)^{1.5} \right) + \frac{1}{2} \log (1.1) + \frac{1}{2} \log f(d, \mu_0, \delta) \right) \nonumber\\
 = \quad & \exp\Bigg(\log \left(\frac{14}{\beta} \frac{\sigma_1^*}{\sigma_k^*} (\mu_0 k)^{1.5} \right) + \frac{1}{2} \log (1.1) + \frac{1}{2} \log (3 k \sqrt{\pi })  \nonumber \\
 	& \quad + \frac{1}{4} \log d - \frac{\delta^2}{4\mu_0 k (1+\delta/3 )}  d \Bigg).
\end{align}
Let
$$
g(d) = \frac{1}{4} \log d - \frac{\delta^2}{4\mu_0 k (1+\delta/3 )}  d.
$$
Then the derivative of $g(d)$ is 
$$
g'(d) = \frac{1}{4 d} - \frac{\delta^2}{4\mu_0 k (1+\delta/3 )}.
$$
For a large $C(\delta, \beta)>0$, $g'(d)$ is always negative for any $d$ satisfying (\ref{inequality:dLowerBound1}), that is,  
$$
d \geq C(\delta, \beta) k^4 \mu_0^2 \left(\frac{\sigma_1^*}{\sigma_k^*} \right)^2   + \frac{5 \mu_0 k (1+ \delta/3)}{\delta^2} \log \left( \frac{1}{\epsilon} \right).
$$
Then the right hand side of (\ref{equation:boundEpsilon_40_1}) is
\begin{align*}
&\leq \exp \Bigg(\log \left(\frac{14}{\beta} (\mu_0 k)^{1.5} \frac{\sigma_1^*}{\sigma_k^*} \right) + \frac{1}{2} \log (1.1) + \frac{1}{2} \log (3 k \sqrt{\pi }) \\
& \quad \quad \quad \quad      + \frac{1}{4} \log \left(C(\delta, \beta) k^4 \mu_0^2 \left(\frac{\sigma_1^*}{\sigma_k^*} \right)^2   + \frac{5 \mu_0 k (1+ \delta/3)}{\delta^2} \log \left( \frac{1}{\epsilon} \right) \right) \\
& \quad \quad \quad \quad      - \frac{\delta^2}{4\mu_0 k (1+\delta/3 )}  \left(C(\delta, \beta) k^4 \mu_0^2 \left(\frac{\sigma_1^*}{\sigma_k^*} \right)^2   + \frac{5 \mu_0 k (1+ \delta/3)}{\delta^2} \log \left( \frac{1}{\epsilon} \right) \right)  \Bigg).
\end{align*}
Using $\log(x+y) \leq \log x + \log y $ for $x,y \geq 2$, we obtain the right hand side of the inequality above
\begin{align*}
&\leq \exp \Bigg(\log \left(\frac{14}{\beta} (\mu_0 k)^{1.5} \frac{\sigma_1^*}{\sigma_k^*} \right) + \frac{1}{2} \log (1.1) + \frac{1}{2} \log (3 k \sqrt{\pi }) \\
& \quad \quad \quad \quad      + \frac{1}{4} \log \left(C(\delta, \beta) k^4 \mu_0^2 \left(\frac{\sigma_1^*}{\sigma_k^*} \right)^2 \right)   + \frac{1}{4} \log \left(\frac{5 \mu_0 k (1+ \delta/3)}{\delta^2} \log \left( \frac{1}{\epsilon} \right) \right) \\
& \quad \quad \quad \quad      - \frac{\delta^2}{4\mu_0 k (1+\delta/3 )}  \left(C(\delta, \beta) k^4 \mu_0^2 \left(\frac{\sigma_1^*}{\sigma_k^*} \right)^2   + \frac{5 \mu_0 k (1+ \delta/3)}{\delta^2} \log \left( \frac{1}{\epsilon} \right) \right)  \Bigg). 
\end{align*}
Using $\log \log(1/\epsilon) \leq \log (1/\epsilon)$ for all $\epsilon \in (0, 2/3)$, we have the right hand side of the inequality above 
\begin{align*}
& \leq \exp \Bigg(\log \left(\frac{14}{\beta} (\mu_0 k)^{1.5} \frac{\sigma_1^*}{\sigma_k^*} \right) + \frac{1}{2} \log (1.1) + \frac{1}{2} \log (3 k \sqrt{\pi }) \\
& \quad \quad \quad \quad      + \frac{1}{4} \log \left(C(\delta, \beta) k^4 \mu_0^2 \left(\frac{\sigma_1^*}{\sigma_k^*} \right)^2 \right)   + \frac{1}{4} \log \left(\frac{5 \mu_0 k (1+ \delta/3)}{\delta^2}  \right)  \\
& \quad \quad \quad \quad      - \frac{\delta^2 C(\delta, \beta) k^3 \mu_0}{ 4(1+\delta/3) } \left(\frac{\sigma_1^*}{\sigma_k^*} \right)^2 -   \log \left( \frac{1}{\epsilon} \right)   \Bigg).
\end{align*}
Observe that the negative terms in the exponent above are polynomial of $\mu_0$, $k$ and $\sigma_1^* \slash \sigma_k^*$ while the positive terms are linear combination of $\log \mu_0$, $\log k$ and $\log(\sigma_1^* \slash \sigma_k^*)$. Hence for a large enough $C(\delta, \beta)>0$, the right hand side of the inequation above is no more than $\epsilon/(40\sqrt{10 k})$, from which the inequality∫ (\ref{inequality:Epsilon_40}) follows.  
\end{proof}

\section{Proof of Lemma \ref{probabilityRate}} \label{ap:proof_of_lemma_probability_rate}

\begin{proof}[Proof of Lemma \ref{probabilityRate}]
	Consider the logarithm of each term in (\ref{equation:probR}) normalized by $dn$. Using Stirling's approximation $a! \approx \sqrt{2\pi a} (a/\mathrm{e})^a$, we have
	\begin{align*}
	\frac{1}{dn} \log { n \choose {\beta n} } =& \frac{1}{dn} \log \frac{n!}{((1-\beta)n)! (\beta n)!}   \\
	=&  o(1) + \frac{1}{dn} \log \frac{\sqrt{2\pi n} n^n}{\sqrt{2\pi(1-\beta)n}((1-\beta)n)^{(1-\beta)n} \sqrt{2\pi \beta n} (\beta n)^{\beta n}}                                   \\
	=& o(1) + \frac{1}{dn} \left( n \log n - (1-\beta)n \log((1-\beta)n) - \beta n \log(\beta n)  \right) \\
	=& o(1) - \frac{ (1-\beta) \log (1-\beta) + \beta \log \beta }{d}, 
	\end{align*}
	Notice that $(\log (\sqrt{n}))/n = o(1)$. In the following expansion of $a!$, for convenience we will not explicitly write down the term $\sqrt{2 \pi a}$.  
	\begin{align*}
	&\frac{1}{dn} \log I_1  \\
	= \quad &  \frac{1}{dn} \log \left( \frac{(\beta d n)!}{(\theta \beta d n)! ((1 - \theta ) \beta d n)!}  f_1  \frac{(\hat{\rho} \gamma_t^2 d n)!}{(\theta \beta d n)! ((\hat{\rho} \gamma_t^2 - \theta \beta) d n)!}  (\theta \beta d n)! \right) \\
	= \quad & \frac{1}{dn} \log \left( f_1 \frac{(\beta d n)!}{ (\theta \beta d n)! ((1 - \theta ) \beta d n)!}  \frac{(\hat{\rho} \gamma_t^2 d n)!}{ ((\hat{\rho} \gamma_t^2 - \theta \beta) d n)!}  \right)  \\
	= \quad & o(1) + \frac{1}{dn} \log \Big( f_1 \frac{(\beta d n)^{\beta d n}}{  (\theta \beta d n)^{\theta \beta d n}  ((1-\theta)\beta d n)^{(1-\theta)\beta d n}} \\
	& \quad \quad \quad \quad \quad \quad  \times \frac{ (\hat{\rho} \gamma_t^2 d n)^{\hat{\rho} \gamma_t^2 d n} \exp(-\hat{\rho} \gamma_t^2 d n)}{ ((\hat{\rho} \gamma_t^2 - \theta \beta) d n)^{(\hat{\rho} \gamma_t^2 - \theta \beta) d n} \exp(-(\hat{\rho} \gamma_t^2 - \theta \beta) d n) }    \Big) \\
	= \quad & o(1) + \frac{1}{dn} \log \Big( f_1 \frac{ (\beta d n)^{\beta d n}}{ (\theta \beta d n)^{\theta \beta d n} ((1-\theta)\beta d n)^{(1-\theta)\beta d n}} \frac{ (\hat{\rho} \gamma_t^2 d n)^{\hat{\rho} \gamma_t^2 d n}}{ ((\hat{\rho} \gamma_t^2 - \theta \beta) d n)^{(\hat{\rho} \gamma_t^2 - \theta \beta) d n}} \exp(-\theta \beta d n)    \Big) \\
	= \quad & o(1) + \frac{1}{dn} \log f_1 + \beta \log(\beta d n) - \theta \beta \log (\theta \beta d n)  - (1-\theta)\beta \log((1-\theta)\beta d n) \\
	& +\hat{\rho}\gamma_t^2 \log(\hat{\rho}\gamma_t^2 d n) - (\hat{\rho}\gamma_t^2-\theta\beta) \log((\hat{\rho}\gamma_t^2 - \theta \beta) d n) - \theta \beta                                            \\
	= \quad & o(1) + \frac{1}{dn} \log f_1 + \beta \log \beta - \theta \beta \log (\theta \beta) - (1-\theta) \beta \log ((1 - \theta ) \beta) + \hat{\rho} \gamma_t^2 \log (\hat{\rho} \gamma_t^2) \\
	&  - (\hat{\rho} \gamma_t^2 - \theta \beta) \log (\hat{\rho} \gamma_t^2 - \theta \beta) + \theta \beta \log(dn) - \theta \beta, 
	\end{align*}
	{
		\begin{align*}
		& \frac{1}{dn} \log I_2 \\
		= ~ & \frac{1}{dn} \log \left(  \frac{((1-\hat{\rho}\gamma_t^2) d n)! }{ ((1-\hat{\rho} \gamma_t^2-\beta + \theta \beta)dn)!}          \right) \\ 
		= ~ & \frac{1}{dn} \log \left(  \frac{ ((1-\hat{\rho}\gamma_t^2)dn)^{(1-\hat{\rho}\gamma_t^2)dn} \exp(-(1-\hat{\rho}\gamma_t^2)dn)}{\ ((1-\hat{\rho}\gamma_t^2-\beta+\theta \beta) dn)^{(1-\hat{\rho}\gamma_t^2-\beta+\theta \beta) dn} \exp(-(1-\hat{\rho}\gamma_t^2-\beta+\theta \beta) dn) }                            \right) + o(1)               \\
		= ~ & (1-\hat{\rho} \gamma_t^2) \log ((1-\hat{\rho} \gamma_t^2)dn) - (1-\hat{\rho} \gamma_t^2 - \beta + \theta \beta) \log ((1 - \hat{\rho} \gamma_t^2 - \beta + \theta \beta)dn)  \\
		& - (1 - \theta) \beta + o(1)  \\
		= ~ & (1-\hat{\rho} \gamma_t^2) \log (1-\hat{\rho} \gamma_t^2) - (1-\hat{\rho} \gamma_t^2 - \beta + \theta \beta) \log (1 - \hat{\rho} \gamma_t^2 - \beta + \theta \beta) + (1 - \theta ) \beta \log (dn) \\
		& - (1 - \theta) \beta + o(1),
		\end{align*}
	}
	{
		\begin{align*}
		& \frac{1}{dn} \log I_3 \\
		= ~ & \frac{1}{dn} \log \left(  f_2 \frac{((1-\beta)dn)!}{((1-\beta-\hat{\rho}\gamma_t^2+\theta \beta)dn)!}         \right) \\
		= ~ & \frac{1}{dn} \log \left(  f_2 \frac{((1-\beta)dn)^{(1-\beta)dn} \exp(-(1-\beta)dn)}{((1-\beta-\hat{\rho}\gamma_t^2+\theta \beta)dn)^{(1-\beta-\hat{\rho}\gamma_t^2+\theta \beta)dn} \exp(-(1-\beta-\hat{\rho}\gamma_t^2+\theta \beta)dn)}         \right)   + o(1) \\
		= ~ & \frac{1}{dn} \log f_2 + (1-\beta) \log((1-\beta)dn) - (1-\beta-\hat{\rho}\gamma_t^2+\theta \beta) \log((1-\beta-\hat{\rho}\gamma_t^2+\theta \beta)dn)  \\
		& -\hat{\rho}\gamma_t^2+\theta \beta + o(1) \\
		= ~ & \frac{1}{dn} \log f_2 + (1-\beta) \log (1-\beta) - (1-\beta - \hat{\rho} \gamma_t^2 + \theta \beta) \log (1-\beta-\hat{\rho} \gamma_t^2 + \theta \beta) \\
		& + ( \hat{\rho} \gamma_t^2 - \theta \beta) \log (dn) - (\hat{\rho} \gamma_t^2 - \theta \beta) + o(1),
		\end{align*}
	}
	\begin{align*}
	&\frac{1}{dn} \log I_4 \\
	= ~ &  \frac{1}{dn} \log \left(  ((1-\beta-\hat{\rho} \gamma_t^2 + \theta \beta) dn)^{(1-\beta-\hat{\rho} \gamma_t^2 + \theta \beta) dn} \exp(-(1-\beta-\hat{\rho} \gamma_t^2 + \theta \beta) dn)   \right) + o(1) \\
	= ~ & (1-\beta-\hat{\rho} \gamma_t^2 + \theta \beta) \log (1-\beta-\hat{\rho} \gamma_t^2 + \theta \beta) + (1-\beta-\hat{\rho} \gamma_t^2 + \theta \beta) \log (dn) \\
	& - (1-\beta-\hat{\rho} \gamma_t^2 + \theta \beta)  + o(1)
	\end{align*}
	$$
	\frac{1}{dn} \log ((dn)!) = \frac{1}{dn} \log ((dn)^{dn} \exp(-dn)) + o(1) =o(1) +  \log (dn) - 1. \qquad \qquad \qquad 
	$$
	Combining these terms above, the expression of $ \log \mathbb{P}(\mathcal{R}(\beta n, \alpha d, \theta \beta d n))$ normalized by $dn$ is rewritten as follows where both the terms with factor $\log(dn)$ and without $\log(\cdot)$ factor cancel out. 
	\begin{align}
	\label{Equation:ProbabilityRateStep2}
	& \frac{1}{dn} \log \mathbb{P}(\mathcal{R}(\beta n, \alpha d, \theta \beta d n))   \nonumber \\
	= ~ & o(1) - \frac{\beta \log \beta + (1-\beta) \log (1-\beta)}{d} +  \frac{1}{dn}   (\log f_1 + \log f_2 ) + \beta \log \beta - \theta \beta \log (\theta \beta) \nonumber \\
	&  - (1-\theta) \beta \log ((1 - \theta) \beta) + \hat{\rho} \gamma_t^2 \log(\hat{\rho} \gamma_t^2) - (\hat{\rho} \gamma_t^2 - \theta \beta) \log (\hat{\rho} \gamma_t^2 - \theta \beta)  + (1-\hat{\rho} \gamma_t^2) \log (1-\hat{\rho} \gamma_t^2) \nonumber\\
	&  - (1-\hat{\rho} \gamma_t^2-\beta+\theta \beta) \log (1-\hat{\rho} \gamma_t^2-\beta+\theta \beta) + (1-\beta) \log (1-\beta).
	\end{align}
	Next we divide the terms on the right hand side of the equation (\ref{Equation:ProbabilityRateStep2}) into five groups and then derive upper bounds on them, respectively. 
	Using the fact $\log(1-a) \geq -a/(1-a)$ for $a \in [0,1)$, we have for the first group
	\begin{align}
	\label{inequality:group1}
	- \frac{\beta \log \beta + (1-\beta) \log (1-\beta)}{d} \leq -\frac{\beta \log \beta}{d} + \frac{1-\beta}{d} \frac{\beta}{1-\beta} = -\frac{\beta \log \beta}{d} + \frac{\beta}{d}.
	\end{align}
	Since $f_1, f_2 \in [0,1]$, we have 
	\begin{align}
	\label{inequality:group2}
	\frac{1}{dn}   (\log f_1 + \log f_2 ) \leq 0.
	\end{align}
	For the third group, we have
	\begin{align}
	\label{inequality:group3}
	& \beta \log \beta - \theta \beta \log (\theta \beta) - (1 - \theta) \beta \log ((1-\theta) \beta)  \nonumber\\
	= &  \beta \log \beta - \theta \beta \log \theta - \theta \beta \log \beta - (1 - \theta) \beta \log (1 - \theta) - (1 - \theta) \beta \log \beta   \nonumber\\
	= & -\beta (\theta \log \theta + (1-\theta) \log (1-\theta)).
	\end{align}
	Using the fact $\log(1-a) \geq -a/(1-a)$ for $a \in [0,1)$ again, we have for the fourth group
	\begin{align}
	\label{inequality:group4}
	& \hat{\rho} \gamma_t^2 \log(\hat{\rho} \gamma_t^2) - (\hat{\rho} \gamma_t^2 - \theta \beta) \log (\hat{\rho} \gamma_t^2 - \theta \beta)    \nonumber\\
	= &  \hat{\rho} \gamma_t^2 \log(\hat{\rho} \gamma_t^2) - (\hat{\rho} \gamma_t^2 - \theta \beta) \left( \log (\hat{\rho} \gamma_t^2) + \log (1-\frac{\theta \beta}{\hat{\rho} \gamma_t^2}) \right)   \nonumber\\
	\leq &  \hat{\rho} \gamma_t^2 \log(\hat{\rho} \gamma_t^2) - (\hat{\rho} \gamma_t^2 - \theta \beta) \log (\hat{\rho} \gamma_t^2) + (\hat{\rho} \gamma_t^2 - \theta \beta) \frac{\theta \beta/(\hat{\rho} \gamma_t^2)}{1- \theta \beta/(\hat{\rho} \gamma_t^2)}   \nonumber\\
	= & \theta \beta \log (\hat{\rho} \gamma_t^2) + \theta \beta.   
	\end{align}
	It follows from the non-negativity of (\ref{number_edges_cross1}) and the right hand side of (\ref{number_edges_cross2}) that 
	$$
	\theta \beta \leq \hat{\rho} \gamma_t^2  \text{ and } \beta(1 - \theta) \leq 1 - \hat{\rho} \gamma_t^2.
	$$
	Also $|Q^t(\tau)|=\hat{\rho}\gamma_t^2 n < n$ gives $1 - \hat{\rho} \gamma_t^2 > 0$. Then we have for the fifth group 
	{\small
		\begin{align}
		\label{inequality:barRhoGammaTheta0}
		& (1-\hat{\rho} \gamma_t^2) \log (1-\hat{\rho} \gamma_t^2) - (1-\hat{\rho} \gamma_t^2-\beta+ \theta\beta) \log (1-\hat{\rho} \gamma_t^2-\beta + \theta \beta ) + (1-\beta) \log (1-\beta)  \nonumber\\
		= ~ & (1-\hat{\rho} \gamma_t^2) \log (1-\hat{\rho} \gamma_t^2) + (\hat{\rho} \gamma_t^2  - \theta \beta)  \log (1-\hat{\rho} \gamma_t^2 - \beta + \theta \beta) - (1-\beta) \log \frac{1-\hat{\rho}\gamma_t^2 - \beta + \theta \beta}{1-\beta}   \nonumber\\
		= ~ & (1-\hat{\rho} \gamma_t^2) \log (1-\hat{\rho} \gamma_t^2) + (\hat{\rho} \gamma_t^2  - \theta \beta)  \log (1-\hat{\rho} \gamma_t^2) + (\hat{\rho} \gamma_t^2 - \theta \beta) \log(1 - \frac{\beta(1-\theta)}{1 - \hat{\rho} \gamma_t^2}) \nonumber\\
		& - (1-\beta) \log \frac{1-\hat{\rho}\gamma_t^2 - \beta + \theta \beta}{1-\beta}.
		\end{align}
	}
	Using $\log(1-a) \leq -a $ for $a \in [0,1)$, the right hand side of (\ref{inequality:barRhoGammaTheta0}) is upper bounded by
	{\small
		\begin{align}
		\label{inequality:barRhoGammaTheta}
		\leq ~ & (1-\theta \beta) \log (1 -\hat{\rho} \gamma_t^2) - (\hat{\rho} \gamma_t^2 - \theta \beta) \frac{\beta(1-\theta)}{1-\hat{\rho} \gamma_t^2} - (1-\beta) \log \frac{1-\hat{\rho}\gamma_t^2-\beta+\theta \beta}{1-\beta} \nonumber\\
		= ~ & (\beta-\theta \beta) \log(1-\hat{\rho}\gamma_t^2) + (1-\beta)\log(1-\hat{\rho}\gamma_t^2) - (\hat{\rho} \gamma_t^2 - \theta \beta) \frac{\beta(1-\theta)}{1-\hat{\rho} \gamma_t^2} \nonumber \\
		& - (1-\beta) \log \frac{1-\hat{\rho}\gamma_t^2-\beta+\theta \beta}{1-\beta}  \nonumber \\
		= ~ & \beta(1-\theta) \log(1-\hat{\rho} \gamma_t^2) - (\hat{\rho} \gamma_t^2 - \theta \beta) \frac{\beta(1-\theta)}{1-\hat{\rho} \gamma_t^2} + (1-\beta) \log \frac{(1-\hat{\rho}\gamma_t^2)(1-\beta)}{1-\hat{\rho}\gamma_t^2-\beta+\theta \beta}  \nonumber\\
		= ~ & \beta(1-\theta) \log(1-\hat{\rho} \gamma_t^2) + (\theta \beta - \hat{\rho} \gamma_t^2) \frac{\beta(1-\theta)}{1-\hat{\rho} \gamma_t^2} + (1-\beta) \log \frac{1-\hat{\rho}\gamma_t^2 - \beta + \hat{\rho}\gamma_t^2 \beta}{1-\hat{\rho}\gamma_t^2-\beta+\theta \beta}. 
		\end{align}
	}
	We claim that the right hand side of (\ref{inequality:barRhoGammaTheta}) is nonpositive. It is easy to see that the first term on the right hand side of (\ref{inequality:barRhoGammaTheta}) is nonpositive. 
	It follows from the non-negativity of the term in (\ref{number_edges_cross1}) that $\theta \beta - \hat{\rho} \gamma_t^2 \leq 0$ and thus the second term is also nonpositive.  
	By $\rho_t = \nu k^2 \mu_0$ and $\gamma_t \leq \frac{1}{C \mu_0 k^{1.5} }$, we have
	\begin{align*}
	\rho_t \gamma_t^2 &\leq \nu k^2 \mu_0 \frac{1}{C^2 \mu_0^2 k^3} = \frac{\nu}{C^2} \frac{1}{\mu_0 k}.
	\end{align*}
	Also by $C \geq \mathrm{e} \sqrt{\nu \lambda}$ and $\alpha = 1/(\lambda \mu_0 k)$, the inequality above becomes
	\begin{align}
	\label{inequality:RhoGammaSquare}
	\rho_t \gamma_t^2 &\leq \frac{1}{ \mathrm{e}^2 \lambda} \frac{1}{\mu_0 k} = \frac{\alpha}{\mathrm{e}^2}. 
	\end{align}
	Then it follows from $\theta \in [\alpha, 1]$ and $\hat{\rho} \in [0, \rho]$ that $\hat{\rho} \gamma_t^2 \leq \theta$. Hence the last term on the right hand side of (\ref{inequality:barRhoGammaTheta}) is also nonpositive. Hence the claim follows. We conclude
	{\small 
		\begin{align}
		\label{inequality:group5}
		(1-\hat{\rho} \gamma_t^2) \log (1-\hat{\rho} \gamma_t^2) - (1-\hat{\rho} \gamma_t^2-\beta+ \theta\beta) \log (1-\hat{\rho} \gamma_t^2-\beta + \theta \beta ) + (1-\beta) \log (1-\beta) \leq 0.
		\end{align}
	}
	Now we sum up those terms on the right hand sides of (\ref{inequality:group1}), (\ref{inequality:group2}), (\ref{inequality:group3}), (\ref{inequality:group4}) and (\ref{inequality:group5}) and obtain an upper bound on $\log \mathbb{P}(\mathcal{R}(\beta n, \alpha d, \theta d n))$ normalized by $dn$.   
	\begin{align}
	\label{inequation:probabilityRMidStep}
	& \frac{1}{dn} \log \mathbb{P}(\mathcal{R}(\beta n, \alpha d, \theta d n))    \nonumber\\
	\leq & -\frac{\beta \log \beta}{d} + \frac{\beta}{d} + o(1) - \beta (\theta \log \theta + (1-\theta) \log(1-\theta)) + \theta \beta \log (\hat{\rho} \gamma_t^2) + \theta \beta.  
	\end{align}
	Using $-(1-\theta)\log(1-\theta) \leq \theta$ for $\theta \in (0,1)$, the inequality above becomes
	\begin{align*}
	\leq  &  -\frac{\beta \log \beta}{d} + \frac{\beta}{d} + o(1) - \beta \theta \log \theta + \beta \theta + \theta \beta \log (\hat{\rho} \gamma_t^2) + \theta \beta \\ 
	=  &  -\frac{\beta \log (\beta/\mathrm{e})}{d} - \beta \theta \log \frac{\theta}{\mathrm{e}} + \theta \beta \log  ( \mathrm{e} \hat{\rho} \gamma_t^2) + o(1) \\
	=  & \beta \log \left(  \left(\frac{\mathrm{e}}{\beta}\right)^{1/d} \left( \frac{\mathrm{e}^2 \hat{\rho} \gamma_t^2}{\theta} \right)^{\theta}   \right) + o(1). 
	\end{align*}
	It follows from (\ref{inequality:RhoGammaSquare}) that $\mathrm{e}^2 \rho_t \gamma_t^2 /\alpha \leq 1$, which, together with $\theta \in [\alpha, 1]$ and $\hat{\rho} \in [0, \rho]$, implies that 
	$$
	\left(\frac{\mathrm{e}}{\beta}\right)^{1/d} \left( \frac{\mathrm{e}^2 \hat{\rho} \gamma_t^2}{\theta} \right)^{\theta} \leq \left(\frac{\mathrm{e}}{\beta}\right)^{1/d} \left( \frac{\mathrm{e}^2  \rho_t \gamma_t^2}{\alpha} \right)^{\theta} \leq  \left(  \frac{\mathrm{e}}{\beta}  \left( \frac{\mathrm{e}^2 \rho\gamma_t^2}{\alpha} \right)^{\alpha d} \right)^{1/d}.
	$$
	Hence for $\beta  \geq 1.1 \mathrm{e} (\mathrm{e}^2 \rho_t \gamma_t^2 /\alpha )^{\alpha d}$, we have
	\begin{align*}
	\frac{1}{n} \log \mathbb{P}(\mathcal{R}(\beta n, \alpha d, \theta d n))  & \leq d\beta \log \left((1/1.1)^{1/d} \right) + o(1)  \nonumber \\
	& = -\beta \log 1.1 + o(1) \nonumber \\
	& \leq - 1.1 \mathrm{e} (\mathrm{e}^2 \rho_t \gamma_t^2 /\alpha )^{\alpha d} \log 1.1 + o(1).
	\end{align*}
	Then we choose $\eta = 1.1 \mathrm{e} (\mathrm{e}^2 \rho_t \gamma_t^2 /\alpha )^{\alpha d} \log 1.1$ and thus the result (\ref{inequality:eta}) follows.
\end{proof}

\end{document}